\documentclass[12pt]{article}
\usepackage[margin=1in]{geometry}
\usepackage[utf8]{inputenc}

\usepackage[english]{babel}
\usepackage[round]{natbib}
\usepackage[dvipsnames]{xcolor}

\usepackage{amsmath, amssymb, amsthm,mathtools,dsfont}
\usepackage{bbm}
\usepackage{blkarray}
\usepackage{cancel}
\usepackage{natbib}
\usepackage{enumitem}
\usepackage{euscript}
\usepackage{float}
\usepackage{bm}
\usepackage{graphicx}
\usepackage{mathrsfs}
\usepackage{microtype}
\usepackage{ragged2e}
\usepackage{sectsty}
\usepackage{thmtools}
\usepackage{tikz}
\usepackage[Symbolsmallscale]{upgreek}

\usepackage{microtype}
\usepackage{graphicx}

\usepackage{hyperref}
\usepackage[capitalize,noabbrev]{cleveref}

\usepackage[dvipsnames]{xcolor}
\usepackage{pifont}
\usepackage{tikz}
\usepackage{amsmath}
\usepackage{graphicx}
\usetikzlibrary{bayesnet}
\usepackage{bm}
\usepackage{hyperref}
\usepackage{enumitem}
\usepackage{bbm}
\usepackage{caption}
\usepackage{subcaption}

\DeclareMathAlphabet{\mathdutchcal}{U}{dutchcal}{m}{n}
\SetMathAlphabet{\mathdutchcal}{bold}{U}{dutchcal}{b}{n}
\DeclareMathAlphabet{\mathdutchbcal}{U}{dutchcal}{b}{n}
\newcommand{\N}{\mathbb{N}}
\newcommand{\vae}{\epsilon}
\usepackage{lmodern}
\newcommand{\cA}{\mathcal{A}}
\newcommand{\cB}{\mathcal{B}}
\newcommand{\cC}{\mathcal{C}}

\newcommand{\cE}{\mathcal{E}}

\newcommand{\cG}{\mathcal{G}}
\newcommand{\cH}{\mathcal{H}}
\newcommand{\cI}{\mathcal{I}}

\newcommand{\cK}{\mathcal{K}}
\newcommand{\cL}{\mathcal{L}}
\newcommand{\cM}{\mathcal{M}}
\newcommand{\cN}{\mathcal{N}}

\newcommand{\cP}{\mathcal{P}}

\newcommand{\cT}{\mathcal{T}}

\newcommand{\cV}{\mathcal{V}}

\newcommand{\cX}{\mathcal{X}}
\newcommand{\cY}{\mathcal{Y}}
\newcommand{\cZ}{\mathcal{Z}}

\newcommand{\deq}{\coloneqq}

\newcommand{\kl}[2]{\mathsf{KL}\left( #1 \parallel #2\right)}

\newcommand*{\norm}[1]{\left\|#1\right\|}
\newcommand*{\triplenorm}[1]{{\left\vert\kern-0.25ex\left\vert\kern-0.25ex\left\vert #1
    \right\vert\kern-0.25ex\right\vert\kern-0.25ex\right\vert}}

\DeclareMathOperator{\id}{id}

\newcommand{\R}{\mathbb{R}}

\renewcommand{\phi}{\varphi}

\newcommand{\sse}{\subseteq}

\newcommand*{\E}{\mathbb E}

\DeclareMathOperator{\cov}{Cov}
\DeclareMathOperator{\tr}{tr}
\DeclareMathOperator{\diag}{diag}

\newcommand*{\defeq}{\coloneqq}
\newcommand*{\rd}{\mathrm{d}}
\newcommand*{\dd}{\rd}

\DeclareMathOperator*{\argmin}{arg\,min}

\DeclarePairedDelimiter{\abs}{\lvert}{\rvert}
\newcommand\mmid{\mathbin{\|}}

\DeclareMathOperator*{\KL}{\mathsf{KL}}

\newcommand{\inner}[1]{\left\langle#1\right\rangle}

\newcommand{\balpha}{\boldsymbol{\alpha}}

\newcommand\mb[1]{\mathbf{#1}}

\newcommand\mc[1]{\mathcal{#1}}

\newcommand\msf[1]{\mathsf{#1}}

\DeclareMathOperator{\rmp}{\mathdutchcal p}
\DeclareMathOperator{\rmq}{\mathdutchcal q}
\DeclareMathOperator{\rmg}{\mathsf g}

\DeclareMathOperator{\pa}{pa}

\newcommand{\bX}{\bm{X}}
\newcommand{\by}{\mathbf{y}}

\newcommand{\iid}{\stackrel{{\rm i.i.d.}}{\sim}}

\newcommand{\LGR}{L_{\mathsf{GR}}}

\usepackage[capitalize]{cleveref}
\crefname{equation}{eq.}{eqs.}
\Crefname{equation}{Eq.}{Eqs.}

\theoremstyle{plain}
\newtheorem{theorem}{Theorem}[section]
\newtheorem{proposition}[theorem]{Proposition}

\newtheorem{lemma}[theorem]{Lemma}
\newtheorem{corollary}[theorem]{Corollary}
\theoremstyle{definition}
\newtheorem{definition}[theorem]{Definition}

\theoremstyle{remark}
\newtheorem{remark}[theorem]{Remark}

\crefname{assumption}{Assumption}{Assumptions}
\crefname{proposition}{Proposition}{Propositions}
\crefname{equation}{eq.}{eqs.}
\Crefname{equation}{Eq.}{Eqs.}
\Crefname{algocf}{Algorithm}{Algorithms}
\crefname{algocf}{algorithm}{algorithms}

\usepackage{algorithm}
\usepackage{algpseudocode}

\hypersetup{citecolor=purple, colorlinks=true, linkcolor=blue}

\usepackage{fancyhdr}

\begin{document}

\vspace*{0.3in}

\begin{center} {\LARGE{{Theory and computation for \\ structured variational inference}}}

{\large{
\vspace*{.3in}
\begin{tabular}{cccc}
Shunan Sheng$^{1,*}$, Bohan Wu$^{1,*}$, Bennett Zhu$^{1}$\\ Sinho Chewi$^{2}$, and Aram-Alexandre Pooladian$^{3}$\\
\end{tabular}
{
\vspace*{.1in}
\begin{tabular}{c}\\
				$^1$Department of Statistics, Columbia University\\
				$^2$Department of Statistics and Data Science, Yale University\\
                $^3$Foundations of Data Science, Yale University\\
 				\small{\texttt{\{ss6574,bw2766,bz2500\}@columbia.edu}}\\
                \small{\texttt{\{sinho.chewi,aram-alexandre.pooladian\}@yale.edu}}\\
\end{tabular} 
}
}}
\vspace*{.1in}

\today

\end{center}

\vspace*{.1in}

\begin{abstract}
Structured variational inference constitutes a core methodology in modern statistical applications. 
Unlike mean-field variational inference, the approximate posterior is assumed to have interdependent structure. 
We consider the natural setting of \emph{star-structured} variational inference, where a root variable impacts all the other ones. We prove the first results for existence, uniqueness, and self-consistency of the variational approximation. In turn, we derive quantitative approximation error bounds for the variational approximation to the posterior, extending prior work from the mean-field setting to the star-structured setting.
We also develop a gradient-based algorithm with provable guarantees for computing the variational approximation using ideas from optimal transport theory. 
We explore the implications of our results for Gaussian measures and hierarchical Bayesian models, including generalized linear models with location family priors and spike-and-slab priors with one-dimensional debiasing. As a by-product of our analysis, we develop new stability results for star-separable transport maps which might be of independent interest.
\end{abstract}

\footnotetext{*Equal contribution.}

\section{Introduction}
Suppose we have a collection of observations $X_1,\ldots,X_n$ which are governed by latent parameters $(z_1,\ldots,z_d)$, where $n$ and $d$ are both potentially large; the joint distribution is written as $p(z,x)$. In Bayesian statistics, the practitioner wishes to perform inference on the \emph{posterior} distribution
\begin{align*}
    \pi(z) \propto p(z \mid X_1,\ldots,X_n)\,,
\end{align*}
which is only known up to a normalizing constant.\footnote{Throughout, we use $\pi$ to denote both the measure and its Lebesgue density.} See, for instance, \citet{Berger1985book, Hoff2009book, gelman2013bayesian} and references therein. 
Inference on the latent parameters is performed by drawing samples from the unnormalized distribution $\pi$, which can be obtained through Markov Chain Monte Carlo (MCMC) methods. MCMC, a standard workhorse in Bayesian computation, has witnessed significant advancements over decades of research in methodological, statistical, and theoretical circles; see \citet{hastings1970monte, gelfand1990sampling, neal1993probabilistic,Robert2004} for standard treatments. Unfortunately, MCMC can often be computationally expensive, since many iterations of a given chain are required to generate even a single approximate sample from the posterior $\pi$. Moreover, if the posterior respects a complex graphical structure, exact MCMC methods become intractable due to the curse of dimensionality \citep{Wainwright2008}. Consequently, the computational burden often prevents practitioners from exploring different models within a reasonable time budget.

To solve this computational hurdle, there is a growing body of literature on \emph{variational inference}, a classical method dating back to \citet{Parisi1980, Hinton1993, jordan1999introduction, Wainwright2008}. In this setting, the statistician posits a tractable class $\cC$ of probability measures, and computes the following projection in the sense of relative entropy, or Kullback--Leibler (KL) divergence:
\begin{align}\label{eq:vi}
    \pi_\cC \in \argmin_{\mu \in \cC} \kl{\mu}{\pi} = \argmin_{\mu \in \cC} \int \log\Bigl(\frac{\mu}{\pi}\Bigr)\,\dd \mu\,.
\end{align}
Variational inference has found applications in, for example, large-scale Bayesian inference problems such as topic modeling \citep{Blei2003}, deep generative modeling \citep{Kingma2014,Lopez2018}, robust Bayes \citep{Wang2018robustBayesianmodeling}, marketing \citep{Braun2010}, and genome sequence modeling \citep{Carbonetto2012,Lopez2018,Wang2020EBVI-VS,Kim2022}. In these scenarios, obtaining an exact posterior sample may be intractable, whereas a well-calibrated approximation is sufficient for practical use.

The choice of $\cC \subset \cP(\R^d)$ in the variational inference problem \eqref{eq:vi} is essential, as it dictates not only the computational complexity of the resulting optimization problem but also the bias introduced by approximation. One of the most popular choices for the constraint set is the space of product measures $\cC = \cP(\R)^{\otimes d}$, known as mean-field VI (MFVI) in the literature \citep{Blei2017}. Due to its computational tractability, MFVI has been integrated into high-dimensional linear models~\citep{Carbonetto2012,Wang2020EBVI-VS,Kim2022}, generative modeling~\citep{Kingma2014}, and language modeling~\citep{Blei2003}, among other problems in probabilistic machine learning~\citep{Murphy2023}. 

However, the minimizer obtained by choosing $\cC = \cP(\R)^{\otimes d}$ can have low fidelity to the posterior $\pi$, as the latent variables are typically \emph{not} independent. For Bayesian inference, MFVI fails to capture many standard statistical procedures; see \citet{Wang2004,Behrooz2019}. This weakness has spurred a line of work on \emph{structured} VI (SVI), where dependencies are introduced among the variables of the variational distribution, dating back to \cite{Saul1995,Lauritzen1996,Barber1999SVI}. Often, the graphical structure of the family $\cC$ is fixed in advance: let $\cG$ denote a fixed tree graph\footnote{Precisely, this is a graph with vertex-edge structure $\cG = (\{1,\ldots,d\}, {\cal E}_{\cal G})$.}. We say that $\mu \in \cC_{\cG} \subset \cP(\R^d)$ if 
\begin{equation} \label{SSVI-constraint}
    \mu(z_1, \dotsc, z_d) = \prod_{(i,j) \in \cE_\cG} \phi_{ij}(z_i, z_j)\,,
\end{equation}
for some clique compatibility functions $\phi_{ij}: \R \times \R \to [0, \infty)$ where $(i,j) \in \cE_{\cG}$ is an edge in the graph $\cG$. 
The SVI problem is then
\begin{equation} \label{SSVI-obj}
\min_{\mu \in \cC_{\cG}}\kl{\mu}{\pi}.
\end{equation}
MFVI is a special case where the set of edges of the tree graph is empty, i.e., $\cal E{_\cG} = \emptyset$. 

With the increasing adoption of SVI, key empirical insights into its statistical and computational behavior have emerged. As noted by \cite{Blei2017}, structured variational families can \emph{“potentially improve the fidelity of the approximation”}, while coming at the cost of a \emph{“more difficult-to-solve variational optimization problem”}. Motivated by these insights, we investigate the two following questions: 
\begin{enumerate}[label = (\alph*)]
    \item \textbf{Approximation Guarantees:} Under what conditions does SVI provide accurate approximations to the posterior, and how can we quantify the approximation error?
    \item \textbf{Computational Guarantees:} Can we design a polynomial-time algorithm for solving the SVI problem with provable computational guarantees? 
\end{enumerate}

\begin{figure}[t]
  \centering
  \begin{subfigure}[t]{0.45\textwidth}
    \centering
    \begin{tikzpicture}
  \node[latent] (lambda) {\(\vartheta\)};
  \node[latent, right=of lambda, xshift=.15cm] (betaj) {\(\beta_j\)};
  \node[obs, right=of betaj, xshift=.15cm] (yi) {\(y_i\)};
  \node[obs, right=of yi] (xi) {\(x_i\)};
  \edge {lambda} {betaj};
  \edge {betaj} {yi};
  \edge {xi} {yi};
  \plate {dplate} {(betaj)} {\(d\)};
  \plate {nplate} {(xi)(yi)} {\(n\)};
\end{tikzpicture}
    \caption{A parameter $\vartheta$ influences $\beta$, which generates the data. The posterior measure is over $(\vartheta,\beta) = (\vartheta,\beta_1,\ldots,\beta_d) \in \R^{d+1}$.\looseness-1}
\label{fig:graphical_model_a}
  \end{subfigure}
  \hfill
  \begin{subfigure}[t]{0.45\textwidth}
    \centering
    \begin{tikzpicture}
  \node[latent] (betaj) {\(\beta_j\)};
  \node[latent, right=of betaj, xshift=-0.25cm, yshift=1.0cm] (beta1) {\(\beta_1\)};
  \node[obs, right=of betaj, xshift=1.25cm] (yi) {\(y_i\)};
  \node[obs, right=of yi] (xi) {\(x_i\)};
  \edge {betaj} {beta1};
  \edge {beta1} {yi};
  \edge {betaj} {yi};
  \edge {xi} {yi};
  \plate {dplate} {(betaj)} {\(d-1\)};
  \plate {nplate} {(xi)(yi)} {\(n\)};
\end{tikzpicture}
    \caption{A Bayesian GLM using \emph{spike-and-slab} and \emph{debiased} priors. The posterior is over the parameters $\beta = (\beta_1,\ldots,\beta_d) \in \R^{d}$.}
    \label{fig:graphical_model_b}
  \end{subfigure}
  \caption{Bayesian GLMs with different structural dependencies on the model parameters.}
  \label{fig:graphical_model}
\end{figure}
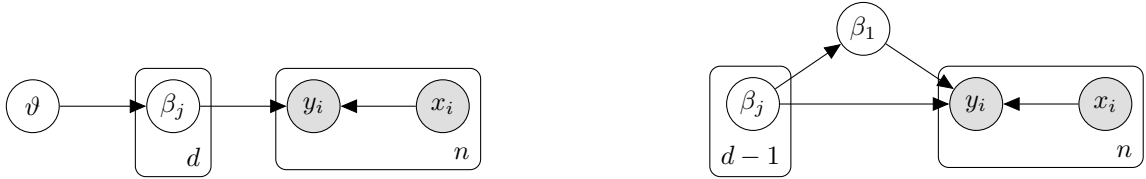

To answer these questions, we focus on a particular structured family which we call the \emph{star variational family}, denoted $\cC_{\rm{star}}$. Formally, $\mu \in \cC_{\rm star}$ if the measures admits the following decomposition 
\begin{align}\label{eq:ccstar_intro}
         \mu(z_1,\ldots,z_d) = \mu_1(z_1) \prod_{j = 2}^d \mu_j(z_j \mid z_1)\,.
\end{align}
In other words, we consider the case when the graph $\cG$ is a star graph: one of the variables takes the role of a ``root'' variable, and the remaining ones are leaves in the graph (see \cref{fig:graphical_model} for examples). 

Star graphs are ubiquitous in Bayesian statistics. Indeed, the fundamental result of de Finetti~\citep{definetti1929,de1970theory} states that the law of any (infinitely) exchangeable sequence $z_1, z_2, \ldots$ can be uniquely represented as the marginal law of the leaf variables in a star graph, where the root variable corresponds to a latent variable, and the leaf variables are $z_1, z_2, \ldots$. This observation has since been adopted as a foundational principle in Bayesian modeling and inference~\citep{Rubin1978CausalBayes, Lindley1981Exchangeable}. \cref{fig:graphical_model}(a) illustrates an example of a nested hierarchical model~\citep{hoffman2015structured, Fasano2022, Loaiza-Maya2022VI, Goplerud2024}. In this setting, the joint distribution of the observed variables $(x_{1:n}, y_{1:n})$, the local latent variables $\beta_{1:d}$, and the global latent variable $\vartheta$ is given by
\begin{equation} \label{nest hierarchical models}
    p_0(\vartheta)\prod_{j=1}^d p_1(\beta_j \mid \vartheta) \prod_{i = 1}^n p_{\cL}(y_i \mid \beta_{1:d}, x_i),
\end{equation}
where $\vartheta \sim p_0$, $\beta_1,\ldots, \beta_d \mid \vartheta \overset{\mathrm{i.i.d.}}{\sim} p_1(\cdot \mid \vartheta)$, and $y_i \mid \beta_{1:d}, x_i \sim p_{\cL}(\cdot \mid \beta_{1:d}, x_i)$ for all $i \in [n]$. Consequently, the posterior distribution of $(\vartheta, \beta_{1:d})$ factorizes according to a star graph, in which the root variable is $\vartheta$ and the leaf variables are $\beta_1, \ldots, \beta_d$.

Another notable example of the class of nested hierarchical models is the high-dimensional generalized linear mixed model. In this context, \cite{Goplerud2024} shows that applying SVI leads to improved uncertainty quantification and scalable computation while MFVI underestimates posterior uncertainty. More broadly, hierarchical priors of the form~\eqref{nest hierarchical models} play a central role in multilevel modeling; see \cite{gelman2007data} for a comprehensive overview.
 
The star-graph setting is, in essence, one step above the mean-field setting, when $\cE_{\cG}$ is empty. Nevertheless, statistical and computational properties of this regime have not been effectively explored in the literature.

\subsection{Contributions}
In this work, we develop the first theoretical and computational properties of \textbf{S}tar-\textbf{S}tructured \textbf{V}ariational \textbf{I}nference (SSVI): For a general Gibbs measure of the form $\pi  \propto e^{-V}$, where $V:\R^d \to \R$ is the potential function, find
\begin{align}\label{eq:SSVI_intro}
    \pi^\star \in \argmin_{\mu \in \cC_{\rm{star}}} \kl{\mu}{\pi}\,.
\end{align}
Our contributions are several-fold.

\paragraph*{\textbf{Existence and uniqueness of the SSVI minimizer}}
First, we investigate when this infinite-dimensional optimization problem is well-defined, i.e., establish conditions under which a unique minimizer exists. Drawing inspiration from recent works which analyze the simpler mean-field setting \citep{arnese2024convergence, Lacker2024, lavenant2024convergence, JiaChePoo25MFVI}, we consider the case where the posterior is log-concave. We show that such a condition is sufficient to obtain a unique minimizer to \eqref{eq:SSVI_intro}.
By definition of $\mc C_{\rm star}$, this minimizer is of the form
\begin{align}\label{eq:minimizer_intro}
    \pi^\star(z_1,\ldots,z_d) = \rmp^\star(z_1) \rmq^\star(z_{-1}\mid z_1)\,,
\end{align}
where $\rmp^\star$ is a univariate probability measure and $\rmq^\star$ is a stochastic kernel mapping inputs $z_1\in\R$ to product measures over $\R^{d-1}$. This is established via a dynamic programming principle; see \cref{prop:dp_principle}, and the complete result is given by \cref{thm:existence}.

\paragraph*{\textbf{Regularity properties of minimizers}}
Next, we characterize the approximation quality of $\pi^\star$, i.e., how close is this approximation to the true posterior $\pi$? 

To this end, we establish novel \emph{self-consistency equations} for the star-structured minimizer. These equations explicitly relate the components $\rmp^\star$ and $\rmq^\star$ to each other and to the posterior $\pi\propto\exp(-V)$:
 \begin{align} \label{eq:selfcon_intro}
 \begin{split}
          \rmp^\star(z_1) &\propto \exp \Bigl( - \int_0^{z_1} \int \partial_1 V(s, z_{-1}) \,  \rmq^\star(\dd z_{-1}\mid s)\, \dd s \Bigr)\,,\\
       \rmq_i^\star(z_i\mid z_1) &\propto \exp \Bigl( -\int V (z_1,z_i, z_{-\{1,i\}} )\, \prod_{j \geq 2,\, j \neq i} \rmq_j^\star (\dd z_j \mid z_1)\Bigr)\,,
 \end{split}
\end{align}
for $i \in \{2,\ldots, d\}$; see \cref{thm-self-consistency}. If we additionally assume the potential $V$ itself satisfies a \emph{root domination criteria}:
\begin{equation} \label{assum:RD_intro}
    \partial_{11} V - \frac{\|\sum_{j = 2}^d \left(\partial_{1j} V \right)^2 \|_{L^\infty}}{\ell_V} > 0\,,
\end{equation}
then we establish the following approximation bound:
\begin{equation}\label{eq:approx_intro}
    \kl{\pi^\star}{\pi} \lesssim \sum_{i \geq 2} \sum_{j > i} \E_{\pi^\star}[(\partial_{ij} V)^2]\,,
\end{equation}
where the underlying constant is explicitly characterized; see \cref{thm: approx gap}. The required assumption essentially asserts that $V$ has more curvature along the root direction than along the root–leaf interactions.

Our approximation guarantee strictly generalizes that of \cite{Lacker2024}, which only applies to the mean-field case. Indeed, bound~\eqref{eq:approx_intro} substantially improves upon the MFVI guarantee of \cite{Lacker2024} when the root–leaf interactions $(\partial_{1j} V)_{j \geq 2}$ dominate the inter-leaf interactions $(\partial_{ij} V)_{i \neq j,\ i,j > 1}$.
We conclude \cref{sec:SSVI_main} with several illustrative examples, including the Gaussian posterior (\cref{sec:gaussian}) and Bayesian generalized linear models (\cref{example:glm}) with location-family priors and spike-and-slab priors, respectively. Notably, in the Gaussian case, we provide an explicit characterization of the SSVI minimizer and quantify the improvement of performing SSVI over MFVI.

\paragraph*{\textbf{Computational guarantees via (linearized) optimal transport}}
In \cref{sec:computation}, we shift focus to computational aspects. We stress that the preceding developments do not naturally give rise to an optimization routine which is amenable to computational analysis. This is simply because, at first glance, the constraint set $\cC_{\rm star}$ is non-convex. To circumvent this, we draw inspiration from the growing body of work that leverages \emph{optimal transport theory} in statistical applications \citep{panaretos2020invitation, CheNilRig25OT}, and recast the SSVI problem into an optimization problem at the level of transport maps. 

Taking $\rho =\cN(0,I)$, for each measure in $\cC_{\rm star}$, we can parameterize it by a map $T$ that belongs to a family of \emph{star-separable transport maps} $\cT_{\rm{star}}$ (see \cref{sec:lift_maps} for a precise definition). These maps naturally form a convex set and thus, the SSVI problem~\eqref{eq:SSVI_intro} can be reformulated as a convex optimization problem over $\cT_{\rm star}$,
\begin{align}\label{eq:SSVI_map_intro}
    T^\star = \argmin_{T \in \cT_{\rm{star}}} \kl{T_\# \rho}{\pi}.
\end{align}
We note that the set $\cT_{\rm{star}}$ is closed under convergence in $L^2(\rho)$, a property that is intimately connected to the adapted Wasserstein distance~\citep[see][and references therein]{beiglbock2023knothe}. This transport map perspective not only yields computational guarantees but also provides an alternative route to proving many of the results in \cref{sec:SSVI_main} via direct analysis of~\eqref{eq:SSVI_map_intro}, including the existence and uniqueness of the SSVI minimizer (\cref{thm:convexity}). 

    To proceed, we borrow inspiration from the machine learning literature \citep{wang2013linear, JiaChePoo25MFVI} and exhibit an explicit parameterization of the transport maps in $\cT_{\rm star}$, giving rise to a finite-dimensional parameter space $\Theta$ such that optimization over $\Theta$ \emph{preserves} the convexity of the underlying optimization problem. Concretely, we devise a (projected) gradient-descent algorithm over our parameterized space, and we show that the minimizers of the gradient-descent scheme are close to the ground-truth minimizer in~\eqref{eq:SSVI_map_intro}. The general scheme is
\begin{align*}
    \argmin_{T \in \cT_{\rm{star}}} \kl{T_\# \rho}{\pi} = T^\star \simeq T^\star_\Theta  = \argmin_{T \in \cT_{\Theta}} \kl{T_\# \rho}{\pi}\,,
\end{align*}
and we perform gradient descent with respect to the parameters $\theta \in \Theta$ to obtain $T^\star_\Theta$.

The aforementioned convergence properties are made possible by way of exploiting the regularity properties of the optimal star-structured map $T^\star :\R^d \to \R^d$ in \eqref{eq:SSVI_map_intro}, which, to be concrete, is of the form
\begin{align*}
    x\mapsto T^\star(x) = (z_1,T_2^\star(x_2 \mid z_1),\ldots, T_d^\star(z_d \mid z_1))\,, \quad z_1 = T_1^\star(x_1)\,.
\end{align*}
In order to develop our computational guarantees, we prove novel regularity estimates for $T^\star$ in the style of \citet{Caffarelli2000}. Under some assumptions (which are stronger than \eqref{assum:RD_intro}), we prove, for instance, bounds on
\begin{align*}
   |\partial_{x_1}^2 T_1^\star(x_1)|\,, \quad |\partial_{x_i}^2 T_i^\star(x_i\mid z_1)|\,,
\end{align*}
and the mixed derivatives. To the best of our knowledge, these are the first results that study the regularity of transport maps with respect to varying target measures supported on unbounded domains, and are of independent interest.

\paragraph*{\textbf{Going beyond the star graph}} Finally, in \cref{sec:beyond_star}, we discuss the additional technical hurdles when going beyond the star-graph structure, which we leave for future work, and other open questions. 

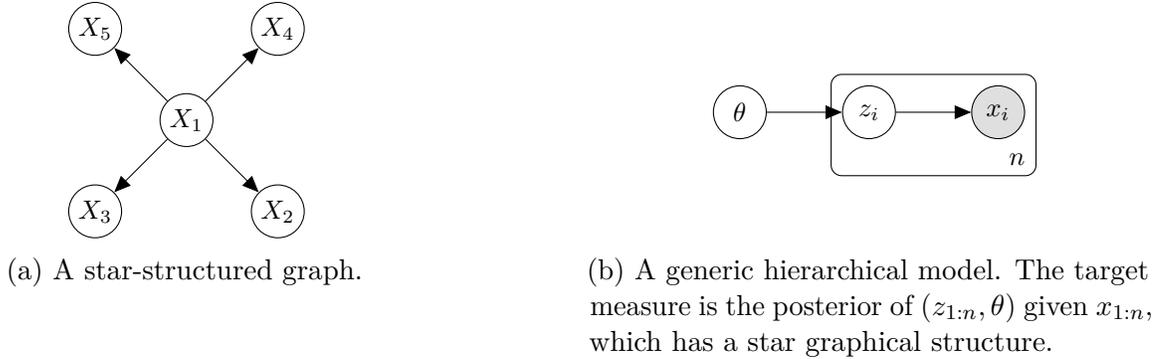
\begin{figure}[t]
  \centering
  \begin{subfigure}[t]{0.45\textwidth}
    \centering
    \begin{tikzpicture}
    \node[latent] (x1) {\(X_1\)};
    \node[latent, above right=of x1] (x4) {\(X_4\)};
    \node[latent, below right=of x1] (x2) {\(X_2\)};
    \node[latent, below left=of x1] (x3) {\(X_3\)};
    \node[latent, above left=of x1] (x5) {\(X_5\)};
    \edge {x1} {x2};
    \edge {x1} {x3};
    \edge {x1} {x4};
    \edge {x1} {x5};    
  \end{tikzpicture}
    \caption{A star-structured graph.}
  \end{subfigure}
  \hfill
  \begin{subfigure}[t]{0.45\textwidth}
    \centering
    \raisebox{2em}{ \begin{tikzpicture}
    \node[latent] (theta) {\(\theta\)};
    \node[latent, right=of theta] (zi) {\(z_i\)};
    \node[obs, right=of zi] (xi) {\(x_i\)};
    \edge {theta} {zi};
    \edge {zi} {xi};
    \plate {plate1} {(zi)(xi)} {\(n\)};
  \end{tikzpicture}}
    \caption{A generic hierarchical model. The target measure is the posterior of $(z_{1:n}, \theta)$ given $x_{1:n}$, which has a star graphical structure.}
  \end{subfigure}
  \caption{A general purpose illustration of star-structured variational inference.}
\end{figure}

\subsection{Related work}

\paragraph*{Structured variational inference}
 There is a growing body of work on structured VI spanning decades of research \citep{Saul1995, Salimans2013, hoffman2015structured,xiao2024treevi} with ample applications, such as time series models \citep{Saul1995, Salimans2013, Tan2018, Frazier2023}, and state-space modeling \citep{ZhangGao2020Rates,Frazier2023}.

The theory of VI has seen accelerated progress in recent years. Closely related to our work are results on the approximation accuracy of Gaussian VI~\citep{Katsevich2023Gaussian} and MFVI~\citep{Lacker2024}. Other lines of research have studied asymptotic normality~\citep{Hall2011GaussianVI, Bickel2013, Wang2019}, posterior contraction rates~\citep{ZhangGao2020Rates}, non-asymptotic risk bounds~\citep{Alquier2016, Alquier2020, Yang2020Alpha}, and properties of variational inference in high-dimensional generalized linear models~\citep{Tan2021VB, Fasano2022, Mukherjee2022, Ray2021, Mukherjee2023, Qiu2024}.

A parallel line of work has focused on establishing computational guarantees for variational inference algorithms. More closely related is the work of \citet{JiaChePoo25MFVI}, in which authors propose a polynomial-time algorithm for computing the MFVI minimizer by directly approximating and optimizing over the space of univariate optimal transport maps. Other results in this direction include convergence analyses for coordinate-ascent VI \citep{arnese2024convergence, lavenant2024convergence, BhaPatYan25CAVI}, black-box variational inference \citep{domke2023provable, Kim2023Convergence}, Gaussian VI \citep{Lambert2022, Diao2023}, particle-based algorithms \citep{Du2024Particle} and computational-statistical trade-offs \citep{Bhatia2022,Wu2024Entropic}. When the log-concavity assumption is violated, a recent result shows that MFVI suffers from mode collapse \citep{sheng2025mode}.We also mention the work of \citet{ZhangZhou2020}, which provides full statistical and computational guarantees for MFVI in stochastic block models.

\paragraph*{Linearized optimal transport}
Our approach to recasting the SSVI problem as an optimization problem at the level of transport maps follows the general scheme of linearized optimal transport~\citep{wang2013linear}. Fixing an absolutely continuous probability measure $\rho \in \cP_{2,\mathrm{ac}}(\R^d)$ with finite second moment, one can represent any other probability measure $\mu$ with finite second moment via the optimal transport map from $\rho$ to $\mu$. This perspective has been popularized in statistical applications. When $d = 1$, for example, \citet{zhu2023autoregressive} use the Fr\'echet mean of a distributional time series as the base measure and propose an auto-regressive model defined directly on the transport maps. This viewpoint builds also on the classical work in distribution learning, such as \citet{Petersen2016}, where probability densities are transformed into a Hilbert space of functions via continuous, invertible maps, allowing tools from functional data analysis~\citep{Hsing2015} to be applied. Similar ideas have also been explored by \citet{Bachoc2018} and \citet{Zemel2019} in related problems involving learning of probability distributions. In higher dimensions, optimal transport maps have been used by \citet{Chernozhukov2017, HallinBarrio2021, ghosal2022multivariate, deb2023multivariate} and others to define multivariate analogues of quantile functions and to perform rank-based inference and quantile regression.

\subsection{Notation}

We denote by $\mathbb{S}_+^d$ the set of positive definite $d\times d$ matrices. For any $A \in \mathbb{S}_+^d$, we write $A^{ij} \deq (A^{-1})_{ij}$, and $\|A\|_2$ for the operator norm of $A$. A function $f: \R^d \to \R \cup \{+\infty \}$ is \emph{$\alpha$-convex} if $f - \frac{\alpha}{2}\, \|\cdot\|_2^2$ is convex, and \emph{$\beta$-smooth} if $\nabla^2 f \preceq \beta I$. We use $\partial_{ij} f$ to denote the $(i,j)$-th entry of the Hessian $\nabla^2 f$. Let $\cP(\R^d)$ denote the space of probability measures on $\R^d$ and $\cP_p(\R^d)$ the subspace of measures with finite $p$-moment for $p\ge 0$. For $\rho \in \cP(\R^d)$, we say that a vector-valued function $T: \mathbb{R}^d \to \mathbb{R}^d$ belongs to $L^2(\rho)$ if the function $x \mapsto \|T(x)\|$ is in $L^2(\rho)$ and define $\|T\|_{L^2(\rho)} \deq \left(\int \|T(x)\|\,\rho(\dd x)\right)^{1/2}$.  We denote by $\|f\|_{L^\infty}$ the $L^\infty$ norm of $f$ under the Lebesgue measure. Given a measurable function $T: \R^d\to\R^d$ and $\mu\in \cP(\R^d)$, we write $T_\#\mu$ for the push-forward measure $\mu\circ T^{-1}$. If $\pi \in \cP(\R^d)$ is absolutely continuous, we abuse notation and identify it with its Lebesgue density. If $\pi \propto \exp(-V)$, we say $\pi$ is $C^2$ and \emph{$\alpha$-log-concave} if $V:\R^d \to \R$ is $C^2$ and $\alpha$-convex; $V$ is called the \emph{potential} of $\pi$. Similarly, we say $\pi$ is \emph{$\beta$-log-smooth} if $V$ is $\beta$-smooth.

We use $\E_{\pi}[f(Z)]$ to denote the expectation of $f$ under $\pi$. The Kullback--Leibler (KL) divergence between $\mu, \pi \in \cP(\R^d)$ is defined as $\kl{\mu}{\pi} \deq \int \log( \frac{\dd\mu}{\dd \pi})\, \dd\mu$ if $\mu$ is absolutely continuous with respect to $\pi$ (denoted $\mu \ll \pi$) and $\kl{\mu}{\pi} = \infty$ otherwise. Here, $\frac{\dd\mu}{\dd \pi}$ denotes the Radon--Nikodym derivative of $\mu$ with respect to $\pi$. The (negative) differential entropy is then defined as $\cH(\mu) \deq \int \log\mu\, \dd\mu$ if $\mu$ is absolutely continuous with respect to the Lebesgue measure on $\R^d$ and $\infty$ otherwise. For any index set $I \subset [d] \defeq \{1,\ldots, d\}$, we identify $z=(z_I;z_{-I}) \in \R^d$ and use $\pi_{I}$ to denote the marginal density of the coordinates $z_I$. For $p\ge 1$ and $\mu,\nu \in \cP_p(\R^d)$, the $p$-Wasserstein distance is given by $W_p(\mu,\nu) \defeq \Bigl(\inf_{\pi \in \Pi(\mu,\nu)}\iint \|x-y\|^p\, \pi(\dd x,\dd y)\Bigr)^{1/p}$, 
where $\Pi(\mu,\nu)$ is the set of all couplings between $\mu$ and $\nu$, i.e., the set of measures in $\cP_p(\R^d\times \R^d)$ with marginals $(\mu,\nu)$.

For quantities $a,b$, we write $a \lesssim b$ or $a = O(b)$ if there exists a constant $C > 0$ (which may depend on other parameters depending on context) such that $a \leq C b$. 
If both $a\lesssim b$ and $b \lesssim a$, then we write that $a \asymp b$. We write $a = \widetilde O(b)$ if $a = O(b  \operatorname{polylog}(b))$, that is, $a \lesssim b$ up to a polylogarithmic factor.

Finally, throughout the rest of the paper, we assume that the potential $V$ of the posterior $\pi$ satisfies the following growth condition: for any $c_2 > 0$, there exists $c_1 > 0$ such that
\begin{equation}\label{eq:GR V}
|V(z)| \leq c_1 e^{c_2\, \|z\|^2} \qquad \text{for all } z \in \R^d\,.
\end{equation}
This growth condition only requires that $V$ grows slower than an exponential function, and it is satisfied when $\nabla^2 V$ is bounded, i.e., when $V$ is smooth. We remark that essentially the same a similar growth condition is required by \citet{Lacker2024} to establish the uniqueness of the mean-field variational inference minimizer. 

\section{A first theory for star-structured variational inference}\label{sec:SSVI_main}
In this first section, we prove several novel properties for solutions to the star-structured variational inference (SSVI) problem:
\begin{align}\tag{$\msf{SSVI}$}\label{eq:SSVI_main}
   \pi^\star \in \argmin_{\mu \in \cC_{\rm star}} \kl{\mu}{\pi}\,.
\end{align}
Recall that $\mu \in \cC_{\rm star} \subset \cP(\R^d)$ is a probability measure that exhibits the following graphical structure:
\begin{align*}
     \mu(z_1,\ldots,z_d) = \mu_1(z_1) \prod_{j = 2}^d \mu_j(z_j \mid z_1)\,,
\end{align*}
where $\mu_j(\cdot \mid z_1)$ is the conditional density of $z_j$ given $z_1$ under $\mu$. 

Rather than studying specific models, we aim to identify general conditions under which one can provably solve the SSVI problem---for which a prerequisite is to have a unique minimizer.
The core assumption we rely on throughout this work is log-concavity~\citep{saumard2014log}:
\begin{equation}\tag{\textbf{SLC}}\label{ass:slc}
    \begin{aligned}
        &\pi \propto \exp(-V)\,\,\text{is}\,\, C^2\,\,\text{and}\,\,\alpha\text{-log-concave, i.e., } 0 \prec  \alpha I \preceq \nabla^2 V\,.\\
    \end{aligned}
\end{equation}
The assumption \eqref{ass:slc} is standard in the theoretical and computational study of sampling \citep{ChewiBook} and variational inference \citep{Lambert2022, domke2023provable, arnese2024convergence, lavenant2024convergence, JiaChePoo25MFVI}, as well as existing works which study regularity properties of the simpler mean-field setting \citep{Lacker2024}. Following recent trends which view sampling as optimization over the space of measures \citep{wibisono2018sampling}, we will see shortly that strong log-concavity allows us to speak of \emph{unique} minimizers to the SSVI problem. 

All remaining proofs from this section are deferred to \cref{app:proofs_main}.

\subsection{Existence and uniqueness of the minimizer}\label{sec:main_characterization}
We first establish existence and uniqueness of the minimizer to the SSVI problem. For any $\mu\in\cC_{\rm star}$, we have by design the following decomposition via the chain rule for relative entropy:
\begin{align*}
    \kl{\mu}{\pi} = \kl{\mu_1}{\pi_1} + \int \underbrace{\kl{\mu_{-1}(\cdot \mid z_1)}{\pi_{-1}(\cdot\mid z_1)}}_{(*)}\,\mu_1(\dd z_1)\,,
\end{align*}
where, for fixed $z_1 \in \R$, we view $\pi_{-1}(\cdot\mid z_1)$ as a probability measure over $\R^{d-1}$. This formulation suggests that we can optimize over $\mu_1$ and $\mu_{-1}$ separately. Above, the nested subproblem $(*)$ is itself a \emph{mean-field} variational inference problem, since $\mu_{-1}(\cdot\mid z_1)$ is constrained to be a product measure, where the target measure is $\pi_{-1}(\cdot \mid z_1) \in \cP(\R^{d-1})$. We make use of the nested structure in our first result, which relates \eqref{eq:SSVI_main} to a dynamic programming problem.

\begin{proposition}[Dynamic programming]\label{prop:dp_principle}
The following equivalence holds:
\begin{align}\label{eq:equivalence}
    \inf_{\mu \in \cC_{\rm star}} \kl{\mu}{\pi} = \inf_{\mu_1 \in \cP(\R)} \kl{\mu_1}{\pi_1} + \int \inf_{\nu \in \cP(\R)^{\otimes (d-1)}} \kl{\nu}{\pi_{-1}(\cdot \mid z_1)} \mu_1(\dd  z_1)
\end{align}
\end{proposition}

The existence and uniqueness of the SSVI minimizer then follows as a consequence.

\begin{theorem}[Existence and uniqueness of the SSVI minimizer]\label{thm:existence}
Under \eqref{ass:slc}, \eqref{eq:SSVI_main} has a unique minimizer of the form
\begin{equation*}
    \pi^\star(\dd z)  \deq  \rmp^\star(\dd z_1) \, \rmq^\star(\dd z_{-1} \mid z_1) \propto e^{-h_{\rm val}(z_1)}\,\pi_1(\dd z_1) \, \rmq^\star(\dd z_{-1} \mid z_1)\,, 
\end{equation*}
where $\rmq^\star(\cdot\mid z_1) \in  \cP(\R)^{\otimes (d-1)}$ is the unique minimizer for the problem
\begin{equation*}
    h_{\rm val}(z_1)  \deq  \inf_{\nu \in \cP(\R)^{\otimes (d-1)}} \kl{\nu}{\pi_{-1}(\cdot \mid z_1)}\,,
\end{equation*}
for $\pi_1$-almost all $z_1$.
\end{theorem}

\cref{thm:existence} characterizes the structure of the SSVI minimizer by revealing a close connection to a conditional MFVI problem over the leaf nodes. The marginal distribution $\rmp^\star$ of the root variable matches the exact marginal $\pi_1$ of the target posterior, up to a tilting factor $h_{\mathrm{val}}$ given by the approximation quality of the inner MFVI procedure. When the conditional $\pi_{-1}(\cdot \mid z_1)$ is exactly a product measure, the tilting vanishes, and the SSVI marginal recovers the true root marginal exactly.

We now sketch the strategy to construct the minimizer from \cref{thm:existence}. On the one hand, the first marginal of the SSVI minimizer is given by
\[
\rmp^\star(z_1) \propto e^{-h_{\rm val}(z_1)}\,\pi_1(z_1)\,,
\]
provided $h_{\rm val}$ is measurable in some sense. On the other hand, for each $z_1$, Assumption \eqref{ass:slc} and \citet[Theorem~1.1]{Lacker2024} imply that there is a unique $\rmq^\star(\cdot \mid z_1) \in \cP(\mathbb{R})^{\otimes (d-1)}$ attaining $h_{\rm val}(z_1)$. Therefore, if $z_1 \mapsto \rmq^\star(\cdot \mid z_1)$ induces a measurable stochastic kernel, \cref{prop:dp_principle} states that SSVI can be solved by ``gluing'' $\rmp^\star$ with the stochastic kernel $z_1 \mapsto \rmq^\star(\cdot \mid z_1)$.  The right notion of measurability is \emph{universal measurability}. We refer the reader to \citet[\S7]{bertsekas1996stochastic} for the definition of this concept, and to the proof in \cref{proof:prop:dp_principle} for the details of the argument. Notably, after modification on a null set under $\pi_1$, both $h_{\rm val}$ and the map $z_1 \mapsto \rmq^\star(\cdot \mid z_1)$ can be made Borel measurable.

Recall that strong log-concavity is also assumed by \cite{Lacker2024} to show the existence and uniqueness of the solution in the simpler mean-field setting. However, the proof of existence in MFVI relies on the direct method using the topology of weak convergence, which does not suffice to establish existence of a solution in SSVI because conditional independence is not preserved under weak convergence. For further discussion on this matter, see, e.g., \cite{lauritzen2024total}. In \cref{thm:convexity}, we provide an alternative proof of the existence and uniqueness of the SSVI minimizer under~\eqref{ass:slc}, based on a reformulation of the problem under a different geometry, in which conditional independence is preserved after taking the limit.

\begin{remark}
We later use the characterization of SSVI as a dynamic program to show that when the posterior $\pi$ is Gaussian, the SSVI minimizer remains Gaussian with a structured covariance; see \cref{thm: Gaussian} below. 
\end{remark}

\subsection{Regularity and approximation quality via self-consistency equations}\label{sec:main_regularity}

\cref{thm:existence} shows that a unique minimizer exists under \eqref{ass:slc}, but it does not directly yield important information of said minimizer. For instance, we do not glean any insights regarding its regularity properties or how good of an approximation it is to $\pi$.

To this end, we now establish \emph{self-consistency equations} for the SSVI minimizer $\pi^\star$. These equations allow us to relate the components of $\pi^\star$ to the original posterior $\pi$, and thus transfer the regularity of $\pi$ (namely, regularity of $V$) to regularity of $\pi^\star$. 

 \begin{theorem}[Self-consistency equations for $\pi^\star$]
\label{thm-self-consistency}
Under \eqref{ass:slc}, any minimizer $\pi^\star(\dd z) =\rmp^\star(\dd z_1)\,\rmq^\star(\dd z_{-1} \mid z_1)$ for \eqref{eq:SSVI_main} with differentiable density satisfies the following equations\footnote{The integrals in \eqref{eq: self-consistent-pi-star} are well-defined due to the integrability of $V$ and $\nabla V$ by \eqref{ass:slc} and the required growth condition on $V$ in~\eqref{eq:GR V}.
}:
 \begin{equation} \label{eq: self-consistent-pi-star}
 \begin{aligned}
      \rmp^\star(z_1) &\propto  \exp \Bigl( - \int_0^{z_1} \int \partial_1 V(s, z_{-1}) \,  \rmq^\star(\dd z_{-1}\mid s)\, \dd s \Bigr)\,,\\
   \rmq_i^\star (z_i\mid z_1) &\propto \exp \Bigl( -\int V (z_1,z_i, z_{-\{1,i\}})\, \rmq^\star_{-i}(\dd z_{-{\{1,i\}}}\mid  z_1)\Bigr)\,, \qquad i \in [d]\setminus \{1\}\,.
 \end{aligned}
\end{equation}
\end{theorem}

Our result is the first to extend the existing fixed-point characterization from the mean-field setting~\cite[Theorem~1.1]{Lacker2024}. Notably, when conditioning on the root {variable} $z_1$, the fixed-point equation in \eqref{eq: self-consistent-pi-star} for the conditional minimizer $\rmq^\star(\cdot \mid z_1)$ recovers the fixed-point equation for the MFVI minimizer corresponding to the conditional target measure $\pi_{-1}(\cdot \mid z_1)$. We remark that the proof of \cref{thm-self-consistency} builds on the first-order optimality condition of SSVI in the Wasserstein geometry, which requires the SSVI minimizer $\pi^\star$ to be \emph{a priori} differentiable.
Since our focus here is to characterize the structure of the minimizer, we leave verification of this assumption to future work. Nonetheless, we will see in \cref{thm: star graph regularity} that both $\rmp^\star$ and $\rmq^\star(\cdot\mid z_1)$ are \emph{a posteriori} log-concave under further assumptions, ensuring their almost everywhere differentiability by \citet[Theorem~25.3]{Rockafellar1997}. From this point onward, we will implicitly assume that $\pi^\star$ is differentiable in all subsequent results.

To proceed further, we require some assumptions on the Hessian of $V$. In light of the dynamic programming principle, it is natural to partition the assumptions between the contribution of the root of the Hessian, $\partial_{11} V$, and its principal minor $(\nabla^2 V)_{-1}$.\footnote{Recall that this means we remove the first row and column of the Hessian.} To this end, let $0 < \ell_V\leq L_V$ and $L_V' > 0$; then, we assume

\begin{align}
\partial_{11} V &\leq \frac{1}{2}\, L_V'\,, \tag{\textbf{R}} \label{assum:R} \\
0 \prec \ell_V I_{d-1} &\preceq (\nabla^2 V)_{-1}\,, \tag{\textbf{P1}} \label{assum:P1} \\
(\nabla^2 V)_{-1} &\preceq L_V I_{d-1}\,. \tag{\textbf{P2}} \label{assum:P2}
\end{align}
\cref{prop:dp_principle} tells us that the inner optimization problem is precisely a MFVI problem with respect to  the leaf  variables. The assumptions~\eqref{assum:P1} and~\eqref{assum:P2} therefore correspond to the assumptions in, e.g., \cite{arnese2024convergence, Lacker2024, lavenant2024convergence}. Assumption~\eqref{assum:R}, also standard, controls the curvature of $V$ in the direction of the root variable. To complete our assumptions, we need to describe how these terms interact through a novel \emph{root domination criteria}, which is necessary for SSVI. We require that 
\begin{equation} \tag{\textbf{RD}} \label{assum:RD}
    \partial_{11} V - \frac{\|\sum_{j = 2}^d \left(\partial_{1j} V \right)^2 \|_{L^\infty}}{\ell_V} \geq \ell_V' > 0\,. 
\end{equation}
 Geometrically, assumption~\eqref{assum:RD}  guarantees that the potential $V$ is more curved along the $z_1$ direction than it is  “twisted”  by the interactions between $z_1$ and $z_2,\dots,z_d$. It can also be viewed as a uniform requirement on the Schur complement of $\nabla^2V$~\citep[see, e.g.,][Theorem~1.12]{Zhang2006}. Indeed, note that for $V$ to be positive definite in general, the following criterion must hold\footnote{Here, $\nabla_{-1,1}^2 V$ represents the first column of the Hessian $\nabla^2 V$ without the first entry.}
 \begin{equation*}
     \partial_{11} V - (\nabla_{-1, 1}^2 V)^\top\, ((\nabla^2 V)_{-1})^{-1}\, (\nabla_{-1, 1}^2 V)  > 0\,.
 \end{equation*}
 Making use of~\eqref{assum:P1}, it is clear that assumption~\eqref{assum:RD} is a quantitative version of this inequality, where we require uniform control on the lower bound.

\begin{remark}
When $V(x) = \sum_{i=1}^d V_i(x_i)$, i.e., $\pi$ is a product measure, it is easy to see that \eqref{assum:RD} is trivially satisfied, and $\ell_V' = \inf_{x\in \R^d} \partial_{11} V(x)$.
\end{remark}

The following lemma establishes that, should a posterior $\pi$ satisfy the aforementioned conditions, it is itself strongly log-concave. 

\begin{lemma} \label{lemma: log-concavity-pi-schur}
If $\pi$ satisfies \eqref{assum:P1} and \eqref{assum:RD}, then $\pi$ satisfies \eqref{ass:slc} with parameter $\ell_V \land \ell_V'$.  Additionally, if \eqref{assum:R} 
 and \eqref{assum:P2} hold, then $\pi$ is $L_V \lor (L_V'/2)$-log-smooth.
\end{lemma}

By \cref{thm:existence}, the strong log-concavity of $\pi$ implies the existence of a unique star-structured solution.

\begin{corollary}[Uniqueness]
Suppose $\pi\propto\exp(-V)$ satisfies \eqref{assum:P1} and \eqref{assum:RD}. Then there exists a unique star-structured minimizer of \eqref{eq:SSVI_main}.
\end{corollary}

Our first main result of this section is the following theorem, which establishes strong log-concavity and log-smoothness of the marginals of the star-structured minimizer.

\begin{theorem}[Log-concavity and log-smoothness of $\pi^\star$]\label{thm: star graph regularity}
Suppose $\pi \propto \exp(-V)$ satisfies \eqref{assum:P1} and \eqref{assum:RD}. Then, for the SSVI minimizer $\pi^\star$:
\begin{enumerate}[label = (\roman*)]
    \item the marginal distribution $\rmp^\star$ is $\ell_V'$-log-concave,
    \item for every $i \geq 2$, the conditional distribution $\rmq_i^\star(\cdot \mid z_1)$ is $\ell_V$-log-concave for any fixed $z_1\in \R$.
\end{enumerate}
In addition, if  \eqref{assum:R} and \eqref{assum:P2} hold, then $\rmp^\star$ is  $L_V'$-log-smooth and $\rmq_i^\star(\cdot \mid z_1)$ is $L_V$-log-smooth for any fixed $z_1\in \R$.
\end{theorem}

Given the conditional independence structure, the properties of $\rmq_i^\star(\cdot \mid z_1)$ for $i \ge 2$ can be derived from the self-consistency equations~\eqref{eq: self-consistent-pi-star}. However, analyzing $\rmp^\star$ is substantially more intricate, as it depends on $\rmq_i^\star(\cdot \mid z_1)$, which in turn depends on $\rmp^\star$. To handle this, we employ a novel stability bound for mean-field variational inference developed by some of the authors \citep{Sheng2025Stability}. Combined with the self-consistency equation for $\rmp^\star$, this yields the desired regularity result. In the special case where $\pi$ is a product measure, we have $\ell_V' = \ell_V$, and our result recovers the properties of MFVI minimizer in \cite{Lacker2024}. Aside from being of theoretical interest in its own right, \cref{thm: star graph regularity} will later act as a cornerstone for our computational guarantees in \cref{sec:gradient_descent}.

Additionally, the aforementioned regularity results will allow us to establish the first approximation guarantees for the SSVI minimizer with respect to the true posterior.

\begin{theorem}[Approximation guarantee for SSVI] \label{thm: approx gap}
Let $\pi \propto \exp(-V)$ be such that $V$ satisfies \eqref{assum:R}, \eqref{assum:P1}, and \eqref{assum:RD}. Then 
\begin{equation}\label{eq:approximation_main}
  \inf_{\mu \in \cC_{\rm star}}\kl{\mu}{\pi} = \kl{\pi^\star}{\pi}  \leq \frac{L_V'}{2\ell_V' \ell_V^2}  \sum_{i \geq 2} \sum_{j > i} \E_{\pi^\star} \bigl[ (\partial_{ij} V(Z))^2 \bigr]\,. 
\end{equation}  
\end{theorem}

The proof of this result is based on delicate applications of functional inequalities for log-concave measures and is deferred to~\cref{proof:thm: approx gap}. Before specializing this result to concrete examples, we present some general remarks here.
Note that, as expected, when the target measure $\pi$ already respects a star-graph structure, the partial derivatives $\partial_{ij}V$ on the
right-hand side of \eqref{eq:approximation_main} vanish.

It is natural to ask how the star-structured approximation compares to mean-field approximation. By \cref{lemma: log-concavity-pi-schur}, the target measure $\pi$ is $(\ell_V\land \ell_V')$-log-concave and, in this setting, \citet[Theorem 1]{Lacker2024} establish an upper bound for the MFVI minimizer $\bar\pi$ which reads
\begin{equation*}
    \kl{\bar \pi}{\pi}  \leq        \frac{1}{(\ell_V \land \ell_V')^2} \sum_{i \geq 1} \sum_{j > i} \E_{\bar\pi} \bigl[ (\partial_{ij} V(Z))^2 \bigr]\,.
\end{equation*}
Compared to our results, the mean-field bound introduces terms related to $\{\partial_{1i} V\}_{i=1}^d$. Consequently, when the association between the root and leaf nodes is more prominent than the inter-leaf associations, we expect a strong improvement when using SSVI compared to MFVI. The gap between the approximation errors of SSVI and MFVI is highlighted in the following Gaussian example.

\subsection{Examples}\label{sec:examples}

We now instantiate our approximation guarantee from~\cref{thm: approx gap} on several examples that appear in the literature.

\subsubsection{Gaussian posterior}\label{sec:gaussian}

We first consider the setting where $\pi$ is itself a multivariate Gaussian distribution. Existing results, such as those by \cite{Lacker2024}, state that the minimizer remains Gaussian when the variational family consists of product measures---the same is true in our setting.

\begin{theorem}\label{thm: Gaussian}
    Let $\pi = \cN(m,\Sigma)$ with $m \in \R^d$ and $\Sigma\succ 0$, where the covariance has entries $\sigma_{ij}$. Then,~\eqref{eq:SSVI_main} admits a unique minimizer $\pi^\star(\dd z) =\rmp^\star(\dd z_1)\,\rmq^\star(\dd z_{-1} \mid z_1)$, where $p^\star= \cN\left(m_1, \sigma_{11}\right)$ and $\rmq^\star(\cdot \mid z_1) = \cN(m^\star_{z_1}, \Sigma_{|1}^\star)$ with 
      \begin{align*}
      m^\star_{z_1,i} &= m_i + (\sigma_{11})^{-1}\,\sigma_{1i}\,(z_1-m_1)\,,\\
      \Sigma_{|1}^\star &= {\rm diag}((\sigma^{22})^{-1}, \ldots, (\sigma^{dd})^{-1})\,,
  \end{align*}
  where $\sigma^{ij}$ is the $(i,j)$-th entry of $\Sigma^{-1}$. In other words, $\pi^\star = \cN(m^\star, \Sigma^\star)$ where $\Sigma^\star = (\sigma^\star_{ij})_{1\leq i,j\leq d}\in \R^{d\times d}$ satisfies
  \begin{equation}\label{eq: SSVI Gaussian}
      \begin{aligned}
          \sigma^\star_{1i} & = \sigma_{1i}\,, &i&\in [d]\,,\\
          \sigma^\star_{ii}& = (\sigma^{ii})^{-1} + \frac{\sigma_{1i}^2}{\sigma_{11}}\,, &i &\in [d]\setminus \{1\}\,,\\
          \sigma^\star_{ij}& = \frac{\sigma_{1i}\sigma_{1j}}{\sigma_{11}}\,, &i\neq j &\in [d]\setminus \{1\}\,.
      \end{aligned}
  \end{equation}
\end{theorem}

In this case, recall that $\nabla^2 V = \Sigma^{-1}$ and thus $\partial_{ij}V = \sigma^{ij}$. The assumptions \eqref{assum:R}, \eqref{assum:P1}, and \eqref{assum:P2} are trivially satisfied. Although assumption~\eqref{assum:RD} is stronger than the Schur complement condition for $\Sigma \succ 0$ and need not hold in general,
 $\pi^\star_1$ is still log-concave and log-smooth by Gaussianity. 

We now offer a probabilistic interpretation of these results and how the SSVI minimizer compares structurally to the MFVI minimizer. When $m=0$, for $(Z_1,\ldots, Z_d) \sim \pi^\star$, 
\cref{thm: Gaussian} states that $Z_i$ can be expressed as a linear combination of $Z_1$ and exogenous variables $\xi_i \iid \cN \left(0, 1\right)$ independent of $Z_1$: 
\begin{equation}\label{eq: representation of Yi conditioned on Y1 SSVI}
\begin{aligned}
Z_1 = \sqrt{\sigma_{11}}\,\xi_1 \,, \qquad Z_i =  \dfrac{\sigma_{1i}}{\sigma_{11}}\,Z_1 + \dfrac{1}{\sqrt{\sigma^{ii}}}\,\xi_i\,,\quad \text{for } i \geq 2\,. 
\end{aligned}
\end{equation}
On the other hand, let $\bar\pi$ be the mean-field approximation to $\pi$ with $(\bar Z_1,\dotsc, \bar Z_d) \sim \bar\pi$. \citet[Theorem~1.1]{Lacker2024} show that each $\bar Z_i$ is precisely a rescaled version of this exogenous variable $\xi_i$, namely that
    \begin{equation*}
        \bar Z_i = \dfrac{1}{\sqrt{\sigma^{ii}}}\,\xi_i\,,\qquad \text{for}~i\in [d]\,.
    \end{equation*}
SSVI improves upon MFVI by adding the correction term $(\sigma_{1i} / \sigma_{11})\, Z_1$ to the MFVI minimizer. The following corollary quantifies this improvement, showing that the gain is precisely given by the logarithmic ratio of the variances of the root variable under the respective minimizers.

\begin{corollary}\label{coro: star-graph outperforms MFVI in optimality gap}
In the setting of \cref{thm: Gaussian}, it holds that
\begin{equation*}\label{eq: Kl difference between StarVI and MFVI}
\begin{aligned}
     \inf_{\mu \in \cC_{\rm star}}\kl{\mu}{\pi} - \inf_{\mu \in \cP(\R)^{\otimes d}}\kl{\mu}{\pi} &= -\dfrac{1}{2}\log\left(\dfrac{\sigma_{11}}{(\sigma^{11})^{-1}} \right)\leq 0\,.
\end{aligned}
\end{equation*}
\end{corollary}

\subsubsection{Bayesian generalized linear models with various priors} \label{example:glm}

We now specialize our approximation results to a family of scenarios where the posterior is defined in terms of a generalized linear model (GLM). As GLMs are standard tools for statistical inference, we only briefly outline the basic principles and terminology required; see \cite{McCullagh2019} for more details.

Given a collection of covariate-response pairs $\{(X_i, Y_i)\}_{i=1}^n$, where $X_i \in \mathbb{R}^d$ and $Y_i \in \cY$, a GLM posits the following data-generating process
\begin{align}\label{glm}
    Y_i\mid \beta \sim {\msf{EF}}[X_i^\top \beta; \sigma]\,,
\end{align}
where ${\msf {EF}}[\cdot;\sigma]$ is defined as a probability measure on $\cY$, called an \emph{exponential family}
\begin{align*}
    \frac{\dd {\msf{EF}}[X_i^\top \beta; \sigma]}{\dd \msf m}(y) = \exp\Bigl(\frac{y X_i^\top \beta - \psi(X_i^\top \beta)}{c(\sigma)}\Bigr)\,,
\end{align*}
where $\psi$ plays the role of the normalizing constant (called the log-partition function), and $\msf m$ can be discrete (when $\cY$ is discrete) or the Lebesgue measure (when $\cY$ is a subset of $\R$).

In the Bayesian version of GLMs, we further posit a prior on the coefficients $\beta \in \R^d$
\begin{equation} \label{glm-prior}
        \beta \sim \pi_0\,.
\end{equation}
Conditioning on the covariates,~\eqref{glm}--\eqref{glm-prior} defines the data generative process for $y_{1:n}$. It is worth emphasizing that many standard regression models are special cases of GLMs.\footnote{For instance: Linear regression corresponds to the case $\cY = \R$, $\psi(t) = t^2/2$, and $c(\sigma) = \sigma^2$. Logistic regression corresponds to the binary class label in $\cY = \{0, 1\}$, $\psi(t) = \log(1 + \exp(t))$, $c(\sigma) = 1$. Poisson regression corresponds to $\cY = \N$, $\psi(t) = \exp(t)$, and $c(\sigma) = 1$.}

In the analysis below, we consider two families of hierarchical priors for $\pi_0$: (i) location-scale priors, and (ii) spike-and-slab priors with one-dimensional debiasing.

\paragraph*{Location-family prior} \label{example:glm-location}
Consider a location family of priors of the form: 
\begin{equation} \label{location-family-prior}
\beta_1, \dotsc, \beta_d \iid   \pi_0\left(\cdot \mid \vartheta \right) = \exp(-\varrho(\cdot - \vartheta))\,, \qquad \vartheta \sim \pi_0 = \exp(- \rmg)  \,, 
\end{equation}
where $\varrho: \R \to \R$ and $\rmg: \R \to \R$ are some potential functions. For example, $\varrho(t) = t^2/(2s^2)$ corresponds to a Gaussian distribution, while $\varrho(t) = t/s + 2 \log \left(1 + \exp\left(-t/s \right) \right)$ corresponds to a logistic distribution.

Denote by $A  \deq  \frac{1}{c(\sigma)}\, \bX^\top\bX$ and $w  \deq  \frac{1}{c(\sigma)}\, \bX^\top \by$, where $\bX = \left[X_1, \dotsc, X_n\right]^\top$ and $\by  \deq  \left[Y_1, \dotsc, Y_n\right]^\top$. 

The full posterior $\pi_{\rm loc}(\dd \vartheta, \dd \beta) \propto \exp(-V_{\rm loc}(\vartheta,\beta))\,\dd \vartheta\,\dd \beta$ is given by:
\begin{equation}\label{eq:pot_loc}
V_{\rm loc}(\vartheta, \beta) =  \rmg(\vartheta) - w^\top \beta +  \sum_{i = 1}^n \frac{\psi(X_i^\top \beta)}{c(\sigma)} + \sum_{j = 1}^d \varrho(\beta_j - \vartheta)\,. 
\end{equation}
An example of a full generative process is shown in 
\cref{fig:graphical_model}(a) as a graphical model. The latent variable $\vartheta$ models a uniform shift in $\beta_{1:d}$ away from $0$. While the prior for $\left(\vartheta,\beta_{1:d}\right)$ follows a star graphical structure, the full posterior need not be a star graph. Nevertheless, we can compute a star-structured approximation and compute the approximation quality.

To state our result, we introduce a key definition. For $\beta \in \R^d$, denote the matrix $\cA(\beta) \deq  \frac{1}{c(\sigma)}\,\bX^\top D(\beta) \bX$ where $D(\beta)$ is a diagonal matrix with $D_{ii}(\beta) = \psi''(X_i^\top\beta)$ for $i = 1, \dotsc, n$.

\begin{theorem}[Approximation quality for location-family GLMs]  \label{thm:glm-location}
Let $\pi_{\rm loc} \propto \exp(-V_{\rm loc})$ where $V_{\rm loc}$ is given by \eqref{eq:pot_loc}. Assume that $b \leq \psi'' \leq B$ for some $B \geq b > 0$ and $0 \leq \varrho'' \leq R$, and that the interaction matrix satisfies $A \succeq a I$ for some $a > 0$. Suppose that $\alpha \leq \rmg'' \leq L$ for some $\alpha \in \mathbb{R}$ such that $\ell_V'  \deq  \alpha - \frac{d R^2}{ab} > 0$. Then, under the model specified by \eqref{glm}–\eqref{location-family-prior}, the star-structured variational approximation $\pi_{\rm loc}^\star$ satisfies the following bound:
\begin{equation*} \label{glm-location-oracle}
   \kl{\pi_{\rm loc}^\star}{\pi_{\rm loc}} \leq \frac{dR + L}{a^2 b^2 \ell_V'}  \,\Bigl\|\sum_{i = 1}^d \sum_{j > i} \cA_{ij}^2 \Bigr\|_{L^\infty}\,.
\end{equation*}
\end{theorem}

It is known that the log-partition function $\psi$ is convex in general exponential families \cite[Proposition 3.1]{Wainwright2008}. \cref{thm:glm-location} further assumes $\psi$ to be $b$-convex, which holds when the variance of the distribution $\msf{EF}(\cdot; \sigma)$ is uniformly lower bounded by $b$ \citep[Theorem 2.2]{Brown1986}.

\cref{thm:glm-location} appears to be the first \emph{quantitative} VI guarantee for GLMs. Our result establishes a star approximation guarantee of order $O\bigl(\|\sum_{i = 1}^d \sum_{j > i} \cA_{ij}^2 \|_{L^\infty} \bigr)$. This recovers the condition identified for the asymptotic correctness of mean-field approximation in canonical GLMs with compact support~\citep[see Definition~2 and Theorem~2 of][] {mukherjee2024naive}.
When the off-diagonal entries of $\cA(\beta)$ are uniformly small across all $\beta$, the exact posterior of the GLM is accurately approximated by a star-structured variational distribution. In the case of a linear model, we obtain the simple form $\cA(\beta) = A$ (see \cref{prop:lm-location}). 

We next specialize \cref{thm: approx gap} to a Bayesian linear model with Gaussian location prior: 
\begin{equation}  \label{eqn-lm-location}
\by = \bX \beta + \epsilon\,, \quad \epsilon \sim \cN(0, \sigma^2 I)\,, \quad \beta_1, \dotsc, \beta_d \sim \cN (\vartheta, \tau^{-2})\,, \quad \vartheta \sim \exp(- \rmg)\,.
\end{equation}
\begin{proposition} \label{prop:lm-location}
Let $A \deq \frac{1}{\sigma^2}\, \bX^\top \bX$. Assume that $a I \preceq A$ for some $a \geq 0$. Suppose that $\alpha \leq \rmg'' \leq L$ such that $\ell_V' = \alpha + \frac{d a \tau^2}{a+ \tau^2} > 0$. Under model~\eqref{eqn-lm-location}, the SSVI minimizer $\pi_{\rm loc}^\star$ satisfies:
\begin{equation} \label{LM-Gaussian-bound}
  \kl{\pi_{\rm loc}^\star}{\pi_{\rm loc}}  \leq \frac{d\tau^2 + L}{\ell_V'\, (a+\tau^2)^2}  \sum_{i = 1}^d \sum_{j > i} A_{ij}^2\,.
\end{equation}
\end{proposition}

\cref{prop:lm-location} establishes a star approximation gap of order $O( \sum_{i = 1}^d \sum_{j > i} A_{ij}^2 )$. When the off-diagonal entries of $A$ are sufficiently small, the exact posterior is accurately approximated by a star-structured variational distribution. 

For \cref{prop:lm-location} to hold, we do not require the interaction matrix $A$ to be positive definite, nor $\rmg$ to be strongly convex.  When $a = 0$, it suffices that the log-concavity parameter $\alpha > 0$ to obtain a non-trivial upper bound in~\eqref{LM-Gaussian-bound}. On the other hand, $\alpha$ could also be negative, provided that $\alpha > - \frac{d a \tau^2}{a + \tau^2}$. Moreover, \Cref{prop:lm-location} yields a slightly stronger statement than \Cref{thm:glm-location}: in \Cref{thm:glm-location} we only assume that $\varrho$ is convex, whereas in \Cref{prop:lm-location} we use the fact that $\varrho'' = \tau^2$.

\paragraph*{Spike-and-slab prior}
Another common setting in statistical inference is to impose a spike-and-slab prior on $\beta_{2:d}$ and a bias-correcting prior on $\beta_1$ \citep{Castillo2024}. To be precise, for $j \geq 2$, $\beta_j$ is drawn from either a spike distribution $\nu_0$ or a more diffuse slab distribution $\nu_1$, with a mixture weight of $\eta \in [0,1]$: 
\begin{equation} \label{spike-slab-prior}
\beta_2, \ldots, \beta_d\; \iid\; \eta  \nu_0 + (1 - \eta) \nu_1\,.
\end{equation}
Following the approach in \cite{George1993}, we specify $\nu_0, \nu_1$ as normal distributions: $\nu_0 \defeq \cN (0, \tau_0^{-2})$, $\nu_1 \defeq \cN (0, \tau_1^{-2})$ where $\tau_1^2 < \tau_0^2$. Now, conditioned on $\beta_{2:d}$, we draw $\beta_1$ from the following prior: 
\begin{equation} \label{debiased-prior}
    \beta_1 \mid \beta_{2:d} \sim \exp\Bigl(-\rmg \Bigl(\cdot + \sum_{j = 2}^d \gamma_j \beta_j \Bigr) \Bigr)\,, 
\end{equation}
where $\gamma_j \deq \bX_1^\top \bX_j / \left\|\bX_1 \right\|_2^2$ is chosen to be the rescaled correlation between covariates $1$ and $j$; here $\bX_j$ is the $j^{\text{th}}$ column of $\bX$. For linear regression, the resulting posterior is shown to debias $\beta_1$ \citep{Yang2019debiased, Castillo2024}. For GLMs, a similar debiasing approach exists in the frequentist setting \citep{VanGeer2014}. \cref{fig:graphical_model}~(b) shows the full generative process, where the full posterior is given by $\pi_{\rm ss} \propto \exp(-V_{\rm ss})$ with
\begin{equation} \label{eq:pot_ss}
    V_{\rm ss}(\beta) = \sum_{i = 1}^n \frac{\psi(X_i^\top \beta)}{c(\sigma)} - w^\top \beta - \sum_{j = 2}^d \log \left( \eta  \nu_0(\beta_j) + (1 - \eta) \nu_1(\beta_j) \right) + \rmg\Bigl(\beta_1 + \sum_{j = 2}^d \gamma_j \beta_j \Bigr)\,.
\end{equation}

\textbf{Comparison with \cite{Castillo2024}.} The previous work of \cite{Castillo2024} studies a high-dimensional linear model with a prior where the distribution of $\beta_1 \mid \beta_{2:d}$ is the same as \eqref{debiased-prior} but replaces the continuous spike-and-slab prior~\eqref{spike-slab-prior} with a discrete spike-and-slab prior.
 
 For inference, \cite{Castillo2024} applies MFVI to $\beta_{2:d}$ through a reparameterization scheme $\beta_1 \mapsto \beta_1' \deq \beta_1 + \sum_{j = 2}^d \gamma_j \beta_j$ which decorrelates the posterior of $\beta_1'$ and $\beta_{2:d}$. The approach, however, has two limitations: (1) the decorrelation approach fails in GLMs, where $\psi$ induces non-linear dependencies between $\beta_1$ and $\beta_{2:d}$ that cannot be easily orthogonalized, and (2) MFVI restricts the marginal law of $\beta_{2:d}$ to be a product measure. 
These are both limitations which we sidestep with our theory.

\begin{theorem}[Approximation quality for spike-and-slab GLMs] \label{thm:glm-ss}
Let $\pi_{\rm ss} \propto \exp(-V_{\rm ss})$ where $V_{\rm ss}$ is given by \eqref{eq:pot_ss}. Assume that $b \leq \psi'' \leq B$ for some $B \geq b > 0$, and that the interaction matrix satisfies $ a_d I \preceq A \preceq a_1 I$ for some $a_1 \geq a_d > 0$. Assume that $\alpha \leq \rmg'' \leq L$ such that:
\begin{enumerate}[label = (\roman*)]
    \item $\ell_V'  \deq  b A_{11} + \alpha - \frac{2\sum_{j = 2}^d \gamma_j^2 L^2}{b a_d} - \frac{2\,\left(B a_1 - b a_d \right)\, \left( B A_{11}- b a_d \right)}{b a_d}> 0$; and
    \item $\ell_V \deq b a_d + \tau_1^2 - 2(\tau_0^2 - \tau_1^2)
    \log \bigl(1 + \frac{\eta \tau_0}{(1-\eta)e \tau_1}\bigr) > 0$.
\end{enumerate}
    Then, under the model specified in \eqref{glm}–\eqref{spike-slab-prior}–\eqref{debiased-prior},  the star-structured variational approximation $\pi_{\rm ss}^\star$ satisfies the following bound:
\begin{equation*}
\kl{\pi_{\rm ss}^\star}{\pi_{\rm ss}} \leq \frac{2\,(B A_{11} + L)}{\ell_V^2 \ell_V'}\, \biggl(\Bigl\|\sum_{i = 2}^d \sum_{j > i} \cA_{ij}^2 \Bigr\|_{L^\infty} +  L^2\sum_{i = 2}^d \sum_{j > i} \gamma_i^2 \gamma_j^2 \biggr)\,.
\end{equation*}
\end{theorem}
The condition $\ell_V' > 0$ is readily satisfied when $a_1$ and $a_d$ are close and $\alpha$ is sufficiently large. In the special case where $A_{1j} = 0$ for all $j \geq 2$, we have $\gamma_j = 0$ by the definition of $\gamma_j$, thus $\ell_V' > 0$ holds when we specify the prior $\rmg$ such that $\alpha   > \frac{2\,\left(B a_1 - b a_d \right)\, \left( B A_{11}- b a_d \right)}{b a_d} - bA_{11}$. 
The condition $\ell_V > 0$ is guaranteed when $a_d$ grows with sample size so that it dominates the remainder term which only depends on the fixed prior. This behavior is typical in well-behaved models, as the next result (\Cref{prop:Gaussian-Ensemble}) illustrates.

We now specialize our results in the setting of \cite{Castillo2024} where 
\begin{equation*}
\by = \bX \beta + \epsilon\,, \quad \epsilon \sim \cN(0, I)\,, 
\end{equation*}
and each $\beta_i$, $i \geq 2$, is drawn i.i.d.\ from a continuous spike-and-slab prior: 
\begin{equation*}  \label{eqn-lm-spike-slab}
\begin{aligned}
\beta_1 \mid \beta_{2:d} \sim \cN \Bigl(- \sum_{j = 2}^d \gamma_j \beta_j , \tau^{-2} \Bigr),  \quad \beta_2, \cdots, \beta_d \iid \eta \nu_0 + (1 - \eta) \nu_1\,. 
\end{aligned}
\end{equation*}
We say that a random matrix $\bX \in \R^{n\times d}$ is from the $\Sigma$-Gaussian ensemble if the rows $x_i$ of $\bX$ are drawn i.i.d.\ from the $\cN(0,\Sigma)$ distribution. 
The following guarantee holds:

\begin{proposition}\label{prop:Gaussian-Ensemble}
    Let $A = \bX^\top \bX$ where $\bX$ is from the $\Sigma$-Gaussian ensemble.  Let $\lambda_d,\lambda_1$ denote the smallest and largest eigenvalues of $\Sigma$. Suppose that $d,n\to\infty$ and $d/n = o(1)$. Suppose further that $\lambda_1 = O(1)$, $\lambda_1/\lambda_d = o(\sqrt{n/d})$, and
\begin{equation} \label{assum:Gaussian Ensemble}
        \ell_V'  \deq  \frac{n}{2}\,\Bigl(\Sigma_{11} - \frac{8\sum_{j = 2}^d |\Sigma_{1j}|^2}{\lambda_d} \Bigr) > 0\,. 
\end{equation}
Then, the approximation error of $\pi_{\rm ss}^\star$ satisfies
\begin{equation*}
  \kl{\pi_{\rm ss}^\star}{\pi_{\rm ss}}  \leq \frac{8\left(A_{11} +\tau^2\right)}{  n^2 \ell_V' \lambda_d^2}  \sum_{i = 2}^d \sum_{j > i} \bigl(A_{ij}^2 + \tau^4 \gamma_i^2\gamma_j^2\bigr)\,,
  \end{equation*}
with high probability. 
\end{proposition}

Condition~\eqref{assum:Gaussian Ensemble} elucidates the root domination criteria: being diagonally dominant in a Gaussian ensemble regression means that the marginal variance of the root node $\Sigma_{11}$ dominates the marginal covariances between the root and leaf nodes, in the sense that $\Sigma_{11} > (4/\lambda_d)\sum_{j = 2}^d |\Sigma_{1j}|^2$. A larger gap implies that the SSVI approximation is more precise.  

\section{Computational guarantees via optimal transport}\label{sec:computation}

We now turn our attention to computational aspects of solving \eqref{eq:SSVI_main}. We find it helpful to compare to the mean-field setting as in \cref{sec:SSVI_main}.

A popular algorithm for solving the mean-field variational inference problem is known as \emph{coordinate-ascent variational inference}, or CAVI \citep{Blei2017}. For a given posterior $\pi$, this algorithm performs the following updates
\begin{align}\label{eq:cavi_iter}
    \pi^{(k+1)}_j = \argmin_{\mu \in \cP(\R)}\KL(\mu \otimes \pi^{(k)}_{-j} \mmid \pi)\,,
\end{align}
and we cycle through the coordinates $j \in \{1,\ldots,d\}$. Despite forming the backbone of Bayesian computation, convergence guarantees for CAVI have only been recently established under the same assumptions as our work, namely log-smoothness and strong log-concavity of $\pi$ \citep{arnese2024convergence,lavenant2024convergence, BhaPatYan25CAVI}.\footnote{Interestingly, \cite{lavenant2024convergence} establish improved convergence rates for CAVI under a random coordinate scan approach than a deterministic scheme.} The iterates \eqref{eq:cavi_iter} are only available in closed form if we further restrict the optimization to take place over a conjugate families, which can be far from typical in practical applications \citep{wang2013variational}.

The algorithmic story for structured variational inference is quite different. Even in the star-structured setting, the equivalence to a dynamic programming principle makes it difficult to derive an algorithm similar to the CAVI update scheme. For general structured VI, it is popular to perform gradient descent on a prespecified parametric family; see \citet{hoffman2015structured}. In this section, we provide computational guarantees for computing the SSVI minimizer using ideas from optimal transport \citep{villani2009optimal, CheNilRig25OT} and non-parametric approximation theory \citep{wasserman2006all,Tsybakov2008}. In \cref{sec:lift_maps}, we show that \eqref{eq:SSVI_main} is equivalent to a convex optimization problem~\eqref{KR-obj} that occurs at the level of transport maps. This equivalence provides an alternative proof of the existence and uniqueness of the SSVI minimizer; see \cref{sec:existence_uniqueness}. In \cref{sec:Caff_regularity}, we derive a Caffarelli-type regularity result for the optimizer of~\eqref{KR-obj}, focusing in particular on its regularity with respect to the root variable. \cref{sec:gradient_descent} is devoted to the computational method for finding a near-optimal transport map by approximating~\eqref{KR-obj} via a finite-dimensional optimization problem. We present both theoretical guarantees for the approximation and a projected gradient descent algorithm with convergence guarantees.

\subsection{From star-structured measures to star-separable maps}\label{sec:lift_maps}
Our approach is inspired by the popular technique of ``normalizing flows'' in the machine learning literature, which suggests that learning vector-valued functions between measures is easier than learning the density of a measure itself; see the review article by \citet{kobyzev2020normalizing}. We henceforth call $T : \R^d \to \R^d$ a transport map between two measures $\mu,\nu \in \cP(\R^d)$ if, for $X \sim \mu$, $T(X)\sim\nu$. Our first objective is to rewrite \eqref{eq:SSVI_main} as an optimization problem over suitable transport maps.

To this end, let $\rho = \bigotimes_{i=1}^d \rho_i \in \cP_{\rm ac}(\R)^{\otimes d}$ be a fixed reference measure---we will always take $\rho=\cN(0,I)$. A transport map $T:\R^d\to\R^d$ emanating from $\rho$ is called \emph{star-separable} if there exists functions $T_1:\R\to\R$ and $T_{i}:\R\times \R \to \R$ (for $i \in \{2,\ldots,d\}$) such that
\begin{align}\label{eq:starsep_map}
    x \mapsto T(x_1,\ldots,x_d) = (z_1,T_2(x_2 \mid z_1),\ldots,T_d(x_d \mid z_1))\quad \text{and}\quad z_1 = T_1(x_1)\,.
\end{align}
Additionally, we require that $T_1$ (resp.\ $T_i$ for $i\in\{2,\ldots,d\}$) to be \textit{(strictly)} increasing in $x_1$ (resp.\ in $x_i$ for any fixed $z_1\in\R$). Let $\cT_{\rm star}$ be the collection of star-separable maps emanating from $\rho$.\footnote{We omit the explicit dependence on $\rho$ for ease of notation.} The set $\cT_{\rm star}$ is a special class of Knothe--Rosenblatt (KR) maps \citep{rosenblatt1952remarks, knothe1957contributions}, which are coordinate-wise transformations depending on the structure of preceding coordinates. We will explore this connection in greater detail in~\cref{sec:beyond_star}.

We claim that the original SSVI problem, an optimization over star-structured probability measures, is equivalent to the following optimization which takes place over star-separable maps:
\begin{equation} \tag{$\sf T$-$\msf{SSVI}$} \label{KR-obj}
    \inf_{T \in \cT_{\rm star}} \kl{T_\# \rho}{\pi}. 
\end{equation}
In short, this equivalence implies that $\pi^\star$ is a solution to \eqref{eq:SSVI_main} if and only if there exists $T^\star \in \cT_{\rm star}$ which solves \eqref{KR-obj}, whence
\begin{align*}
    \pi^\star = (T^\star)_\#\rho\,.
\end{align*}
Indeed, as $\pi^\star \in \cP_{2,{\rm ac}}(\R^d)$ by \cref{thm: star graph regularity}, we may restrict the minimization~\eqref{eq:SSVI_main} to the set $\cP_{2,{\rm ac}}(\R^d)\cap \cC_{\rm star}$, and thus, without loss of generality, we may assume $\cC_{\rm star}$ is a subset of $\cP_{2,{\rm ac}}(\R^d)$. Therefore, Brenier's theorem~\citep[see, e.g.,][Theorem~2.32]{Villani2003} implies $\mu \in \cC_{\rm star}$ if and only if $\mu = T_\# \rho$ for some $T \in \cT_{\rm star}$, which establishes the equivalence between \eqref{eq:SSVI_main} and \eqref{KR-obj}.

This lifting yields a remarkable property: although $\cC_{\rm star}$ is not convex, its image $\cT_{\rm star}$ forms a convex subset of $L^2(\rho)$. We demonstrate this in the following proposition.
\begin{proposition}\label{lemma: convex minimizing set} The set of star-separable transport maps 
 $\cT_{\rm star}$ is convex. 
\end{proposition}

This result highlights an interesting computational paradigm. While \cref{thm:existence} guarantees the existence of a unique minimizer to \eqref{eq:SSVI_main}, it does not provide a means for computing this minimizer. In fact, the set $\cC_{\rm{star}}$ is not convex under linear geometry --- indeed, a mixture of two star-structured graphical models need not itself be star-structured. In contrast, the equivalence between \eqref{eq:SSVI_main} and \eqref{KR-obj}, combined with \cref{lemma: convex minimizing set}, reveals the possibility of a convex optimization scheme over the space of star-separable transport maps. In the sequel, we will flesh out this idea by developing the theoretical properties of the class $\cT_{\rm{star}}$, and conclude by presenting a gradient-based algorithm that is computationally tractable and comes with convergence guarantees.

\subsection{Existence and uniqueness of the minimizer via (strong) convexity}\label{sec:existence_uniqueness}

We now introduce a notion of distance over $\cT_{\rm star}$. Inspired by the literature on \emph{linearized} optimal transport~\citep{wang2013linear}, a natural candidate is the $L^2(\rho)$ distance $ \|T - \tilde T\|_{L^2(\rho)}$ between two maps $T,\tilde{T} \in \cT_{\rm star}$. This metric dominates the well-known adapted Wasserstein distance, which preserves conditional independence under limits (see \cref{app:proofs_comp} for a discussion). In contrast, conditional independence is not preserved under weak convergence; see \citet{Lauritzen1996} for a counterexample. Thus, the topology induced by the $L^2(\rho)$ distance is particularly well-suited to our setting, as it aligns with the structural properties of conditional distributions that underlie SSVI\@.

The above preparation brings us to the following theorem, which highlights the requirements of strong log-concavity of $\pi$ for minimizing \eqref{KR-obj}, just as it was necessary for \eqref{eq:SSVI_main}.

\begin{theorem}[Existence and uniqueness]\label{thm:convexity}
Suppose \eqref{ass:slc} holds, and let $\cT \sse \cT_{\rm star}$ be a convex set. Then
\begin{enumerate}[label = (\roman*)]
    \item the functional $T \mapsto \kl{T_\# \rho}{\pi} $ is $\alpha$-convex over $\cT$, and
    \item there exists a unique minimizer to \eqref{KR-obj}.
\end{enumerate} 
\end{theorem}
To understand this claim, one can begin by expanding the objective \eqref{KR-obj} for any $T \in \cT \subseteq \cT_{\rm star}$,
\begin{align}\label{eq:KR_expanded}
    \KL(T_\#\rho \mmid \pi) = \int V(T(x))\,  \rho(\dd x)  - \int \log \det DT(x)\, \rho(\dd x) + \textsc{const}\,,
\end{align}
where the constant consists of the entropy of $\rho$ and the unknown normalizing constant of $\pi$. We see two terms: one involving the potential function $V$ and the other involving the eigenvalues of the Jacobian of the transport map $T$. The following lemma explains the properties of the first term.
\begin{lemma}
Let $\cT \subset \cT_{\rm star}$ be a convex set and $V: \R^d \to \R^d$ which is $\alpha$-convex and $\beta$-smooth in the usual sense. Then the functional $T \mapsto \int V(T(x)) \, \rho(\dd x)$ is $\alpha$-convex and $\beta$-smooth over $\cT$.
\end{lemma}

For the second term, note that for any $T \in \cT \subset \cT_{\rm star}$, the Jacobian $DT$ is a lower triangular matrix with non-negative diagonals, thus
\begin{align*}
    T \mapsto - \int_{\R^d} \log \det (DT(x)) \, \rho(\dd x)
\end{align*}
is a \emph{convex} functional. Combining these claims, one arrives at the first conclusion of \cref{thm:convexity} that $T \mapsto \kl{T_\# \rho}{\pi} $ is $\alpha$-convex over $\cT$. Rigorous proofs are found in \cref{proof:thm:convexity}.

Remark that the statement of \cref{thm:convexity} holds true for arbitrary convex subsets $\cT \subseteq \cT_{\rm star}$. This point will be crucial in \cref{sec:gradient_descent}, where we will choose a particular subset of the family of star-separable maps to derive a projected gradient method.

\subsection{Theoretical properties of star-separable maps}\label{sec:Caff_regularity}

Before diving into the final algorithmic details of our approach, we investigate some novel theoretical aspects of these transport maps, which are crucial for establishing the computational guarantees.

The regularity of the optimal star-separable map builds on a pointwise stability estimate of the map with respect to the root variable. This, in turn, requires a strengthened version of the root domination criteria \eqref{assum:RD}: for all $i \geq 2$,
  \begin{equation} \tag{\textbf{RD}$+$} \label{assum:RD+}
  \begin{aligned}
       \Bigl\|\sum_{j \geq 2,\, j\neq i} |\partial_{ij} V| \Bigr\|_{L^\infty} < \ell_V\,, \qquad  \|\partial_{1i} V\|_{L^\infty} <\infty\,,  \qquad  \text{and}  \\
  \partial_{11}  V- \frac{(d -1) \max_{i \geq 2}\|\partial_{1i} V\|_{L^\infty}^2}{\ell_V- \max_{i \geq 2} \|\sum_{j \geq 2,\, j\neq i} |\partial_{ij} V|\|_{L^\infty}} \geq \ell_V'\,.
  \end{aligned}
\end{equation} 
It is easy to see that \eqref{assum:RD+} implies \eqref{assum:RD} as 
\begin{equation*}
\begin{aligned}
\frac{(d - 1) \max_{i \geq 2}\|\partial_{1i}V\|^2_{L^\infty}}{\ell_V- \max_{i \geq 2} \|\sum_{j \geq 2,\, j\neq i} |\partial_{ij}V|\|_{L^\infty}}  \geq \frac{\sum_{j = 2}^d \|\partial_{1j}V\|^2_{L^\infty}}{\ell_V}\,. 
\end{aligned}
\end{equation*}
Therefore, all results derived in \cref{sec:SSVI_main} concerning the regularity of the SSVI minimizer $\pi^\star$ still hold under~\eqref{assum:RD+}. Moreover, \eqref{assum:RD+} enables us to derive \emph{a priori} bounds on the derivatives of the density of the SSVI minimizer, as well as on those of the optimal star-separable map, as detailed in \cref{lemma: log density mix derivative bound} and \cref{thm:star_caff}.

A fully rigorous analysis would require establishing the second-order differentiability of both $\rmq^\star$ and $T^\star$, which requires arguments in the style of \citet{caffarelli1992boundary,caffarelli1996boundary} and lies beyond the scope of this work. In the present analysis, we assume that $\rmq^\star$ and $T^\star$ are twice continuously differentiable, and focus on controlling the higher-order derivatives of the optimal star-separable map, which is crucial for establishing the desired computational guarantees in \cref{sec:gradient_descent}.

We begin with the following lemma, which provides a bound on the mixed derivative of $\rmq_i^\star$ via a self-bounding argument under \eqref{assum:RD+}. This bound serves as the cornerstone for establishing control over the derivative of the optimal star-separable map with respect to the root variable.

\begin{lemma}[Mixed derivative bound]\label{lemma: log density mix derivative bound}
Let \eqref{assum:R}, \eqref{assum:P1}, \eqref{assum:P2} and \eqref{assum:RD+} be satisfied.  Let $$L \deq  \max_{i \geq 2}\sup_{z_1,z_i}|\partial_1\partial_i \log \rmq_i^\star(z_i\mid z_1)|\,.$$  Then 
\begin{equation}\label{eq: self-bounding result}
        L \leq \overline{L} \defeq \frac{\ell_V\max_{i \geq 2}\|\partial_{1i}V\|_{L^\infty}}{\ell_V- \max_{i \geq 2} \|\sum_{j \geq 2,\, j\neq i} |\partial_{ij}V|\|_{L^\infty}} < +\infty\,. 
\end{equation}
\end{lemma}
 Our main contributions in this subsection culminate in the following theorem, which states higher-order control on the optimal star-separable map; its proof appears in \cref{proof:thm:star_caff}. 

\begin{theorem}[Regularity for star-separable maps]\label{thm:star_caff}
Assume \eqref{assum:R}, \eqref{assum:P1}, and \eqref{assum:P2} hold. Let $T^\star$ be the unique minimizer of \eqref{KR-obj} over $\cT_{\rm star}$, written
\begin{align*}
    x \mapsto T^\star(x) = (z_1, T_2^\star(x_2 \mid z_1), \ldots, T_d^\star(x_d \mid z_1))\,, \qquad \text{with} \qquad z_1 = T_1^\star(x_1)\,.
\end{align*}
If $T^\star$ is twice-differentiable, then:
\begin{enumerate}[label=(\roman*)]
    \item If \eqref{assum:RD} is satisfied, then for $z_1 \in \mathbb{R}$ and $i \in \{2,\ldots,d\}$,
    \begin{align*}
        \sqrt{1/L_V'} \leq \left| \partial_{x_1} T_1^\star(\cdot) \right| \leq \sqrt{1/\ell_V'}\,, \qquad 
        \sqrt{1/L_V} \leq \left| \partial_{x_i} T_i^\star(\cdot \mid z_1) \right| \leq \sqrt{1/\ell_V}\,,
    \end{align*}
    and
    \begin{align*}
        | \partial^2_{x_1} T_1^\star(\cdot) | \leq \frac{L_V'}{(\ell_V')^{3/2}}\,(1 + |\cdot|)\,, \qquad 
        \text{and} \qquad 
        | \partial^2_{x_i} T_i^\star(\cdot \mid z_1) | \leq \frac{L_V}{\ell_V^{3/2}}\,(1 + |\cdot|)\,.
    \end{align*}
    \item If, in addition, \eqref{assum:RD+} holds, then for all $x_i \in \mathbb{R}$,
    \begin{equation*}
        \left| \partial_{z_1} T_i^\star(x_i \mid \cdot) \right| \leq \frac{\overline{L}}{\ell_V}\,, \qquad 
        \text{and} \qquad 
     \left| \partial_{x_i} \partial_{z_1} T_i^\star(x_i \mid \cdot) \right| \lesssim \frac{\overline{L} L_V}{\ell_V^2}\,(1 + |x_i|)\,.
    \end{equation*}
\end{enumerate}
\end{theorem}

The first part of \cref{thm:star_caff} can be viewed as an analogue of the celebrated Caffarelli contraction theorem \citep{Caffarelli2000}, specialized to star-separable transport maps. To establish higher-order regularity with respect to the leaf variable, we adopt a bootstrap strategy similar to that of \citet{JiaChePoo25MFVI}.

The second and third parts of \cref{thm:star_caff} concern the regularity of the optimal star-separable transport map with respect to the root variable. The first-order estimate in \cref{thm:star_caff}~(ii) builds upon the bound for the mixed derivative of $\rmq_i^\star$ established in \cref{lemma: log density mix derivative bound}, and leverages recent advances in transport-information inequalities \citep{KhuMaaPed25LInf}. Higher-order estimates are then obtained via a bootstrap argument.

Since assumption~\eqref{assum:RD+} is stronger than~\eqref{assum:RD}, the results in \cref{thm:star_caff}~(i) remain valid under~\eqref{assum:RD+}.

\subsection{A projected gradient descent algorithm for solving \ref{KR-obj}}\label{sec:gradient_descent}

In \cref{sec:lift_maps} and \cref{sec:existence_uniqueness}, we showed how to lift the infinite-dimensional space of star-structured probability measures to the space of star-separable transport maps, over which the objective becomes convex in the standard sense. However, the space $\cT_{\text{star}}$ in \eqref{KR-obj} is still prohibitively large to optimize over. We therefore approximate the set $\cT_{\rm star}$ by a finite-dimensional approximation, and show how to compute the minimizer over this set via first-order methods.

Existing gradient-based approaches to structured VI rely on parametrizing the variational distributions \citep{archer2015black,hoffman2015structured,ranganath2016hierarchical,Tan2018,Tan2021VB}. To approximate $\pi^\star$, one posits a parametric family $\{\pi_\theta \mid \theta \in \Theta\} \subset \cC_{\text{star}}$, often chosen as an exponential family, and optimizes the objective function $\theta \mapsto \kl{\pi_\theta}{\pi}$ using gradient descent. A key limitation is that the map $\theta \mapsto \kl{\pi_\theta}{\pi}$ is generally non-convex. As a result, despite empirical successes, such “black-box” VI methods lack theoretical guarantees and can be highly sensitive to initialization \citep{Blei2014BBVI}. 

While existing algorithms face theoretical hurdles, we show that by carefully approximating the set of transport maps $\cT_{\rm{star}}$, one can retain convexity of the objective and derive convergence guarantees. To approximate the infinite-dimensional problem \eqref{KR-obj} with a finite-dimensional one, we consider a parametrized set of star-separable transport maps $\cT_\Theta \subset \cT_{\rm{star}}$ given by
\begin{align*}\label{eq:T_theta}
    \cT_\Theta \defeq \{T_\theta \mid \theta \in \Theta\}\,,
\end{align*}
where $\Theta \subseteq \R^p$ and $p \in \N$. The parametrized problem amounts to optimizing the following objective :
\begin{equation}\label{eq:parametrized}\tag{$\msf{P}$-$\msf{SSVI}$}
    \hat \theta \in \argmin_{\theta \in \Theta} \kl{(T_\theta)_\# \rho}{\pi}. 
\end{equation}

When $\Theta$ and $\cT_\Theta$ are convex and the parametrization $\theta \mapsto T_\theta$ is affine, \cref{thm:convexity} implies that the parametrized problem~\eqref{eq:parametrized} is a strongly convex optimization problem over the finite-dimensional space $\Theta$ equipped with the norm $\|\theta\|_\Theta \deq \|T_\theta\|_{L^2(\rho)}$, and thus the optimizer $\hat \theta$ is uniquely defined.

The following theorem shows that one can find a parameter set $\Theta$ that is sufficiently large---yet still finite-dimensional---such that $T_{\hat{\theta}}$ approximates $T^\star$ up to the desired accuracy. As mentioned in \Cref{sec:existence_uniqueness}, we use the $L^2(\rho)$ norm as the evaluation metric, as it captures the graphical structure inherent to $\cC_{\text{star}}$.

\begin{theorem}[Existence of a scalable convex parametrization]\label{thm: optimality gap informal}
Assume \eqref{assum:R}, \eqref{assum:P1}, \eqref{assum:P2} and \eqref{assum:RD+} hold. Assume further that 
\begin{equation}\tag{\textbf{GR}}\label{assum:GR}
    |\partial^2_{z_1} T_i^{\star}(x_i\mid \cdot)|\le \LGR\,  (1+ |x_i|)^{\gamma}\,,\quad \text{for some}\,\, \LGR, \gamma>0\,.
\end{equation}
 Then, for any $\epsilon > 0$, there exists a closed convex set $\Theta$ of dimension $O((d^2/\epsilon^2) \log(d/\epsilon^2))$ and an affine parametrization $\theta \mapsto T_\theta$, where the implied constant depends polynomially on $\ell_V, L_V, \ell_V', L_V', \overline{L}, \LGR, M_{2\gamma}(\rho_1)$ and their inverses, where $M_{2\gamma}(\rho_1)$ is the $2\gamma$-th absolute moment of $\rho_1$, such that the~\ref{eq:parametrized} minimizer $\hat{\theta}$ satisfies
\begin{equation*}
    \|T_{\hat{\theta}} - T^{\star}\|_{L_2(\rho)} \leq \epsilon\,.
\end{equation*}
\end{theorem}

\begin{remark}
In the theorem above, we assume that the problem parameters $\ell_V$, $L_V$, etc.\ are dimension-free and we only report the final dimension dependence. This is purely done for the sake of exposition. Note that in order for, e.g.,~\eqref{assum:RD+} to hold with a dimension-free constant, it requires $|\partial_{1j} V| = O(d^{-1/2})$, which can be restrictive.
The complete dependence on the implied constants in \Cref{thm: optimality gap informal}, which is somewhat complicated, is explicitly derived in the proof.
\end{remark} 

The condition~\eqref{assum:GR} controls the growth rate of the second-order derivative of $T_i^\star$ with respect to the root variable. Notably, when $\pi^\star_i(\cdot \mid z_1) = \cN(z_1, 1)$, then \eqref{assum:GR} is satisfied with $\gamma=0$ as $\partial_{z}^2 T_i^\star(x_i\mid z)=0$. 
Whereas we have derived a priori bounds for the other derivatives of $T_i^\star$ in~\cref{thm:star_caff}, the second derivative w.r.t.\ $z$ is particularly difficult to control and hence we must adopt~\eqref{assum:GR} as an assumption; we leave it for future work to obtain an \emph{a priori} bound for this quantity as well.

\cref{thm: optimality gap informal} is a direct consequence of \cref{thm: optimality gap} in \cref{sec:computational guarantee}, where we provide an explicit construction of the parameter set $\Theta$ based on a fixed dictionary of piecewise linear maps $\cM$ and an explicit expression of the hidden constants. Our proposed dictionary $\cM$ consists of piecewise linear maps of two variables, designed to approximate $T_i^\star$ for $i \geq 2$; see details in \Cref{subsec:dictionary}. A similar dictionary construction for approximating univariate maps is introduced in \citet{JiaChePoo25MFVI}.

The proof of \cref{thm: optimality gap} builds on two non-trivial steps. First, we construct an “oracle” approximator $\widehat{T} \in \cT_\Theta$ that is close to $T^\star$. To obtain such an approximation, we exploit the regularity results established in \cref{thm:star_caff}, particularly those concerning the dependence of $T^\star$ on the root variable. Second, we verify that the computed optimizer $T_{\hat{\theta}}$ is close to $\widehat{T}$. This is achieved by leveraging the geometry of the parametrized problem~\eqref{eq:parametrized} and the optimality of both $\pi^\star$ and $\pi^\star_{\hat{\theta}}$.

We now turn to the projected gradient descent algorithm for computing the minimizer $\hat{\theta}$. Starting from $\theta^{(0)}$, the iterates $\theta^{(t)}$, $t\in\N$, are defined recursively via
\[
\theta^{(t+1)} = {\rm Proj}_{\Theta, \|\cdot\|_\Theta}\bigl(\theta^{(t)} -h \nabla_\theta \kl{(T_{\theta^{(t)}})_\# \rho}{\pi}\bigr)\,,
\]
where ${\rm Proj}_{\Theta,\|\cdot\|_\Theta}:(\R^p,\|\cdot\|_2) \to (\Theta, \|\cdot\|_\Theta)$ is the projection operator and $h >0$ is the constant step size
for the detailed algorithm and the explicit formula for the gradient. 

The following result provides the iterative complexity of the proposed scheme.
\begin{theorem}[Iteration complexity of gradient descent]\label{thm:GD informal}
    Let assumptions \eqref{assum:R}, \eqref{assum:P1}, \eqref{assum:P2}, and \eqref{assum:RD+} be satisfied.
    Moreover, assume that $\|\theta - \tilde \theta\|_\Theta \geq \lambda_\Theta\, \|\theta - \tilde \theta\|_2$ for a conditioning number parameter $\lambda_\Theta > 0$. 
    Let $\{\theta^{(t)}\}_{t \in \mathbb{N}}$ denote the iterates generated by the projected gradient descent algorithm. There exists $\Upsilon > 0$ depending linearly on $L_V, L_V', 1/\lambda_\Theta$ and quadratically on $d$ such that  for $\kappa  \deq  \frac{L_V \lor (L_V'/2) + \Upsilon}{\ell_V \land \ell_V'} > 1$ and step size $h = \frac{1}{L_V \lor (L_V'/2) + \Upsilon}$, 
    \begin{enumerate}[label = (\roman*)]
\item $\|\theta^{(t)} - \hat \theta\|_{\Theta}^2  \leq \left(1 - \kappa^{-1}\right)^t \|\theta^{(0)} - \hat \theta\|_{\Theta}^2$. 
\item $\kl{\pi_{\theta^{(t)}}}{\pi} - \kl{\pi_{\hat \theta}}{\pi}  \leq (1 - \kappa^{-1})^t\, \frac{(L + \Upsilon)}{2}\, \|\theta^{(0)} - \hat \theta\|_{\Theta}^2$. 
\end{enumerate}
\end{theorem}
The conditioning number $\lambda_\Theta$ measures the distortion of $\|\cdot\|_\Theta$ compared to the standard Euclidean norm $\|\cdot\|_2$. In the detailed algorithm in \cref{subsec:algorithm}, we compute $\lambda_\Theta$ from the Gram matrix of the constructed dictionary for $\cT_\Theta$. 

The setting of \cref{thm:GD informal} subsumes that of \cref{thm: optimality gap informal}, and together these results provide a complete non-asymptotic, polynomial-time guarantee for solving the SSVI problem. \cref{thm: optimality gap informal} establishes an approximation bound—i.e., bias control—for the best approximation within a family whose size scales quadratically with $d$, while \cref{thm:GD informal} shows that the projected gradient descent algorithm achieves an $\epsilon$-accurate solution in $t \asymp \kappa \log(1/\epsilon)$ iterations. 

The main technical challenge in proving \cref{thm:GD informal} lies in establishing the Lipschitz smoothness of the objective function over the minimizing set.  The smoothness is far from obvious due to the non-smoothness of the entropy $H$ as a functional over the full Wasserstein space~\citep[noted in][]{Diao2023, CheNilRig25OT}. Prior solutions have focused on establishing the smoothness of $H$ over restricted subsets of the Wasserstein space~\citep[e.g.,][]{Lambert2022}. Following \cite{JiaChePoo25MFVI}, we explicitly characterize the additional smoothness parameter $\Upsilon$ of $\theta \mapsto \mathrm{KL}((T_\theta)_\# \rho \,\|\, \pi)$ over $\Theta$ in our analysis, where we use our explicit construction of the approximating family $\cT_\Theta$; see details in \cref{subsec:algorithm}.

\section{Going beyond the star graph}\label{sec:beyond_star}
In this work, we inaugurated a rigorous study of star-structured variational inference. We established the existence and uniqueness of this minimization problem, and subsequent results pertaining to the regularity of said minimizers. By reformulating the problem in a more suitable geometry (one given by the theory of optimal transport), we provided a projected gradient descent algorithm which is able to approximate the minimizer to arbitrary precision.  Here, we discuss possible extensions of star-structured variational inference when we decide to optimize over a fixed tree graph $\cG = \left(\{1,\ldots,d\}, \cE_{\cG}\right)$, and highlight the resulting technical challenges. 
\subsection{Difficulties in complete generality}
Recall that a measure $\mu \in \cP(\R^d)$ is an element of $\cC_{\cG}$ if it can be represented in the following manner:
\begin{align*}
    \mu(z_1, \dotsc, z_d) = \prod_{(i,j) \in \cE_\cG} \phi_{ij}(z_i, z_j)\,, \qquad \text{for} \qquad \phi_{ij}: \R^2 \to [0, \infty)\,.
\end{align*}
As before, we aim to extend our parameterization of such measures by representing them as pushforwards through suitable transport maps. Fixing $\rho = \cN(0, I)$ (or more generally, any product measure with a density), we can associate a Knothe--Rosenblatt (KR) map to any distribution in $\cC_{\cG}$ via the following recursive representation for each $i \in \{1, \ldots, d\}$:
\[
z_i = T_i(x_i \mid z_{\pa(i)})\,,
\]
where $\pa(i)$ denotes the parent node(s) of $i$. In analogy with \cref{sec:lift_maps}, we denote by $\cT_{\cG}$ the set of KR maps induced by $\cC_{\cG}$.

However, for a general graph $\cG$,  the set $\cT_{\cG}$ may not be convex. Indeed, the argument in \cref{lemma: convex minimizing set} relies on the injectivity of the map $\lambda T_1^1 + (1-\lambda) T_1^2$ for any two $T^1, T^2 \in \cT_{\rm star}$ and $\lambda \in (0,1)$, which cannot be easily verified for general $\cG$. For example,  let $\cG$ be a three‐node Markov chain $z_1\to z_2\to z_3$, and let $\rho = \cN(0, I_3)$.  Define $T^1(x_1,x_2,x_3)  \deq  \bigl(x_1,\;x_2,\;x_3\bigr)$ and $T^2(x_1,x_2,x_3)  \deq  \bigl(x_1,\;x_1,\;x_2\bigr)$, then $T^1, T^2 \in \cT_{\cG}$,
yet $\frac{1}{2} T^1+\frac{1}{2} T^2  \not\in \cT_{\cG}$. 
Moreover, obtaining regularity results for the minimizers is also challenging due to the complex structure of the problem. As a result, a principled approach to computation based on the polyhedral approximation is more costly, if not intractable, since the KR map $T_i$ (as a function of $x$) depends on all nodes on the path to the tree root prior to node $i$.

\subsection{A suitably nice subset of transport maps}

A possible remedy is to restrict the minimization to a ``nice'' subset of $\cT_{\cG}$, making the problem more tractable. Define $\widetilde \cT_{\cG}$ as the set of transport maps with the following structure:
\begin{equation*} \label{eq: surrogate-KR}
\widetilde \cT_{\cG}  \deq  \bigl\{ T : x\mapsto (T_i(x_i; x_{\pa(i)}))_{i\in [d]}\bigm\vert T_i(\cdot \,; x_{\pa(i)}) \text{ is strictly increasing},\, \forall i \in [d] \bigr\}\,.
\end{equation*}

Recall that for any $T \in \cT_{\cG}$, the component maps $T_i$ take the form $T_i(\cdot \mid z_{\pa(i)})$, where $z_{\pa(i)}$ depends on all the nodes lying on the path from node $i$ to the root.  Thus, we have $\widetilde \cT_{\cG} \subseteq \cT_{\cG}$, with equality achieved if and only if $\cG$ is a star graph. We introduce the surrogate structured variational inference (Surro-SSVI) problem as follows:
\begin{equation} \tag{$\msf{Surro}$-$\msf{SSVI}$}\label{approx-KR-obj}
T^\star \deq \argmin_{T \in \widetilde \cT_{\cG}} \kl{T_\# \rho}{\pi}.
\end{equation}
Observe that $\widetilde \cT_{\cG}$ is convex. Thus, by the same argument as in \cref{thm:convexity}, we can show that~\eqref{approx-KR-obj} admits a unique solution under~\eqref{ass:slc}.

However, the minimizing set is less expressive than that of the full SSVI problem~\eqref{KR-obj}, as it imposes additional graphical constraints on the distributions. For example, consider the case where the graph is a Markov chain given by $z_1 \to z_2 \to \dots \to z_d$, and let $(Z_1, \ldots, Z_d) \sim T_\# \rho$ for some $T \in \widetilde \cT_{\cG}$. Then, $Z_1$ is independent of $Z_i$ for any $i \geq 3$. In particular, if we restrict $T \in \widetilde \cT_{\cG}$ to be a linear map, then $T_\# \rho$ is again a Gaussian distribution. Let $\Sigma \in \R^{d \times d}$ denote its covariance matrix, then $\Sigma_{ij} = 0$ whenever $|j - i| \geq 2$, implying that $\Sigma$ is necessarily sparse. In contrast, the general structured variational inference framework typically assumes that the \emph{precision} matrix $\Sigma^{-1}$ is sparse, rather than $\Sigma$ itself.

The extension of our results under~\eqref{approx-KR-obj}, and more generally to arbitrary tree graphs, is beyond the scope of this paper.

\subsection*{Acknowledgments}
AAP thanks the Foundations of Data Science at Yale University for financial support. The authors thank Marcel Nutz for helpful discussions. In particular, Proposition~2.1 resulted from discussions with Marcel.

\bibliography{aos-submission} 

\begin{appendix}

\section{Proofs from Section~\ref{sec:SSVI_main}}\label{app:proofs_main}

We recall that a function $f:\R \to [0,\infty]$ is called \emph{lower semianalytic} if the sets $\{f< c\}$ are analytic for all $c \in \R$, where a subset of $\R$ is called analytic if it is the image of a Borel subset of a Polish space under a Borel mapping. Any Borel function is lower semianalytic and a lower semianalytic function is universally measurable.
We refer readers to \citet[\S7]{bertsekas1996stochastic} for the terminology used in the following proof, including analytic sets, lower semianalytic functions, and universally measurable functions.

\begin{proof}[Proof of \cref{prop:dp_principle}]\label{proof:prop:dp_principle}
    For any $\mu \in \cC_{\rm star}$, $\mu(\dd z)= \mu_1(\dd z_1)\,\mu_{-1}(\dd z_{-1} \mid z_1) $ with $\mu_{-1}(\cdot \mid z_1)\in \cP(\R)^{\otimes (d-1)}$.
    The lower bound
\begin{align*}
    &\inf_{\mu \in \cC_{\rm star}}\kl{\mu}{\pi} \\
    &\qquad \geq   \inf_{\mu_1 \in \cP(\R)} \Bigl\{\kl{\mu_1}{\pi_1} + \int \inf_{\nu \in \cP(\R)^{\otimes (d-1)}} \kl{\nu}{\pi_{-1}(\cdot \mid z_1)} \mu_1(\dd z_1)\Bigr\}  \eqcolon  V_{\rm iter}
\end{align*}
follows trivially from the KL chain rule.

To see the converse inequality, for every $\vae>0$, we want to show that
\[
\inf_{\mu \in \cC_{\rm star}}\kl{\mu}{\pi} \leq V_{\rm iter} + \vae\,.
\]
Consider \begin{equation}\label{eq: inner MF}
   h_{\rm val}(z_1)  \deq  \inf_{\nu \in \cP(\R)^{\otimes (d-1)}} \kl{\nu}{\pi_{-1}(\cdot \mid z_1)}\,.
\end{equation}
Then the minimizer to the outer problem is given by 
\[
\rmp^\star(\dd z_1) \propto e^{-h_{\rm val}(z_1)}\,\pi_1(\dd z_1)\,.
\]
Since $\nu$ varies over a Borel subset of a Polish space, $z_{1}\mapsto \pi_{-1}(\cdot \mid z_1)$ is Borel with values in a Polish space, and $\kl{\cdot}{\cdot}$ is jointly Borel and non-negative. This implies that $h_{\rm val}$ is lower semianalytic \citep[Proposition~7.47]{bertsekas1996stochastic}.  Then, the selection theorem for lower semianalytic functions~\citep[see, e.g.,][Proposition 7.50(a)]{bertsekas1996stochastic} yields a universally measurable function $z_{1}\mapsto\nu^\vae (\cdot \mid z_{1})\in \cP(\R)^{\otimes (d-1)}$ such that 
$$
   \kl{\nu^\vae(\cdot \mid z_1)}{\pi_{-1}(\cdot \mid z_1)} \leq  \inf_{\nu \in \cP(\R)^{\otimes (d-1)}} \kl{\nu}{\pi_{-1}(\cdot \mid z_1)} + \vae = h_{\rm val}(z_{1}) + \vae\,.
$$
Let $\pi^{\vae}(\dd z) \deq \rmp^\star_1(\dd z_1)\, \nu^\vae(\dd z_{-1} \mid z_1)$. Then,
\begin{align*} 
    \inf_{\mu \in \cC_{\rm star}} \kl{\mu}{\pi} 
     \leq \kl{\pi^{\vae}}{\pi}
     \leq V_{\rm iter} + \vae\,.
\end{align*}
As $\vae$ is chosen arbitrarily, we conclude the statement as desired.
\end{proof}

\begin{proof}[Proof of \cref{thm:existence}]\label{proof:thm:existence}
    Under~\eqref{ass:slc}, $h_{\rm val}(z_1)$ attains a unique solution, denoted by $\rmq^\star(\cdot \mid z_1)$; and $z_1\mapsto \rmq^\star(\cdot \mid z_1)$ is universally measurable~\cite[see, e.g.,][Proposition 7.50(b)]{bertsekas1996stochastic}. Therefore, after modifying on a null set with respect to $\pi_1$, $z_1\mapsto \rmq^\star(\cdot \mid z_1)$ is a Borel-measurable stochastic kernel~\citep[e.g.,][Lemma~7.28]{bertsekas1996stochastic} so that 
    \[
    \pi^\star(\dd z)  \propto  e^{-h_{\rm val}(z_1)}\,\pi_1(\dd z_1)\, \rmq^\star(\dd z_{-1} \mid z_1) \in \cC_{\rm star}\,.
    \]
   This concludes the existence and uniqueness of the solution to \eqref{eq:SSVI_main}.
\end{proof}

\begin{proof}[Proof of \cref{thm-self-consistency}]\label{proof:thm-self-consistency}
 The SSVI problem is equivalent to finding the optimal star-separable map $T^\star $ in~\eqref{KR-obj}. Recall that the Wasserstein gradient of $\mu\mapsto \kl{\mu}{\pi}$ at $\mu$ is $\nabla \log \frac{\mu}{\pi}$~\citep[see e.g.,][] {ambrosio2008gradient, CheNilRig25OT}. This implies the first-order optimality condition
\begin{equation} \label{eq:first order optimality}
\E_{\pi^\star} \Bigl \langle \nabla \log \frac{\pi^\star(Z)}{\pi(Z)},\,  T(Z)- Z \Bigr \rangle = 0\,,
\end{equation}
for any map $T$ such that $T_\# \pi^\star \in \cC_{\rm star}$, i.e., for $T\in \cT_{\rm star}$.

In particular, \eqref{eq:first order optimality} holds for star-separable maps of the form 
\begin{equation*}
\begin{aligned}
   \mathbb{T}^i(z) &= z + \left[0, \dotsc,0, T_i(z_i; z_1) - z_i, 0, \dotsc, 0\right]
\end{aligned}
\end{equation*}
with the convention that $ \mathbb{T}^1(z) = \left[T_1(z_1), 0,\dotsc, 0\right]$. Applying~\eqref{eq:first order optimality} to each $\mathbb T^i$,
\begin{equation} \label{first-order-KR-1}
  \E_{\pi^\star} \Bigl[\partial_1 \log \frac{\pi^\star(Z)}{\pi (Z)} \left( T_1(Z_1) - Z_1 \right ) \Bigr]= 0\,,
\end{equation}
and for $i = 2,\ldots, d$
\begin{equation*} \label{first-order-KR-i}
\E_{\pi^\star} \Bigl[\partial_i \log \frac{\pi^\star(Z)}{\pi(Z)}\left(  T_i\left(Z_i; Z_1 \right) -Z_i \right)\Bigr] = 0\,.
\end{equation*}
After decomposing $\pi^\star$, we have that for $i \neq 1$, 
\begin{equation*}\label{eq:self-consistency node i}
\E_{\pi^\star} \left[ \left(\partial_i \log\rmq_i^\star(Z_i \mid Z_1) + \partial_i V(Z)\right ) \left(T_i(Z_i ; Z_1) - Z_i\right )\right] = 0\,.
\end{equation*}
Since the function $ T_i$ is an arbitrary monotone function in the first argument and an arbitrary function in the second argument, the closed linear span of the family of functions $(x,y) \mapsto T_i(x;y) - x$ contains $C_{\rm b}(\R^2)$, the set of all bounded continuous functions.
Therefore, we obtain a simple identity
\begin{equation*} \label{self-consistency-partial-diff}
    \partial_i \log \rmq_i^\star(z_i \mid z_1) + \E_{\pi^\star}\left[\partial_i V(Z) \mid Z_1 = z_1 ,  Z_i = z_i \right] = 0\,,
\end{equation*}
for $\pi^\star$-a.e.\ $(z_1,z_i)$. Integrating the above display on both sides with respect to $z_i$ and noting that the order of differentiation and conditional expectation can be exchanged as $V$ is convex, we have
\begin{equation*} 
    \rmq_i^\star(z_i \mid z_1) \propto \exp \left( - \E_{\pi^\star}\left[ V(Z)\mid Z_1=z_1, Z_i= z_i \right] \right).
\end{equation*}

By the same argument as above, \eqref{first-order-KR-1} can be written as
\begin{equation*}
    \E_{\pi^\star} \Bigl[ \Bigl(\partial_1 \log \rmp^\star(Z_1) +  \sum_{j \geq 2} \partial_1 \log \rmq_j^\star(Z_j \mid Z_1) + \partial_1 V(Z)\Bigr )\, (T_1(Z_1) - Z_1 )\Bigr] = 0\,.
\end{equation*}
Using the tower property, the above display becomes
\begin{equation*}
    \E_{\pi^\star} \Bigl[ \Bigl(\partial_1 \log \rmp^\star(Z_1) +  \E_{\pi^\star}\Bigl[\sum_{j \geq 2} \partial_1 \log \rmq_j^\star(Z_j \mid Z_1) + \partial_1 V(Z)\Bigm\vert Z_1\Bigr]\Bigr )\, (T_1(Z_1) - Z_1)\Bigr] = 0\,. 
\end{equation*}
Hence,
\begin{equation}  \label{root-self-consistency}
\partial_1 \log \rmp^\star(z_1) + \E_{\pi^\star}\Bigl[\sum_{j \geq 2} \partial_1 \log \rmq_j^\star(Z_j \mid Z_1) + \partial_1 V(Z)\Bigm\vert Z_1 = z_1\Bigr]= 0\,,
\end{equation}
for $\rmp^\star$-a.e.\ $z_1$. To further simplify, we note that
\begin{equation} \label{ch-pa-diff-sum-to-0}
\begin{aligned}
\E_{\pi^\star}\Bigl[\partial_1 \sum_{j  \geq 2} \log \rmq_j^\star(Z_j \mid Z_1)  \Bigm\vert Z_1 = z_1\Bigr]
&=  \sum_{j  \geq 2} \E_{\pi^\star}\left[\partial_1 \log \rmq_j^\star(Z_j \mid Z_1) \Bigm\vert Z_1 = z_1\right]  \\
&= \sum_{j  \geq 2} \int \frac{\partial_1 \rmq_j^\star(z_j \mid z_1) }{\rmq_j^\star(z_j \mid z_1) }\,\rmq_j^\star(\dd z_j   \mid z_1) \\
&= \sum_{j  \geq 2} \partial_1 \int  \rmq_j^\star(\dd z_j \mid  z_1) = 0\,.
\end{aligned}    
\end{equation}
Therefore, 
\begin{equation*}
    \partial_1 \log \rmp^\star(z_1) + \E_{\pi^\star}\left[\partial_1 V(Z)\mid Z_1 = z_1\right]= 0\,,
\end{equation*}
for $\rmp^\star$-a.e.\ $z_1$.  Integrating both sides with respect to $z_1$ yields
\begin{equation} 
    \rmp^\star(z_1) \propto \exp \left( - \int_0^{z_1} \E_{\pi^\star}\left[\partial_1 V(Z) \mid Z_1 = s\right] \dd s \right),
\end{equation}
as desired.
\end{proof}

\begin{proof}[Proof of \cref{lemma: log-concavity-pi-schur}]\label{proof:lemma: log-concavity-pi-schur}
Since $L_V\vee (L_V'/2) I_{d-1} \succeq (\nabla^2 V)_{-1}\succeq \ell_V' \land \ell_V  I_{d-1}$ and its Schur complement
\begin{equation*}
    \begin{aligned}
            L_V\vee (L_V'/2) &\geq \partial_{11} V - (\nabla_{-1,1}^2 V)^\top\, ((\nabla^2 V)_{-1})^{-1}\,(\nabla_{-1,1}^2 V ) \\
    &\geq \partial_{11} V - \frac{1}{\ell_V} \sum_{j = 2}^d \left(\partial_{1j} V \right)^2 \geq  \ell_V' \land \ell_V\,,
    \end{aligned}
\end{equation*}
where the last inequality follows from~\eqref{assum:RD}.
Applying the Schur's complement theorem \citep[Theorem~1.6]{Zhang2006} yields the desired result.
\end{proof}

We recall the following result by \cite{Sheng2025Stability}.
\begin{lemma}\label{lemma: stability}
 Let \eqref{assum:R}, \eqref{assum:P1}, \eqref{assum:P2}, and \eqref{assum:RD} hold. For any $z_1,z_1'\in \R$,
    \[
    W_2\left(\rmq^\star(\cdot \mid z_1),\, \rmq^\star(\cdot \mid z_1')\right) \leq \frac{\left\|\sqrt{\sum_{i \geq 2}|\partial_{1i}V|^2} \right\|_{L^\infty}}{\ell_V}\,|z_1 - z_1'|\,.
    \]
\end{lemma}
\begin{proof}
    The assumptions in \citet[Theorem 2.1]{Sheng2025Stability} are satisfied by \cref{lemma: log-concavity-pi-schur}, and therefore
    \begin{align*}
       W_2\left(\rmq^\star(\cdot \mid z_1),\, \rmq^\star(\cdot \mid z_1')\right) & \leq \frac{1}{\ell_V}\,\|\nabla_{-1} V(z_1, \cdot) - \nabla_{-1} V(z_1', \cdot)\|_{L^2(\rmq^\star(\cdot\mid z_1'))}\\
       &\leq \frac{\left\|\sqrt{\sum_{i \geq 2}|\partial_{1i}V|^2} \right\|_{L^\infty}}{\ell_V}\,|z_1 - z_1'|\,.
\end{align*}
\end{proof}

\begin{proof}[Proof of \cref{thm: star graph regularity}]\label{proof:thm: star graph regularity}
Differentiating the self-consistency equation in \cref{thm-self-consistency} for $ \rmq_i^\star(z_i \mid z_1)$, we have 
\begin{equation*}
\begin{aligned}
    -\partial_i^2 \log \rmq_i^\star\left(z_i \mid z_1 \right)
    &= \partial_i \int \partial_i V(z_1,z_i, z_{-\{1,i\}})\,\rmq_{-i}^\star(\dd z_{-\{1,i\}} \mid z_1) \\
    &= \E_{\pi^\star}\left[\partial_{ii} V(Z) \mid Z_1 = z_1,\, Z_i = z_i\right].
\end{aligned}
\end{equation*}
In the last line, we can interchange differentiation and integration because the map $z_i\mapsto \partial_i V(z_1,z_i, z_{-\{1,i\}})$ is Lipschitz, as $\ell_V \leq \partial_{ii} V \leq L_V$. Thus, $\ell_V \leq -\partial_i^2 \log \rmq_i^\star(z_i \mid z_1) \leq L_V$. 

On the other hand, by the self-consistency equation for $\rmp^\star$ in \cref{thm-self-consistency}, we obtain
\begin{equation*}
- \partial_1 \log \rmp^\star(z_1) = \E_{\pi^\star}\left[\partial_1 V(Z) \mid  Z_1= z_1 \right], 
\end{equation*}
and thus,
\begin{align*}
- \partial_1^2 \log \rmp^\star(z_1) &= \partial_1 \int \partial_1 V(z_1, z_{-1}) \rmq^\star(\dd z_{-1} \mid z_1)\\
&=  \lim_{z_1'  \to z_1}\frac{\int \partial_1 V(z_1', z_{-1}) \rmq^\star(\dd z_{-1} \mid z_1') - \int \partial_1 V(z_1, z_{-1}) \rmq^\star(\dd z_{-1} \mid z_1)}{|z_1 - z_1'|} \\
&=  \lim_{z_1'  \to z_1}\frac{\int\left( \partial_1 V(z_1', z_{-1})  - \partial_1 V(z_1, z_{-1})\right) \rmq^\star( \dd z_{-1} \mid z_1)}{|z_1 - z_1'|}\\
&\qquad + \underbrace{\lim_{z_1'  \to z_1}\frac{\int \partial_1 V(z_1', z_{-1}) \left(\rmq^\star(z_{-1} \mid z_1') - \rmq^\star(z_{-1} \mid z_1)\right)\dd z_{-1} }{|z_1 - z_1'|}}_{\eqqcolon A} \\
&= \E_{\pi^\star}\left[\partial_{11} V(Z) \mid Z_1= z_1 \right] +  A\,. 
\end{align*}
As $z_{-1} \mapsto \partial_1 V(z_1',z_{-1})$ is $\|\sqrt{\sum_{i\geq 2} |\partial_{1i}V|^2}\|_{L^\infty}$-Lipschitz, the dual formulation for the $1$-Wasserstein distance~\citep[see e.g.,][]{Villani2003} yields that for each $z_1'\in \R$, 
\begin{equation}\label{eq: diff1}
    \begin{aligned}
    &\Bigl\lvert\int \partial_1 V(z_1',z_{-1}) \left(  \rmq^\star\left(z_{-1} \mid z_1'  \right) -  \rmq^\star\left(z_{-1} \mid z_1  \right) \right) \dd z_{-1}\Bigr\rvert \\
    &\qquad \le \Bigl\|\bigl({\sum_{i \geq 2}|\partial_{1i}V|^2}\bigr)^{1/2} \Bigr\|_{L^\infty}\,  W_1\left(\rmq^\star(\cdot \mid z_1),\, \rmq^\star(\cdot \mid z_1')\right).
\end{aligned}
\end{equation}
As $W_1 \leq W_2$, we can apply \cref{lemma: stability} and obtain
\begin{align*}
       W_1\left(\rmq^\star(\cdot \mid z_1),\, \rmq^\star(\cdot \mid z_1')\right) 
 &\leq W_2\left(\rmq^\star(\cdot \mid z_1),\, \rmq^\star(\cdot \mid z_1')\right) \\
 & \leq \frac{1}{\ell_V}\,\|\nabla_{-1} V(z_1, \cdot) - \nabla_{-1} V(z_1', \cdot)\|_{L^2(\rmq^\star(\cdot \mid z_1'))}\\
       &\leq \frac{\left\|\sqrt{\sum_{i \geq 2}|\partial_{1i}V|^2} \right\|_{L^\infty}}{\ell_V}\,|z_1 - z_1'|\,.
\end{align*}
Combining with~\eqref{eq: diff1} gives that 
\[
|A| \leq \frac{\left\|\sum_{i \geq 2}|\partial_{1i}V|^2 \right\|_{L^\infty}}{\ell_V}\,.
\]
Therefore, 
\[
- \partial_1^2 \log \rmp^\star(z_1) = \E_{\pi^\star}\left[\partial_{11} V(Z) \mid Z_1= z_1 \right] +  A \geq \ell_V' 
\]
as we assumed that
\begin{equation*}
    \partial_{11} V - \frac{\left\|\sum_{i \geq 2}|\partial_{1i}V|^2\right\|_{L^\infty}}{\ell_V} \geq \ell_V'\,. 
\end{equation*}
Similarly,
\begin{equation*}
    - \partial_1^2 \log \rmp^\star(z_1) \leq \E_{\pi^\star}\left[\partial_{11} V(Z) \mid Z_1= z_1 \right] + \frac{\left\|\sum_{i \geq 2}|\partial_{1i}V|^2\right\|_{L^\infty}}{\ell_V}  \leq L_V'\,.
\end{equation*}
\end{proof}

\begin{proof}[Proof of Theorem~\ref{thm: approx gap}]\label{proof:thm: approx gap}
The proof builds on a collection of functional inequalities, such as the log-Sobolev inequality \citep[Theorem~9.9]{Villani2003}, the Talagrand inequality \citep[Theorem~1]{Otto2000}, and the Poincar\'e inequality \citep[Corollary~4.8.2]{bakry2014analysis}. 

First, by the optimal representation of $\pi^\star$ and (two applications of) the log-Sobolev inequality, we have
\begin{equation}
\begin{aligned}
\kl{\pi^\star}{\pi} &\leq \kl{\rmp^\star}{\pi_1} + \int \kl{\rmq^\star(\cdot \mid z_1)}{\pi_{-1}(\cdot \mid z_1)}\rmp^\star(\dd z_1)\\    
&\leq \frac{1}{2 \ell_V'} \int_{\R}\left|\partial_1\log \rmp^\star(z_1) - \partial_1 \log \pi_1(z_1) \right|^2\, \rmp^\star(\dd z_1) \\
&\qquad + \frac{1}{2 \ell_V}  \sum_{i \geq 2} \int \left|-\partial_i\log \rmq_i^\star(z_i\mid z_1) - \partial_i V(z)\right|^2\,\pi^\star(\dd z)\,. 
\end{aligned}
\end{equation}
For each $i \geq 2$, the self-consistency equation~\eqref{eq: self-consistent-pi-star} implies that 
$$\E_{\pi^\star}[\partial_{i}V(Z)\mid Z_1,Z_i] = - \partial_i\log q_i^\star(Z_i \mid Z_1)\,.$$ 
Then as $\rmq_i^\star(\dd z_i \mid z_1)$  is $\ell_V$-log-concave by \cref{thm: star graph regularity} and $\rmq^\star(\dd z_{-1} \mid z_1) \in \cP(\R)^{\otimes (d-1)}$, the Poincar\'e inequality yields that
\begin{equation}\label{eq:approx-guarantee-1}
\begin{aligned}
    &\int \left|-\partial_i\log \rmq_i^\star(z_i \mid z_1) - \partial_i V(z)\right|^2\, \pi^\star(\dd z) \\
    &\qquad = \int {\rm Var}_{\pi^\star}(\partial_iV (Z)\mid Z_1= z_1,Z_i=z_i)\,\pi_{1i}^\star(\dd z_1, \dd z_i)
    \leq \frac{1}{\ell_V}{\sum_{j \geq 2,\,j\neq i}\E_{\pi^\star}[|\partial_{ij}V(Z)|^2]}\,.
\end{aligned}
\end{equation}
On the other hand, we note that 
\begin{equation*}
    \partial_1 \log \pi_1(z_1)=  \partial_1 \log \int \exp\left( - V(z)\right) \dd z_{-1} =- \int \partial_1 V(z)\, \pi_{-1}(\dd z_{-1} \mid z_1)\,. 
\end{equation*}
Then, combining with the self-consistency equation~\eqref{eq: self-consistent-pi-star}, we obtain
\begin{align*}
&\int \left|\partial_1\log \rmp^\star(z_1) - \partial_1 \log \pi_1(z_1) \right|^2\, \rmp^\star (\dd z_1)\\
&\qquad = \int\Bigl\lvert\int \partial_1V(z)\,(\rmq^\star(\dd z_{-1} \mid z_1) - \pi_{-1}(\dd z_{-1} \mid z_1))\Bigr\rvert^2\, \rmp^\star (\dd z_1)\\
&\qquad \leq \bigl\lVert\sum_{i \geq 2} |\partial_{1i}V|^2\bigr\rVert_{L^\infty}\int W_1^2\left(\rmq^\star(\dd z_{-1} \mid z_1),\, \pi_{-1}(\dd z_{-1} \mid z_1)\right) \rmp^\star (\dd z_1)\,.
\end{align*}
The last line uses the fact that $\partial_1V(z_1, \cdot)$ is $\|\sum_{i \geq 2}|\partial_{1i}V|^2\|_{L^\infty}^{1/2}$--Lipschitz.

By \eqref{assum:R} and \eqref{assum:RD}, we have
\begin{equation*}
L_V'/2 - \frac{1}{\ell_V} \sum_{j = 2}^d \left(\partial_{1j} V \right)^2 \geq \partial_{11} V - \frac{1}{\ell_V} \sum_{j = 2}^d \left(\partial_{1j} V \right)^2 \geq  \ell_V'\,. 
\end{equation*}
Hence, $\sum_{i \geq 2}(\partial_{1i}V)^2 \leq \frac{1}{2}\ell_V L_V' -  \ell_V \ell_V' < \infty$.

Thus, by the $T_1$-transport inequality and then the log-Sobolev inequality,
\begin{align*}
     &\int \left|\partial_1\log \rmp^\star(z_1) - \partial_1 \log \pi_1(z_1) \right|^2\,\rmp^\star (\dd z_1)\\
     &\qquad \leq \frac{\frac{1}{2}\ell_V L_V' - \ell_V \ell_V' }{\ell_V^2} \sum_{i \geq 2}
    \int \left|-\partial_i\log \rmq_i^\star(z_i \mid z_1) - \partial_i V(z)\right|^2\, \pi^\star(\dd z) \\
&\qquad \leq \frac{\frac{1}{2}L_V' - \ell_V'}{\ell_V^2} \sum_{i \geq 2}\sum_{j \geq 2,\,j\neq i}\E_{\pi^\star}[|\partial_{ij}V(Z)|^2]\,,
\end{align*}
where the last line follows from \eqref{eq:approx-guarantee-1}.

Combining all the above bounds, we conclude that
\begin{align*}
      \kl{\pi^\star}{\pi}  &\leq 
      \bigl(\frac{1}{2\ell_V^2}+ \frac{L_V'}{4\ell_V' \ell_V^2} - \frac{1}{2\ell_V^2}\bigr) \sum_{i \geq 2} \sum_{j \geq 2,\, j\neq i} \E_{\pi^\star} \bigl[ (\partial_{ij} V(Z))^2 \bigr] \\
      &=  \frac{L_V'}{2\ell_V' \ell_V^2}  \sum_{i \geq 2} \sum_{j > i} \E_{\pi^\star} \bigl[ (\partial_{ij} V(Z))^2 \bigr]\,. 
\end{align*}
\end{proof}

\subsection{Proofs for Examples}

We first recall the following lemma, which provides the explicit expression for the KL divergence between two Gaussian distributions.

\begin{lemma}[KL divergence between Gaussians]\label{lemma: KL W2 between normal distributions}
Consider $p_0 \sim \cN(\mu_0, \Sigma_0)$, $p_1 \sim \cN(\mu_1, \Sigma_1)$ with $\mu_0,\mu_1\in \R^d$, $\Sigma_0,\Sigma_1 \in \R^{d\times d}$. Then,
\begin{equation*}
\begin{aligned}
    \kl{p_0}{p_1} & = \dfrac{1}{2} \log \dfrac{\det \Sigma_1}{\det \Sigma_0} + \dfrac{1}{2}({\rm tr}(\Sigma_1^{-1}\Sigma_0) - d) +\dfrac{1}{2}(\mu_0 - \mu_1)^\top \Sigma_1^{-1} (\mu_0 -\mu_1)\,.
    \end{aligned}
\end{equation*}
\end{lemma}

\begin{proof}[Proof of \cref{thm: Gaussian}]\label{proof:thm-Gaussian}
  By \cref{prop:dp_principle} and by \cref{thm:convexity}, we know that~\eqref{SSVI-obj} admits a unique optimizer given by 
  \[
  \pi^\star(z)  \propto   e^{-h(z_1)}\,\pi_1(z_1)\, \rmq^\star(z_{-1}\mid z_1)\,,
  \]
  where $\rmq^\star(\cdot \mid z_1)$ is the minimizer for
  \[
 h_{\rm val}(z_1)=\inf_{\nu \in \cP_{ac}(\R)^{\otimes (d-1)}} \kl{\nu}{\pi_{-1}(\cdot \mid z_1)}\,.
  \]
  Since $\pi =\cN(m,\Sigma)$, denote by $m_{-1}=(m_2,\ldots,m_d)^\top$, $\sigma_{ij} = \Sigma_{ij}$, $\Sigma_{-1,-1} = (\sigma_{ij})_{2\leq i,j\leq d}$, and $\Sigma_{-1,1} = (\sigma_{12},\ldots, \sigma_{1d})^\top$, we know that $\pi_{-1}(\cdot\mid z_1)= \cN(m_{z_1}, \Sigma_{|1})$ where 
  \begin{align*}
      m_{z_1} &=  m_{-1} + (\sigma_{11})^{-1}\Sigma_{-1,1}(z_1-m_1)\,,\\
      \Sigma_{|1} &=  \Sigma_{-1,-1} - (\sigma_{11})^{-1}\Sigma_{-1,1}\Sigma_{-1,1}^\top = ((\Sigma^{-1})_{-1,-1})^{-1}\,.
  \end{align*}
  Therefore, \citet[Theorem~1.1]{Lacker2024} yields that $\rmq^\star(\cdot\mid z_1) = \cN(m_{z_1}, \Sigma_{|1}^\star)$ where $\Sigma_{|1}^{*} = {{\rm diag}(\Sigma_{|1}^{-1})}^{-1}$;
  and by \cref{lemma: KL W2 between normal distributions}, we derive that $h_{\rm val}(z_1) \equiv C$ for some $C>0$. Thus, we may conclude that $\pi^\star$ is also Gaussian. Finally, \eqref{eq: SSVI Gaussian} can be derived using the tower property of expectation.
\end{proof}

\begin{proof}[Proof of \cref{coro: star-graph outperforms MFVI in optimality gap}]\label{proof:coro: star-graph outperforms MFVI in optimality gap}

By \cref{thm: Gaussian} and the results by \citet{Lacker2024}, we know that both $\pi^\star$ and $\bar \pi$ are normal distributions with mean zero and covariance matrices $\Sigma^\star, \bar \Sigma$, respectively. Then \cref{lemma: KL W2 between normal distributions} implies that
\[
\kl{\bar \pi}{\pi} =  \frac{1}{2}\log \frac{\det \bar \Sigma}{\det \Sigma} + \frac{1}{2}\,\bigl({\rm tr}( \bar{\Sigma} \Sigma^{-1} )-d\bigr)\,,
\]
and 
\[
   \kl{\pi^\star}{\pi}  = \frac{1}{2}\log \frac{\det\Sigma^\star}{\det \Sigma} + \frac{1}{2}\,\bigl({\rm tr}(\Sigma^\star \Sigma^{-1} ) - d\bigr)\,.
\]
Thus,
\begin{equation}\label{eq: KL Gaussian diff}
     \kl{\pi^\star}{\pi} - \kl{\bar\pi}{\pi}  = \frac{1}{2}\log \frac{\det\Sigma^\star}{\det\bar \Sigma} + \frac{1}{2}\tr((\Sigma^\star - \bar{\Sigma}) \Sigma^{-1})\,.
\end{equation}
By \eqref{eq: SSVI Gaussian}, we can express the difference between $\Sigma^\star$ and $\bar{\Sigma}$ as follows:
\begin{equation*}
\begin{aligned}
    \Sigma^\star  - \bar{\Sigma} & = \begin{pmatrix}
        \sigma_{11} - (\sigma^{11})^{-1} & \sigma_{12} & \cdots & \sigma_{1i}& \cdots &\sigma_{1d}\\[0.25em]
        \sigma_{12} &  \dfrac{\sigma_{12}^2}{\sigma_{11}} & \cdots &  \dfrac{\sigma_{12}\sigma_{1i}}{\sigma_{11}}  & \cdots &\dfrac{\sigma_{12}\sigma_{1d}}{\sigma_{11}}\\
        \vdots & \vdots & \ddots &\vdots & \cdots & \vdots\\
        \sigma_{1i} & \dfrac{\sigma_{1i}\sigma_{12}}{\sigma_{11}} &\cdots &\dfrac{\sigma_{1i}^2}{\sigma_{11}} & \cdots &\dfrac{\sigma_{1i}\sigma_{1d}}{\sigma_{11}}\\
        \vdots & \vdots & \cdots & \vdots&\ddots & \vdots\\
    \sigma_{1d} &\dfrac{\sigma_{1d}\sigma_{12}}{\sigma_{11}} & \cdots &\dfrac{\sigma_{1d}\sigma_{1i}}{\sigma_{11}} & \cdots &\dfrac{\sigma_{1d}^2}{\sigma_{11}}
    \end{pmatrix}\\
    & \eqqcolon \Sigma + \tilde{\Sigma}\,,
\end{aligned}
\end{equation*}
where
\[
\tilde{\Sigma}  \deq  \begin{pmatrix}
        - (\sigma^{11})^{-1} & 0 & \cdots & 0 & \cdots & 0 \\
       0 &  \dfrac{\sigma_{12}^2}{\sigma_{11}} - \sigma_{22}& \cdots &  \dfrac{\sigma_{12}\sigma_{1i}}{\sigma_{11}}  - \sigma_{2i} & \cdots &\dfrac{\sigma_{12}\sigma_{1d}}{\sigma_{11}} - \sigma_{2d}\\
        \vdots & \vdots & \ddots &\vdots & \cdots & \vdots\\
       0 & \dfrac{\sigma_{1i}\sigma_{12}}{\sigma_{11}} - \sigma_{i2} &\cdots &\dfrac{\sigma_{1i}^2}{\sigma_{11}} - \sigma_{ii} & \cdots &\dfrac{\sigma_{1i}\sigma_{1d}}{\sigma_{11}} - \sigma_{id}\\
        \vdots & \vdots & \cdots & \vdots&\ddots & \vdots\\
   0 &\dfrac{\sigma_{1d}\sigma_{12}}{\sigma_{11}}  - \sigma_{d2}& \cdots &\dfrac{\sigma_{1d}\sigma_{1i}}{\sigma_{11}} -\sigma_{di}& \cdots &\dfrac{\sigma_{1d}^2}{\sigma_{11}} - \sigma_{dd}
    \end{pmatrix}.
\]
Then 
\begin{equation*}
\begin{aligned}
\tr\bigl( (\Sigma^\star - \bar\Sigma)\Sigma^{-1} \bigr) & = \tr(I + \tilde{\Sigma}\Sigma^{-1}) 
= d + \tr( \tilde{\Sigma}\Sigma^{-1})  \\
&  = d -(\sigma^{11})^{-1}\sigma^{11}  + \sum_{i\geq2}\sum_{j\geq2} \bigl(\dfrac{\sigma_{1i}\sigma_{1j}}{\sigma_{11}} - \sigma_{ij}\bigr)\,\sigma^{ji}  \\
&  = d-1   + \sum_{i\geq2} \dfrac{\sigma_{1i}}{\sigma_{11}} \sum_{j\geq2} \sigma_{1j}\sigma^{ji} - \sum_{i \geq2} \sum_{j \geq2} \sigma_{ij}\sigma^{ji}\,.
\end{aligned}
\end{equation*}
By the fact that $\Sigma\Sigma^{-1} = I$, we have 
\begin{equation*}\label{eq:Trace Sigma star minus Sigma bar equals 0}
\begin{aligned}
\tr\bigl( (\Sigma^\star - \bar\Sigma)\Sigma^{-1} \bigr) &  = d-1   + \sum_{i\geq2} \dfrac{\sigma_{1i}}{\sigma_{11}}\, (-\sigma_{11}\sigma^{1i}) - \sum_{i \geq2} (1- \sigma_{i1}\sigma^{1i})=0\,.
\end{aligned}
\end{equation*}
To compute the determinant of $\Sigma^\star$, we observe that: letting $U_i \deq \begin{pmatrix}
    1 & v_i^\top\\
    0 & I_{d-1}
\end{pmatrix}$, where $v_i\in \R^{d-1}$ with its $(i-1)$-th entry being $-\sigma_{11}^{-1}\sigma_{1i}$ and all the others being zero,
\begin{equation*}
   \bigl(\prod_{i=2}^d U_i\bigr)^\top \Sigma^\star \bigl(\prod_{i=2}^d  U_i\bigr) = \diag(\sigma_{11}, (\sigma^{22})^{-1},\dotsc, (\sigma^{dd})^{-1})\,.
\end{equation*}
Since $\det U_i = 1$, we obtain that
\begin{equation*}
    \begin{aligned}
        \det\Sigma^\star = \sigma_{11}\prod_{i=2}^d(\sigma^{ii})^{-1}\,.
    \end{aligned}
\end{equation*}
On the other hand, we know that $\det\bar\Sigma = (\prod_{i=1}^d \sigma^{ii})^{-1}$, thus
\begin{equation}\label{eq: determinant ratio}
   \dfrac{\det\Sigma^\star}{\det\bar\Sigma} = \dfrac{\sigma_{11}}{(\sigma^{11})^{-1}}\,.
\end{equation}
Substitute \eqref{eq: determinant ratio} into \eqref{eq: KL Gaussian diff}, and by the fact that $\sigma_{11} \geq(\sigma^{11})^{-1}$, we obtain
\begin{equation*}
    \kl{\pi^\star}{\pi} - \kl{\bar\pi}{\pi} = -\dfrac{1}{2}\log\Bigl(\dfrac{\sigma_{11}}{(\sigma^{11})^{-1}} \Bigr) \le 0\,. 
\end{equation*}
\end{proof}

\begin{proof}[Proof of \cref{thm:glm-location}]\label{proof:thm:glm-location}
Recall that the target measure has the form
\begin{equation*}
V_{\rm loc}(\vartheta, \beta) = \rmg(\vartheta) - w^\top \beta + \sum_{i = 1}^n \frac{\psi(X_i^\top \beta)}{c(\sigma)} + \sum_{j = 1}^d \varrho(\beta_j - \vartheta)\,.
\end{equation*}
To verify the conditions in \cref{thm: approx gap}, we compute $\partial^2_\vartheta V_{\rm loc}$, the principal minor $\nabla^2_\beta V_{\rm loc}$, and the cross terms. By the assumptions on $\rmg$ and $\varrho$, we have
\begin{equation*}
    \partial_\vartheta^2 V_{\rm loc}(\vartheta,\beta) = \sum_{j = 1}^d \varrho''(\beta_j - \vartheta) + \rmg''(\vartheta) \leq d R + L\,.
\end{equation*}
Letting $D$ be a $d\times d$ diagonal matrix with $D_{jj} = \varrho''(\beta_j - \vartheta) \geq 0$, the principal minor of $\nabla^2 V_{\rm loc}$ satisfies
\begin{equation*}
    \nabla_\beta^2 V_{\rm loc}(\vartheta,\beta) = \sum_{i = 1}^n \frac{X_i X_i^\top}{c(\sigma)}\,\psi''(X_i^\top \beta) + 
    D \succeq bA \succeq ba I\,,
\end{equation*}
as $A= \sum_{i = 1}^n \frac{X_i X_i^\top}{c(\sigma)} \succeq a I$ and $\psi''\geq b$. Finally, the cross-term is $\partial_{\vartheta \beta_j} V_{\rm loc}  = -\varrho''(\beta_j-\vartheta)$. Then the root-domination criteria is satisfied as
\begin{align*}
&\partial_\vartheta^2 V_{\rm loc} - \frac{\left\|\sum_{j = 1}^d \left(\partial_{\vartheta \beta_j} V_{\rm loc} \right)^2 \right\|_{L^\infty}}{ab} \\
&\qquad =   \rmg''(\vartheta) + \sum_{j=1}^d \varrho''(\beta_j-\vartheta) - \frac{\left\|\sum_{j = 1}^d \varrho''(\beta_j - \vartheta)^2 \right\|_{L^\infty}}{ab}
\geq \alpha - \frac{d R^2}{ab}\,.
\end{align*}
The computation above verifies assumptions \eqref{assum:R}, \eqref{assum:P1}, \eqref{assum:P2}, and \eqref{assum:RD}, with parameters given by $L_V' \gets 2(d R + L)$, $\ell_V' \gets \alpha - \frac{d R^2}{ab}$, $\ell_V \gets ab$. Finally, by the definition of $\cA(\beta)$, we have
\begin{equation*}
\sum_{i = 1}^d \sum_{j > i} \left(\partial_{ij}  V_{\rm loc}\right)^2 = \sum_{i=1}^d \sum_{j>i} \cA_{ij}^2(\beta) \;\le\; \bigl\|\sum_{i=1}^d \sum_{j>i} \cA_{ij}^2(\beta)\bigr\|_{L^\infty}\,.
\end{equation*}
Then, \cref{thm: approx gap} implies the conclusion.
\end{proof}

\begin{proof}[Proof of \cref{prop:lm-location}]
Note that $\varrho(t) = \frac{\tau^2}{2}\,t^2$, $\psi(t)= \frac{t^2}{2}$, $c(\sigma) = \sigma^2$. Then $\psi''=1$, $\varrho''=\tau^2$ (note that this is stronger than the assumption $\varrho'' \ge 0$ in Theorem~\ref{thm:glm-location}), $\ell_V = a + \tau^2$, and 
\begin{align*}
        \partial_\vartheta^2 V_{\rm loc} - \frac{\left\|\sum_{j = 1}^d \left(\partial_{\vartheta \beta_j} V_{\rm loc} \right)^2 \right\|_{L^\infty}}{a+ \tau^2}   &=   \rmg''(\vartheta) + d \tau^2 - \frac{d \tau^4}{a + \tau^2} \\
        &= \rmg''(\vartheta) + \frac{ad \tau^2}{a + \tau^2}\geq \ell_V' > 0. 
\end{align*}
In addition,  $L_V' = 2\,(d \tau^2 + L)$.  Applying \cref{thm: approx gap} yields the result.
\end{proof}

\begin{lemma}[Semi-log-concavity of a scale mixture of two Gaussians] \label{lemma:ss-lb}
Let $\nu_0, \nu_1$ be normal densities with $\nu_0 \defeq \cN(0, \tau_0^{-2})$ and $\nu_1 \defeq \cN(0, \tau_1^{-2})$, where $\tau_1^2 < \tau_0^2$. For $\eta \in [0,1]$, define the function
\begin{equation*}
    \xi(x) \;\deq\; - \log \left( \eta \,\nu_0(x) + (1-\eta)\,\nu_1(x) \right).
\end{equation*}
Then, its second derivative satisfies
\begin{equation*}
 \xi''(x) \geq \tau_1^2 \;-\; 2\,(\tau_0^2 - \tau_1^2)
    \log \Bigl(1 + \frac{\eta \tau_0}{(1-\eta)e \tau_1}\Bigr)\,.
\end{equation*}
\end{lemma}
\begin{proof}
Let $\zeta \deq \frac{\eta}{1-\eta}$, and let $w \deq \frac{\zeta \nu_0}{\zeta \nu_0 + \nu_1}$. Here we suppress the argument $x$ for simplicity. For $k \in \{0,1\}$, a calculation shows that
\begin{equation*} 
\nu_k' = - \tau_k^2 x \nu_k\,, \qquad \nu_k'' = \tau_k^4 x^2  \nu_k - \tau_k^2 \nu_k\,. 
\end{equation*}
Then, we have some useful identities:
\begin{equation*}
(\nu_k')^2 - \nu_k \nu_k'' = \tau_k^2 \nu_k^2\,, \qquad 2 \nu_0'\nu_1' - \nu_1 \nu_0'' -  \nu_0 \nu_1'' = -(\tau_0^2 - \tau_1^2)^2 x^2 \nu_0 \nu_1 + (\tau_0^2 + \tau_1^2) \nu_0 \nu_1\,.
\end{equation*}
Therefore,
\begin{align*}
\Bigl( \frac{\zeta \nu_0' + \nu_1'}{\zeta \nu_0 + \nu_1} \Bigr)^2 
- \frac{\zeta \nu_0'' + \nu_1''}{\zeta \nu_0 + \nu_1}
&= \frac{(\zeta \nu_0' + \nu_1')^2 - (\zeta \nu_0 + \nu_1 )\, (\zeta \nu_0'' + \nu_1'' )  }{\bigl(\zeta \nu_0 + \nu_1\bigr)^2}\\
&= \underbrace{\frac{\zeta^2 \tau_0^2 \nu_0^2 + \tau_1^2 \nu_1^2 + \zeta(\tau_0^2 + \tau_1^2) \nu_0 \nu_1 }{\bigl(\zeta \nu_0 + \nu_1\bigr)^2}}_{\deq A_1} - \underbrace{\frac{\zeta(\tau_0^2 - \tau_1^2)^2 x^2 \nu_0 \nu_1 }{\bigl(\zeta \nu_0 + \nu_1\bigr)^2}}_{\deq A_2}\,.
\end{align*}
We lower bound $A_1$ by completing the square: 
\begin{align*}
    A_1 = \frac{\tau_0^2 + \tau_1^2}{2} + \frac{(\tau_0^2 -\tau_1^2)\, (\zeta^2 \nu_0^2 - \nu_1^2)}{2\, \bigl(\zeta \nu_0 + \nu_1\bigr)^2} = \frac{\tau_0^2 + \tau_1^2}{2} + \frac{(\tau_0^2 -\tau_1^2)\, (\zeta \nu_0 - \nu_1)}{2 \,(\zeta \nu_0 + \nu_1)}\,. 
\end{align*}
As $|\frac{\zeta \nu_0 - \nu_1}{\zeta \nu_0 + \nu_1}| < 1$, we can further bound $A_1$: 
\begin{equation*}
     A_1 \geq \frac{\tau_0^2 + \tau_1^2}{2}  - \frac{\tau_0^2 - \tau_1^2}{2}  = \tau_1^2\,. 
\end{equation*}
For $A_2$, we derive a bound in the following steps: 
\begin{enumerate}[label = (\alph*)]
\item With $f(t) \deq\frac{ t}{\zeta \tau_0 + \tau_1 \exp \left(\frac{1}{2}(\tau_0^2 - \tau_1^2) t\right)}$, we have
\begin{equation*}
A_2 \leq  \frac{(\tau_0^2 - \tau_1^2)^2 x^2 \zeta \nu_0}{\zeta \nu_0 + \nu_1} = (\tau_0^2 - \tau_1^2)^2 \zeta \tau_0 f(x)\,. 
\end{equation*}
\item Note that $\lim_{t \to \infty} f(t) = 0$ and $f(0) = 0$. The maximum must be achieved at the stationary point $f'(t) = 0$. The first-order condition implies
\begin{equation*}
 \bigl(\frac{1}{2}\,(\tau_0^2 - \tau_1^2)\, t -1\bigr) \exp \Bigl(\frac{1}{2}\,(\tau_0^2 - \tau_1^2)\, t -1\Bigr) = \frac{\zeta \tau_0 }{e \tau_1}\,.  
\end{equation*}
Let $W(\cdot)$ denote the Lambert $W$ function, defined by $W(a) = u$ when $u \exp(u) = a$.  The optimal solution $t^*$ solves
\begin{equation*}
\frac{1}{2}\,(\tau_0^2 - \tau_1^2)\, t^* -1 = W\left(\frac{\zeta \tau_0}{e \tau_1}\right)\,.
\end{equation*}
Therefore, by the first-order condition, we have
\begin{align*}
    \zeta \tau_0+ \tau_1 \exp \Bigl(\frac{1}{2}\,(\tau_0^2 - \tau_1^2)\, t^*\Bigr) = \frac{1}{2}\, (\tau_0^2 - \tau_1^2)\, \tau_1 t^*\exp \Bigl(\frac{1}{2}\,(\tau_0^2 - \tau_1^2)\, t^*\Bigr)\,,
\end{align*}
and 
\[
\tau_1 \exp \Bigl(\frac{1}{2}\,(\tau_0^2 - \tau_1^2)\, t^*\Bigr) = \frac{\zeta\tau_0}{W\bigl(\frac{\zeta \tau_0}{e \tau_1}\bigr)}\,.
\]
Plugging in $t^*$ yields the maximum of $f$,
\[
f(t^*) = \frac{2}{(\tau_0^2 - \tau_1^2) \tau_1 \exp \left(\frac{1}{2}(\tau_0^2 - \tau_1^2) t^*\right)} = \frac{2 W\bigl(\frac{\zeta \tau_0}{e \tau_1}\bigr)}{(\tau_0^2 - \tau_1^2) \zeta \tau_0}\,. 
\]
Thus, $A_2 \leq (\tau_0^2 - \tau_1^2)^2 \zeta \tau_0 f(t^*) = 2 (\tau_0^2 - \tau_1^2)\, W\bigl(\frac{\zeta \tau_0}{e \tau_1}\bigr)$. 
\item Since $\exp(\ln(1 + z)) - 1 = z$ and $\exp(u) - 1< u \exp(u)$, $W(z) \leq \log(1 + z)$. We conclude with the desired bound. 
\end{enumerate}
\end{proof}

\begin{proof}[Proof of \cref{thm:glm-ss}]\label{proof:thm:glm-ss}
We again compute the requisite partial derivatives. The second partial derivative of $V_{\rm ss}$ with respect to $\beta_1$ is given by
\begin{equation*}
\partial_{11} V_{\rm ss}(\beta) = \biggl[\sum_{i = 1}^n \frac{X_i X_i^\top}{c(\sigma)}\, \psi''(X_i^\top \beta) \biggr]_{11} + \rmg''\Bigl(\beta_1 + \sum_{j = 2}^d \gamma_j \beta_j \Bigr)\,.
\end{equation*}
By our assumptions on $\psi$ and $\rmg$, we have the following uniform bounds:
\begin{equation*} \label{ss-partial-beta1}
b A_{11} + \alpha \leq \partial_{11} V_{\rm ss}(\beta) \leq  B A_{11} + L\,.
\end{equation*}
 Thus, \eqref{assum:R} is satisfied with $L_V' = 2\,(B A_{11} + L)$.  

To continue, define $\zeta \deq \eta/(1 - \eta)$ and $\xi(x) \deq   - \log \left( \eta \,\nu_0(x) + (1-\eta)\,\nu_1(x) \right)$.  After some simplifications, one can compute the principal minor of the Hessian of $V_{\rm ss}$ to be  
\begin{align*}
    \bigl[\nabla^2 V_{\rm ss}(\beta)\bigr]_{-1}  &=\biggl[ \sum_{i = 1}^n \frac{X_i X_i^\top}{c(\sigma)}\,\psi''(X_i^\top \beta) \biggr]_{-1} + {\rm{diag}}([\xi^{''}(\beta_k)]_{k=2}^d) + [\gamma_i\gamma_j \rmg'']_{2 \leq i,j \leq d} \\
    &\succeq bA_{-1} + {\rm{diag}}([\xi^{''}(\beta_k)]_{k=2}^d) \succeq \ell_V I_{d-1}\,, 
\end{align*}
where the last two inequalities follow by our assumptions and \Cref{lemma:ss-lb}.

We now verify \eqref{assum:RD}. For each $i \geq 2$, the mixed derivative is
\begin{equation} \label{ss-mixed-derivative}
\partial_{1j} V_{\rm ss}(\beta) = \biggl[\sum_{i = 1}^n \frac{X_i X_i^\top}{c(\sigma)}\, \psi''(X_i^\top \beta) \biggr]_{1j} + \gamma_j \rmg'' \Bigl(\beta_1 + \sum_{j = 2}^d \gamma_j \beta_j \Bigr)\,. 
\end{equation}

Since $b a_d I_d \preceq  \sum_{i = 1}^n \frac{X_i X_i^\top}{c(\sigma)}\,\psi''(X_i^\top \beta) \preceq B a_1 I_d$, the Schur complement theorem \citep[Theorem~1.6]{Zhang2006} implies
\begin{equation}\label{eq:ss-mixed-derivative-bound}
\biggl[\sum_{i = 1}^n \frac{X_i X_i^\top}{c(\sigma)}\, \psi''(X_i^\top \beta ) \biggr]_{11} - b a_d - \frac{\sum_{j = 2}^d \bigl(\bigl[\sum_{i = 1}^n \frac{X_i X_i^\top}{c(\sigma)}\, \psi''(X_i^\top \beta ) \bigr]_{1j} \bigr)^2}{B a_1 - b a_d}  \geq 0\,. 
\end{equation}

For ease of notation, we omit $\beta$ in the following expressions. Combining~\eqref{ss-mixed-derivative} and~\eqref{eq:ss-mixed-derivative-bound}, we readily compute 
\begin{equation*}
    \sum_{j = 2}^d \left(\partial_{1j} V_{\rm ss} \right)^2 \leq 2\left(B a_1 - b a_d \right) \left( B A_{11}- b a_d \right) + 2\sum_{j = 2}^d \gamma_j^2 L^2\,, 
\end{equation*}
thus
\begin{align*}
&\partial_{11} V_{\rm ss} - \frac{\bigl\|\sum_{j = 2}^d \left(\partial_{1j} V_{\rm ss} \right)^2 \bigr\|_{L^\infty}}{b a_d}  \\
&\qquad \ge b A_{11} + \alpha - \frac{2\sum_{j = 2}^d \gamma_j^2 L^2}{b a_d} - \frac{2\left(B a_1 - b a_d \right) \left( B A_{11}- b a_d \right)}{b a_d} = \ell_V' >0\,.
\end{align*}
This verifies \eqref{assum:RD}. Now, computing the cross derivatives $\partial_{ij} V_{\rm ss}$ for $i\neq j \geq 2$, we find
\begin{equation*}
    \partial_{ij} V_{\rm ss}(\beta)  = \sum_{k=1}^n \frac{\psi''(X_k^\top \beta)}{c(\sigma)} X_{ki}X_{kj} + \gamma_i\gamma_j g''\Bigl(\beta_1 + \sum_{j = 2}^d \gamma_j \beta_j \Bigr). 
\end{equation*}
Finally, using the definition of $\cA(\beta)$ and $g{''} \leq L$,  we have
\begin{equation*}
\frac{1}{2}\sum_{i = 2}^d \sum_{j > i} \left(\partial_{ij} V_{\rm ss}\right)^2 \leq \Bigl\|\sum_{i = 2}^d \sum_{j > i} \cA_{ij}^2 \Bigr\|_{L^\infty} +  L^2\sum_{i = 2}^d \sum_{j > i} \gamma_i^2 \gamma_j^2 \,. 
\end{equation*}
We apply \cref{thm: approx gap} to yield the desired result.
\end{proof}

 \begin{proof}[Proof of \cref{prop:Gaussian-Ensemble}]\label{proof:prop:Gaussian-Ensemble}
 The full posterior is given by $\pi_{\rm ss} \propto \exp(- V_{\rm ss})$ with
\begin{align*}
V_{\rm ss}(\beta) = \frac{\beta^\top A \beta}{2} - \sum_{j = 2}^d \log \left( \eta  \nu_0(\beta_j) + (1 - \eta) \nu_1(\beta_j) \right) + \frac{\tau^2}{2}\, \Bigl(\beta_1 + \sum_{j = 2}^d \gamma_j \beta_j \Bigr)^2\,. 
\end{align*}
In the remainder of the proof, we verify that assumptions \eqref{assum:R}, \eqref{assum:P1}, and \eqref{assum:RD} hold with high probability.

The second-order partial gradient of $V_{\rm ss}$ with respect to $\beta_1$ is 
$\partial_{11} V_{\rm ss}(\beta) = A_{11}+ \tau^2$, thus \eqref{assum:R} is satisfied with $L_V' = 2\,(A_{11} + \tau^2)$.  

We now verify \eqref{assum:P1}.
Let $\delta>0$ be given.
By~\citet[Theorem 6.1]{Wainwright2019}, with probability at least $1-e^{-n\delta^2/2}$,
\begin{align*}
    \lambda_{\min}(A)
    &= \sigma_{\min}(\boldsymbol X)^2
    \ge (\sqrt{n\lambda_d}\,(1-\delta) - \sqrt{d\lambda_1})_+^2\,.
\end{align*}
Take $\delta = 0.9$. Since $\lambda_1/\lambda_d = o(n/d)$ by assumption, for $n$ sufficiently large, the lower bound is at least $0.8n\lambda_d$. That is, $A\succeq 0.8\,n\lambda_d I_d$.

For the full potential, this further implies $\nabla_{-1}^2 V_{\rm ss} \succeq 0.8\, n \lambda_d I_{d-1} + D$, where $D$ is a diagonal matrix that consists of the second derivatives of $-\log(\eta \nu_0(\beta_j) + (1-\eta)\nu_1(\beta_j))$ for $j = 2, \dotsc, d$. By \Cref{lemma:ss-lb}, $D \succeq c_0 I_{d-1}$ for some constant $c_0$ independent of $d$ and $n$. This implies that, for sufficiently large $n$, it holds that $\nabla_{-1}^2 V_{\rm ss} \succeq \frac{2n \lambda_d}{3}\, I_{d-1}$. 
Therefore, \eqref{assum:P1} holds with the constant $\ell_V = \frac{2n \lambda_d}{3}$  with probability at least $1 - 2 e^{- 0.4n} = 1-o(1)$.

It remains to verify \eqref{assum:RD}. As $\partial_{1j} V_{\rm ss}(\beta) = A_{1j} + \tau^2\gamma_{j}$ for $j\geq 2$ and $(a+b)^2 \leq 2a^2+2b^2$, we know that
\begin{equation}\label{eqn:LM-RD-1.5}
    \begin{aligned}
      A_{11} + \tau^2  - \frac{\frac{3}{2}\sum_{j=2}^d (A_{1j} + \tau^2 \gamma_j)^2}{n\lambda_d}  \geq     \underbrace{ A_{11}  - \frac{3 \sum_{j = 2}^d A_{1j}^2}{n \lambda_d}}_{\eqqcolon D_1} + \underbrace{\tau^2 - \frac{3 \tau^4 \sum_{j=2}^d \gamma_j^2}{n \lambda_d}}_{\eqqcolon D_2}\,.
\end{aligned}
\end{equation}
We proceed by bounding $D_1$, $D_2$ separately.

By Young's inequality,
\begin{align*}
    \sum_{j=2}^d A_{1j}^2
    &= \sum_{j=2}^d \langle e_1, A\,e_j\rangle^2
    \le 1.1\sum_{j=2}^d \langle e_1, n\Sigma\,e_j\rangle^2 + O\Bigl(\sum_{j=2}^d \langle e_1, (A-n\Sigma)\,e_j\rangle^2\Bigr) \\
    &\le 1.1n^2 \sum_{j=2}^d \abs{\Sigma_{1j}}^2 + O\bigl( \norm{(A-n\Sigma)\,e_1}_2^2\bigr)
    \le 1.1n^2 \sum_{j=2}^d \abs{\Sigma_{1j}}^2 + O(\norm{A-n\Sigma}_2^2)\,.
\end{align*}
By Theorem~6.1 and Example~6.3 in \cite{Wainwright2019}, 
\[
\bigl\|n^{-1}\,A - \Sigma\bigr\|_2 \leq \lambda_1\,(2\vae + \vae^2)\,,
\]
where $\vae \deq \sqrt{\frac{d}{n}} + \delta$, with probability at least $1-2e^{-n\delta^2/2}$. Then, setting $\delta = \sqrt{d/n}$, we see that with probability $1-o(1)$,
\begin{align}\label{eq:Gaussian-ensemble-A1j}
    \sum_{j=2}^d A_{1j}^2
    &\le 1.1n^2 \sum_{j=2}^d \abs{\Sigma_{1j}}^2 + O(dn\lambda_1^2)\,.
\end{align}

On the other hand, $A_{11}$ is a sum of $n$ i.i.d.\ $\chi_1^2$ random variables scaled by $\Sigma_{11}$, i.e., $A_{11}$ is a $\chi_n^2$ random variable scaled by $\Sigma_{11}$.
By concentration of $\chi^2$ random variables~\citep[Example 2.5]{Wainwright2019}, it follows that $A_{11} \ge \frac{2}{3}\,n\Sigma_{11}$ with high probability.

Combining with~\eqref{eq:Gaussian-ensemble-A1j}, we conclude that
\begin{align*}
    D_1
    &\ge \frac{2}{3}\,n\Sigma_{11} - \frac{3.3n^2 \sum_{j=2}^d \abs{\Sigma_{1j}}^2 + O(dn\lambda_1^2)}{n\lambda_d}\,.
\end{align*}
Since $d\lambda_1^2/\lambda_d \ll n\lambda_d \ll n\Sigma_{11}$,
with probability $1-o(1)$, we derive that
\begin{align}\label{eqn:LM-RD-2}
    D_1
    \ge \frac{n}{2}\,\Bigl(\Sigma_{11} - \frac{8\sum_{j=2}^d \abs{\Sigma_{1j}}^2}{\lambda_d}\Bigr)
    = \ell_V'\,.
\end{align}

Moving on to $D_2$. Recall that $D_2 \deq \tau^2 - \frac{3 \tau^4 \sum_{j=2}^d \gamma_j^2}{n \lambda_d}$.
By the preceding arguments,
\begin{align*}
    \sum_{j=2}^d \gamma_j^2
    &= \frac{\sum_{j=2}^d A_{1j}^2}{A_{11}^2}
    \le \frac{1.1n^2\sum_{j=2}^d \abs{\Sigma_{1j}}^2 + O(dn\lambda_1^2)}{4n^2\Sigma_{11}^2/9}\,.
\end{align*}
Therefore, for $\lambda_1/\lambda_d = o(\sqrt{n/d})$, as $\ell_V' =   \frac{n}{2}\,\Bigl(\Sigma_{11} - \frac{8\sum_{j = 2}^d |\Sigma_{1j}|^2}{\lambda_d} \Bigr) > 0$.
\begin{align}
    D_2
    \ge \tau^2 - \frac{8\tau^4 \sum_{j=2}^d \abs{\Sigma_{1j}}^2}{n\lambda_d \Sigma_{11}^2} - O\Bigl(\frac{d\tau^4\lambda_1^2}{n^2\lambda_d^3}\Bigr)
    &= \tau^2 + \frac{2\tau^4 \ell_V'}{n^2 \Sigma_{11}^2} - \frac{\tau^4}{n\Sigma_{11}}- o\Bigl(\frac{\tau^4}{n\lambda_d}\Bigr) \nonumber\\
    &\ge \tau^2 - O\Bigl(\frac{\tau^4}{n\lambda_d}\Bigr)\,.\label{eqn:LM-RD-3}
\end{align}
Hence, $D_2 \ge 0$ with probability $1-o(1)$ provided $\tau^2 = o(n\lambda_d)$.

By combining \eqref{eqn:LM-RD-1.5}, \eqref{eqn:LM-RD-2}, and \eqref{eqn:LM-RD-3}, \eqref{assum:RD} is satisfied with the defined $\ell_V'$ with high probability. Finally, applying \cref{thm: approx gap} with $L_V' = 2\,(A_{11} + \tau^2)$, $\ell_V = \frac{n \lambda_d}{2}$, and $\ell_V'$, and noting that $\partial_{ij} V_{\mathrm{ss}}(\beta) = A_{ij} + \tau^2 \gamma_i \gamma_j$ for $i \neq j \geq 2$, we obtain
\begin{equation*}
  \kl{\pi^\star_{\rm ss}}{\pi_{\rm ss}} \leq \frac{8\,(A_{11} + \tau^2)}{n^2 \ell_V' \lambda_d^2} \sum_{i = 2}^d \sum_{j > i} \left(A_{ij}^2 + \tau^4 \gamma_i^2 \gamma_j^2\right)\,.
\end{equation*}

This concludes the proof.
\end{proof}

\section{Proofs from Section~\ref{sec:computation}}\label{app:proofs_comp}

\begin{proof}[Proof of \cref{lemma: convex minimizing set}]\label{proof: lemma: convex minimizing set}
Consider two elements $T, \tilde T \in \cT_{\rm star}$. We aim to show $\lambda T + (1 - \lambda) \tilde T \in \cT_{\rm star}$ for any $\lambda \in (0,1)$. Let $X = (X_1,\ldots, X_d) \sim \rho$ and denote $Y \deq T(X)$, $Z \deq \tilde T(X)$. It is equivalent to show that ${\rm Law}(\lambda Y + (1 -\lambda) Z) \in \cC_{\rm star}$.

Let $S_1, S_2 \sse \{2,\ldots,d\}$ be two sets such that $S_1 \cap S_2 = \emptyset$. As $\rho =\cN(0,I)$, we know that $X_{S_1} \perp X_{S_2}$, where $X_{S_l}  \deq (X_j)_{j\in S_l }$ for $l\in\{1,2\}$.
 Thus, as
  \begin{align*}
  \lambda Y_i + (1 - \lambda) Z_i =  \lambda T_i\left(X_i \mid T_1(X_1)\right) + (1 - \lambda) \tilde T_i(X_i \mid \tilde T_1(X_1))
  \end{align*}
  for every $i \in \{2,\ldots, d\}$, we conclude that
  \begin{equation}\label{eq: independence 1}
    \left\{\lambda Y_i + (1 - \lambda) Z_i \right\}_{i \in S_1} \perp \left\{\lambda Y_j + (1 - \lambda) Z_j \right\}_{j \in S_2} \mid X_1\,. 
  \end{equation}
Furthermore, since $\lambda T_1 + (1 - \lambda) \tilde T_1$ is strictly increasing, the $\sigma$-algebra generated by $\lambda Y_1 + (1 - \lambda) Z_1$ is the same as the one generated by $X_1$. Combining with \eqref{eq: independence 1} implies 
  \begin{equation*}
    \left\{\lambda Y_i + (1 - \lambda) Z_i \right\}_{i \in S_1} \perp \left\{\lambda Y_j + (1 - \lambda) Z_j \right\}_{j \in S_2}\mid \left(\lambda Y_1 + (1 - \lambda) Z_1 \right).
  \end{equation*}
Hence, applying the Hammersley--Clifford theorem \citep[see, e.g.,][Theorem~11.8]{Wainwright2019} shows the result.
\end{proof}

Before presenting the proof of \cref{thm:convexity}, we provide an alternative proof of the existence and uniqueness of the solution to \eqref{eq:SSVI_intro}, relying on the convexity of the equivalent problem \eqref{KR-obj} and the stability of conditional independence under convergence in $L^2(\rho)$ over $\cT_{\rm star}$. A recent work by \cite{beiglbock2023knothe} discusses the close relation between the topology induced by star-separable maps (and more generally, Knothe--Rosenblatt maps) and the one induced by the adapted Wasserstein distance. The latter is finer than the information topology \citep{Backhoff2020}, which preserves conditional independence. 

The \emph{adapted} ($2$-)Wasserstein distance is defined as
\begin{align}\label{eq: AW}
    AW_2^2(\mu,\nu) \defeq \inf_{\pi \in \Pi_{\rm bc}(\mu,\nu)}\int \|x-y\|^2\, \pi({\rm d}x,{\rm d}y)\,,
\end{align}
where we call $\Pi_{\rm bc}(\mu,\nu)$ the set of \emph{bicausal} couplings between $\mu$ and $\nu$. This set consists of $\pi \in \Pi(\mu,\nu)$ such that
\[
\pi_k(\rd x_k,\rd y_k \mid x_{1:k-1}, y_{1:k-1}) \in \Pi\left(\mu(\dd x_k\mid x_{1:k-1}), \nu(\dd y_k\mid y_{1:k-1})\right),\quad \pi\text{-a.s.}
\]
for all $k\leq d$, i.e.,
\[
\pi(\dd x_{1:d},\dd y_{1:d}) = \pi_1(\dd x_1,\dd y_1)\,\pi_2(\dd x_2,\dd y_2\mid x_1,y_1)\dotsm \pi_d(\dd x_d,\dd y_d\mid x_{1:d-1}, y_{1:d-1})\,.
\]
\begin{remark}\label{rk:AW}
To connect with star-separable maps, we observe that for any $T_1,T_2\in \cT_{\rm star}$ with ${T_1}_\# \rho = \mu$ and ${T_2}_\# \rho = \nu$, where $\cT_{\rm star}$ is defined in~\cref{sec:lift_maps}, then $(T_1,T_2)_\# \rho \in \Pi_{{\rm bc}}(\mu,\nu)$. Thus, $AW_2(\mu,\nu)\leq \norm{T_1-T_2}_{L^2(\rho)}$.
\end{remark}

\begin{definition}[Hellwig's information topology]\label{defn: information topology}
    Define the family of maps $\{\cI_{k}\}_{k=1}^{d-1}$ by
    \begin{align*}
        \cI_k&: \cP(\R^d) \to \cP(\R^{k}\times \cP(\R^{d-k}))\,,\\
        \cI_k(\mu)& \deq  \cK^k_\# \mu\,,\\
              \cK^k(x_1,\ldots, x_d)& \deq  \left(x_1,\ldots, x_{k}, \mu(\dd x_{k+1},\ldots, \dd x_d\mid x_1,\ldots, x_{k})\right)\,.
    \end{align*}
     Hellwig's information topology is the coarest topology such that $\cI_k$ is continuous for all $k\leq d-1$, where the space $ \cP(\R^{k}\times \cP(\R^{d-k}))$ is endowed with the usual topology of weak convergence.
\end{definition}
The following theorem says that the topology induced by the adapted Wasserstein distance is finer than the information topology.

\begin{theorem}[{\citet[Theorem~1.2, Theorem~1.3, and Lemma~1.4]{Backhoff2020}}]\label{thm:adpted-Wasserstein-equiv-info}
    Convergence in the adapted Wasserstein distance is equivalent to convergence in the information topology plus convergence of the second moment.
\end{theorem}
Finally, we recall the following fact which implies that convergence in the adapted Wasserstein distance preserves conditional independence.

\begin{theorem}[{\citet[Theorem~4.5]{barbie2014topology}}]\label{thm: stability of conditional independence}
    Let $\cX,\cY,\cZ$ be Polish spaces, $\{(\mu_n,\nu_n)\}_{n\in \N} \subset \cP(\cX\times\cY)\times \cP(\cX\times \cZ)$. Suppose that $\mu_n \to\mu$ in the information topology and $\nu_n \to\nu$ weakly. Let $\mu_n^x$ denote the transition kernel of $\mu_n$ w.r.t.\ the first marginal. Then,
\begin{equation*}
         \nu_n \otimes \mu_n^x\to \nu  \otimes  \mu^x \quad \text{weakly}\,. 
\end{equation*}
\end{theorem}

Given $\pi \in \cP(\cX\times\cY\times \cZ)$, let $\pi_{xy} \in \cP(\cX\times \cY)$ (resp.\ $ \pi_{xz}\in \cP(\cX\times\cZ)$) denote the pushforward of $\pi$ by the projection operator $(x,y,z)\mapsto (x,y)$ (resp.\ $(x,y,z)\mapsto (x,z)$). 
\begin{proposition}\label{prop: stability of conditional independence}
   Given a sequence of probability measures $\{\pi_n\}_{n\in \N} \subset \cP(\cX\times\cY\times \cZ)$ such that $\pi_n \to \pi$ weakly, suppose that $\pi_n = \pi_{n,xz} \otimes \pi_{n,xy}^x$, i.e., $\pi_n(\dd y\mid x,z) = \pi_n(\dd y\mid x)$, and $\pi_{n,xy} \to \pi_{xy}$ in the adapted Wasserstein distance. Then, $\pi = \pi_{xz}  \otimes \pi_{xy}^x$.
\end{proposition}
\begin{proof}
By \cref{thm:adpted-Wasserstein-equiv-info}, $\pi_{n,xy} \to \pi_{xy}$ in the information topology. Applying \cref{thm: stability of conditional independence} to $\mu_n=\pi_{n,xy}$ and $\nu_n = \pi_{n,xz}$, the result follows from the uniqueness of the weak limit.
\end{proof}

\begin{proof}[Proof of \cref{thm:convexity}]\label{proof:thm:convexity}
\textbf{Convexity:} 
This is similar to the proof that the KL divergence with respect to a strongly log-concave measure is strongly convex along generalized geodesics, see~\citet[Theorem 9.4.11]{ambrosio2008gradient}.

Let $T_0, T_1 \in\mathcal T \subseteq \mathcal T_{\rm star}$, define $T_t \deq (1-t) T_0 + tT_1$, and $\mu_t \deq T_{t\#} \rho$.
We now compute the second derivative of 
\begin{align*}
    t\mapsto \kl{\mu_t}{\pi} = \cV(\mu_t) + \cH(\mu_t) \defeq \int V\, \dd\mu_t + \int \log \mu_t\, \dd\mu_t\,.
\end{align*}
The following two equalities can be easily verified \citep[see, e.g.,][Chapter~5]{Villani2003}:
\begin{equation}\label{eq:potential_time_deriv}
    \frac{\dd^2\cV(\mu_t)}{\dd t^2} = \int \bigl\langle\nabla^2 V(T_t)\, (T_1 - T_0),\,T_1 - T_0\bigr\rangle\, \dd\rho, 
\end{equation}
and
\begin{equation} \label{entropy time derivative}
    \frac{\dd^2\cH(\mu_t)}{\dd t^2} = \int \|(D T_t)^{-1} D (T_1 - T_0)\|^2_{\rm F} \,\dd\rho\,.
\end{equation}
It follows from~\eqref{ass:slc} that
\begin{align*}
    \frac{\dd^2}{\dd t^2}\KL(\mu_t \mmid \pi)
    &\ge \alpha\,\|T_1 - T_0\|_{L^2(\rho)}^2\,,
\end{align*}
which is what we wanted to show.

\textbf{Existence and uniqueness:}  Let $\{\mu^n\}_{n\in \N} \subset \cC_{\rm star}$ be a minimizing sequence, i.e., $\kl{\mu^n}{\pi} \to a \deq \inf_{\mu \in \cC_{\rm star}}\kl{\mu}{\pi}$ as $n\to\infty$. For every $m,n \in \N$, let $T^m\in \cT_{\rm star}$ (resp.\ $T^n \in \cT_{\rm star}$) denote the star-separable map from $\rho$ to $\mu^m$ (resp.\ $\mu^n$). By \cref{lemma: convex minimizing set}, $\mu^{m,n} \deq (\frac{1}{2}T^m + \frac{1}{2}T^n)_\# \rho \in \cC_{\rm star}$, thus $\kl{\mu^{m,n}}{\pi}\geq a$ and $\liminf_{m,n\to \infty} \kl{\mu^{m,n}}{\pi}\geq a$. By $\alpha$-convexity of $T\mapsto \kl{T_\# \rho}{\pi}$, we have
\[
\kl{\mu^{m,n}}{\pi} \leq \frac{1}{2}\,\kl{\mu^m}{\pi} + \frac{1}{2}\,\kl{\mu^n}{\pi} - \frac{\alpha}{8}\, \|T^m - T^n\|^2_{L^2(\rho)}\,.
\]
As $\alpha>0$, taking $\liminf_{m,n\to \infty}$ on both sides implies that 
\[
\limsup_{m,n\to \infty}\|T^m - T^n\|^2_{L^2(\rho)} = 0\,.
\]
Therefore, by completeness of $L^2(\rho)$, there exists $T^\star \in L^2(\rho)$ such that $T^n \to T^\star$ in $L^2(\rho)$. Denote by $\pi^\star  \deq  T^\star_\# \rho$, then $\mu^n\to \pi^\star$ converges in Wasserstein distance and thus converges weakly. Combining with the weak lower semi-continuity of $\mu \mapsto \kl{\mu}{\pi}$, we derive that $\kl{\pi^\star}{\pi} \leq a$. 

We claim that $\pi^\star \in \cC_{\rm star}$, whence, we conclude that $\pi^\star$ is a minimizer to~\eqref{SSVI-obj}. Indeed, by the Hammersley--Clifford theorem \cite[Theorem~11.8]{Wainwright2019}, it suffices to verify that for any $1<j\leq d$, $\pi_j^\star(\dd z_j\mid z_1, \ldots, z_{j-1}) = \pi_j^\star(\dd z_j\mid z_1)$. Since $T^n \to T^\star$ in $L^2(\rho)$, we know that $AW_2(\mu^n_{\{1,j\}}, \pi^\star_{\{1,j\}})\to 0$ as the measurability is preserved under $L^2(\rho)$ convergence. Therefore, applying \cref{prop: stability of conditional independence} with $\pi_n = \mu^n_{[j]}$ and $\pi = \pi^\star_{[j]}$ yields the claim.

Finally, the uniqueness of $T^\star$ follows from the $\alpha$-convexity and the convexity of $\cT_{\rm star}$, which also shows the uniqueness of $\pi^\star$.
\end{proof}

In fact, the key is the convergence of $T_n$ to $T^\star$ in $L^2(\rho)$, as this convergence preserves conditional independence. This also shows that the $L^2(\rho)$-closure of $\cT_{\rm star}$ is also convex. Therefore, we can conclude that a minimizer exists over any set of structured distributions if the set of lifted maps is convex, which gives potential guidance for other structured VI problems.

\section{Computational guarantees}\label{section: approximation}

\subsection{Regularity of star-separable maps}
In this subsection, we provide the proofs of \cref{lemma: log density mix derivative bound} and \cref{thm:star_caff}.

\cref{lemma: log density mix derivative bound} establishes a bound on the mixed derivative of the SSVI minimizer. As this result is used repeatedly in the subsequent proofs, we restate it below for convenience.

\begin{lemma}
Let \eqref{assum:R}, \eqref{assum:P1}, \eqref{assum:P2}, and \eqref{assum:RD+} be satisfied. Denote by $$L \deq  \max_{i \geq 2}\sup_{z_1,z_i\in\R}|\partial_1\partial_i \log \rmq_i^\star(z_i\mid z_1)|\,.$$  Then 
\begin{equation*}
        L \leq \overline{L} \deq \frac{\ell_V\max_{i \geq 2}\|\partial_{1i}V\|_{L^\infty}}{\ell_V- \max_{i \geq 2} \bigl\|\sum_{j \geq 2,\, j\neq i} |\partial_{ij}V|\bigr\|_{L^\infty}} < +\infty\,.
\end{equation*}
\end{lemma}

\begin{proof}[Proof of \cref{lemma: log density mix derivative bound}]\label{proof:lemma: log density mix derivative bound}
The proof uses the strategy of self-bounding.  Denote $h_i(z_1,z_i)  \deq  -\partial_1\log \rmq_i^\star(z_i\mid z_1)$. \cref{thm-self-consistency} yields that
\begin{equation*}
\begin{aligned}
    \partial_ih_i(z_1,z_i) &= \E_{\pi^\star}\left[\partial_{1i} V(Z) \mid Z_1 = z_1, Z_i = z_i \right] \\
    &\qquad + \int \partial_i V(z_1,z_i, z_{-\{1,i\}}) \sum_{j\geq 2,\,j\neq i} \partial_1 \log \rmq^\star_j(z_j\mid z_1) \rmq_{-i}^\star(\dd z_{-\{1,i\}}\mid z_1)\\
    & = \E_{\pi^\star}\left[\partial_{1i} V(Z) \mid Z_1 = z_1, Z_i = z_i \right] \\
    &\qquad - \E_{\pi^\star}\Bigl[\partial_i V(Z) \sum_{j \geq 2,\, j \neq i} h_j(Z_1,Z_j) \Bigm\vert Z_1 = z_1, Z_i = z_i \Bigr]\,.
\end{aligned}
\end{equation*}
Therefore, as $\E_{\pi^\star}[h_j(Z_1,Z_j) \mid Z_1= z_1, Z_i = z_i] = \E_{\pi^\star}[h_j(Z_1,Z_j) \mid Z_1 = z_1] = 0$ for all $j \geq 2$, $i\neq j$, we have
\begin{equation}\label{eq: self-bounding 1}
\begin{aligned}
&\left|\partial_ih_i(z_1,z_i) - \E_{\pi^\star}\left[\partial_{1i} V(Z) \mid Z_1 = z_1, Z_i = z_i \right]\right| \\
&\qquad =  \Bigl|\E_{\pi^\star}\Bigl[\partial_i V(Z) \sum_{j \geq 2,\, j \neq i} h_j(Z_1,Z_j) \Bigm\vert Z_1 = z_1, Z_i = z_i \Bigr] \Bigr|\\
&\qquad =  \Bigl|\sum_{j \geq 2,\, j \neq i}\E_{\pi^\star
} \bigl[\E_{\pi^*}\left[\partial_i V(Z) \mid Z_1, Z_i, Z_j\right] h_j(Z_1,Z_j) \bigm\vert Z_1 = z_1, Z_i = z_i \bigr] \Bigr|\\
&\qquad =  \Bigl|\sum_{j \geq 2,\, j \neq i}\cov_{\pi^\star} \bigl(\E_{\pi^*}[\partial_i V(Z) \mid Z_1, Z_i, Z_j], h_j(Z_1,Z_j) \bigm\vert Z_1 = z_1, Z_i = z_i \bigr) \Bigr|\,.
\end{aligned}
\end{equation}
Using the asymmetric Brascamp–Lieb inequality \citep[Theorem 1.1]{Carlen2013}  with $p=\infty$, $q= 1$, and noting that $-\partial_{jj} \log \rmq_j^\star(z_j\mid z_1) = \E_{\pi^\star}[\partial_{jj}V(Z) \mid Z_1 = z_1, Z_j = z_j] \geq \ell_V$ by \cref{thm-self-consistency}, we derive that for every $j \neq i \geq 2$,
    \begin{equation}\label{eq: self-bounding 2}
    \begin{aligned}
    &\left|\cov_{\pi^\star} \bigl(\E_{\pi^*}[\partial_i V(Z) \mid Z_1, Z_i, Z_j], h_j(Z_1,Z_j) \bigm\vert Z_1 = z_1, Z_i = z_i \bigr) \right| \\
&\qquad \leq \Bigl(\int \left|\E_{\pi^*}\left[\partial_{ij} V(Z) \mid Z_1 = z_1, Z_i = z_i, Z_j= z_j\right]\right|\rmq_j^\star(\dd z_j\mid z_1)\Bigr) \\
&\qquad\qquad\times \sup_{z_j \in \R}{ \Bigl|\frac{\partial_j h_j(z_1,z_j)}{\E_{\pi^\star}\left[\partial_{jj} V(Z) \mid Z_1 = z_1, Z_j = z_j\right]} \Bigr|}\\
&\qquad\leq \frac{L}{\ell_V}\, \E_{\pi^\star} \left[ \left|\partial_{ij} V(Z) \right| \mid Z_1 = z_1, Z_i = z_i\right].
\end{aligned}
\end{equation}
Combining~\eqref{eq: self-bounding 1} and \eqref{eq: self-bounding 2}, we obtain
\begin{align*}
 &\left|\partial_ih_i(z_1,z_i) - \E_{\pi^\star}\left[\partial_{1i} V(Z) \mid Z_1 = z_1, Z_i = z_i \right]\right| \\
 &\qquad \leq   \frac{L}{\ell_V} \sum_{j \geq 2,\, j \neq i}  \E_{\pi^\star} \left[ \left|\partial_{ij} V(Z) \right| \mid Z_1 = z_1, Z_i =z_i\right]. 
\end{align*}
By the triangle inequality, we derive that 
\begin{align*}
 &\left|\partial_ih_i(z_1,z_i)\right| \\
 &\qquad \leq   \frac{L}{\ell_V} \sum_{j \geq 2,\, j \neq i}  \E_{\pi^\star} \left[ \left|\partial_{ij} V(Z) \right| \mid Z_1 = z_1, Z_i =z_i\right] + \E_{\pi^\star}\left[\left|\partial_{1i} V(Z)\right| \mid Z_1 = z_1, Z_i = z_i \right].
\end{align*}
Taking the supremum over $z_1,z_i$ and then maximum over $i \geq 2$ on both sides, we obtain
\begin{equation*}
\begin{aligned}
& \phantom{=} L- \max_{i \geq 2}\|\partial_{1i}V\|_{L^\infty} 
\leq \frac{L}{\ell_V}\max_{i \geq 2}{\Bigl\|\sum_{j \geq 2,\, j\neq i} |\partial_{ij}V|\Bigr\|_{L^\infty}}\,. 
\end{aligned}
\end{equation*}
Under \eqref{assum:RD+}, the inequality above implies a bound on $L$:
\begin{align*}
        L
        &\leq \frac{\ell_V\max_{i \geq 2}\|\partial_{1i}V\|_{L^\infty}}{\ell_V- \max_{i \geq 2} \bigl\|\sum_{j \geq 2,\, j\neq i} |\partial_{ij}V|\bigr\|_{L^\infty}} =\overline{L}\,.
\end{align*}
\end{proof}

\begin{lemma}\label{lemma:perturbation Ti}
Let \eqref{assum:R}, \eqref{assum:P1}, \eqref{assum:P2}, and \eqref{assum:RD+}  be satisfied. Let $T^\star$ be the optimal star-separable map from $\rho$ to $\pi^\star$.
Then, for any $z_1,\tilde z_1 \in \R$, we have 
\begin{equation}\label{eq: L2 bound result}
\begin{aligned}
      &\phantom{=}\,\, \sum_{i=2}^d \int | T_i^{\star}(x_i \mid \tilde z_1) - T_i^{\star}(x_i \mid z_1)|^2\, \rho_i\left(\dd x_i\right)\leq   \frac{\bigl\|\sum_{i \geq 2}|\partial_{1i}V|^2 \bigr\|_{L^\infty}}{\ell_V^2} \, |\tilde z_1 - z_1|^2\,.
\end{aligned}
\end{equation}
\end{lemma}
\begin{proof}
    Since $x_i\mapsto T_i^{\star}(x_i\mid z_1)$ is increasing for any $z_1\in \R$, by \cref{lemma: stability} we know that
    \begin{align*}
       \sum_{i=2}^d \int | T_i^{\star}(x_i \mid \tilde z_1) - T_i^{\star}(x_i \mid z_1)|^2\, \rho_i\left(\dd x_i\right) &=\sum_{i=2}^d W_2^2(\rmq_i^\star(\cdot \mid  \tilde z_1),\, \rmq_i^\star(\cdot \mid  z_1)) \\
       &= W_2^2(\rmq^\star(\cdot \mid  \tilde z_1),\, \rmq^\star(\cdot \mid  z_1)) \\
        &\leq \frac{\bigl\|\sum_{i=2}^d |\partial_{1i}V|^2 \bigr\|_{L^\infty}}{\ell_V^2}\,|\tilde z_1 - z_1|^2\,,
    \end{align*}
    where the second line follows from the fact that $\rmq^\star(\cdot \mid  \tilde z_1),\, \rmq^\star(\cdot \mid  z_1) \in \cP(\R)^{\otimes (d-1)}$.
\end{proof}

\begin{lemma}\label{lemma: lipschitz bound}
Assume \eqref{assum:R}, \eqref{assum:P1}, \eqref{assum:P2}, and \eqref{assum:RD+} hold. For all $z_1, \tilde z_1 \in \R$, $i \geq 2$, $x_i \in \R$, 
\begin{equation} \label{eqn: infty bound}
        \left|T_i^{\star}(x_i\mid z_1) - T_i^{\star}(x_i\mid \tilde z_1)\right| \leq \frac{\overline{L}}{\ell_V}\, |\tilde z_1 - z_1|\,,
\end{equation}
 i.e., $z_1 \mapsto T_i^{\star}(x_i\mid z_1)$ is $\frac{\overline{L}}{\ell_V}$-Lipschitz in $z_1$ uniformly for all $x_i\in \R$. 
\end{lemma}
\begin{proof}
We drop the subscripts in $x_i,z_1,\tilde z_1$ for brevity. We first make the following key observation:
\begin{equation} \label{eqn:W_infty bound}
        \left|T_i^{\star}(x\mid z) - T_i^{\star}(x\mid \tilde z)\right| \leq W_\infty( \rmq_i^\star(\cdot\mid z),\, \rmq_i^\star(\cdot\mid \tilde z))\qquad \text{for all}\,\, x\in \R\,.
\end{equation}
Here $W_\infty$ denotes the $\infty$-Wasserstein distance between $\mu,\nu \in \cP(\R^2)$ defined by 
$
W_\infty(\mu,\nu) \deq  \inf_{\pi \in \Pi(\mu,\nu)}\|d\|_{L^\infty(\pi)},
$
where $d(x,y) = \|x-y\|$ and $L^\infty(\pi)$ denotes the $L^\infty$ norm under $\pi$.

    To see this, we claim that the above inequality holds for $\rho_1$-almost every $x \in \R$, then by the continuity of $T_i^{\star}(\cdot\mid z)$ and $T_i^{\star}(\cdot\mid \tilde{z})$, the inequality holds for all $x\in  \R$.

We now validate the claim using contradiction. Let $\delta>0$. Suppose the inequality fails over a set $E\subset \R$ with $\rho_1(E) > 0$, such that 
\begin{equation*}
    \left|T_i^{\star}(x\mid z) - T_i^{\star}(x \mid \tilde{z})\right| \geq W_\infty(\rmq_i^\star(\cdot\mid z) ,\, \rmq_i^\star(\cdot \mid \tilde{z})) + \delta \,,\qquad x\in E\,.
\end{equation*}
Then, because $(T_i^\star(\cdot\mid z), T_i^\star(\cdot \mid \tilde z))_\# \rho_1$ is a $W_p$-optimal coupling of $\rmq_i^\star(\cdot \mid z)$ and $\rmq_i^\star(\cdot \mid \tilde z)$,
\begin{equation}
\begin{aligned}
W_p(\rmq_i^\star(\cdot\mid z), \rmq_i^\star(\cdot\mid \tilde{z})) &\geq\Bigl(\int_E |T_i^{\star}(x\mid z) - T_i^{\star}(x\mid \tilde{z})|^p\, \rho_1(\dd x) \Bigr)^{1/p}  \\
&\geq\bigl(W_{\infty}(\rmq_i^\star(\cdot\mid z),\, \rmq_i^\star(\cdot\mid \tilde{z})) + \delta\bigr)\,\rho_1(E)^{1/p}\,,   
\end{aligned}
\end{equation}
where $W_p$ is the $p$-Wasserstein distance. Letting $p \to \infty$, as $W_p(\rmq_i^\star(\cdot\mid z),\, \rmq_i^\star(\cdot\mid \tilde{z})) \to W_{\infty}(\rmq_i^\star(\cdot\mid z),\, \rmq_i^\star(\cdot\mid \tilde{z}))$ when $p\to \infty$ \citep[see, e.g.,][Proposition 3]{givens1984class}, we arrive at the contradiction.

Therefore, combining \eqref{eqn:W_infty bound}, Theorem~2.5 in \cite{KhuMaaPed25LInf}, and \cref{lemma: log density mix derivative bound}, we have 
\begin{align*}
    \left|T_i^{\star}(x\mid z) - T_i^{\star}(x\mid \tilde z)\right|  
    &\leq W_\infty(\rmq_i^\star(\cdot\mid z),\, \rmq_i^\star(\cdot\mid \tilde z))\\
    &\leq \frac{1}{\ell_V}\, I_\infty(\rmq_i^\star(\cdot\mid z) \mmid \rmq_i^\star(\cdot\mid \tilde z))\\
    &= \frac{1}{\ell_V} \,\| (\log \rmq_i^\star)'(\cdot \mid z) - (\log \rmq_i^\star)'(\cdot \mid \tilde z) \|_{L^\infty(\rmq_i^\star(\cdot \mid z))}\\
    & \leq \frac{\overline{L}}{\ell_V}\, |z - \tilde z|\,,
\end{align*}
where for $\nu \ll \mu$, the $L^\infty$ relative Fisher information is given by
\begin{align*}
I_\infty(\nu \mmid \mu)
&\deq \bigl\|\nabla \log\bigl(\frac{\rd\nu}{\rd\mu}\bigr)\bigr\|_{L^\infty(\mu)}\,.
\end{align*}
\end{proof}

\begin{remark}
    \cref{lemma: lipschitz bound} implies that $z_1 \mapsto T_i^{\star}(x_i\mid z_1)$ is differentiable in $z_1$ almost everywhere  for all $x_i \in \R$. 
\end{remark}

We will make use of the following standard facts throughout the next part of the proof; see \cite{JiaChePoo25MFVI} for proofs.

\begin{lemma}\label{lem:T_near_mean}
    Let $T$ denote the optimal transport map from $\rho = \cN(0,1)$ to $\mu$, and let $m$ denote the mean of $\mu$. If $T' \leq \beta$, then $|T(0) - m| \leq \sqrt{2/\pi}\,\beta$.
\end{lemma}
\begin{lemma}\label{lem:mean_near_mode}
    Let $m$ and $\tilde m$ denote the mean and the mode of $\mu$, respectively, where $\mu$ is $\ell_V$-log-concave and univariate.
    Then, $|m-\tilde m| \le 1/\sqrt{\ell_V}$.
\end{lemma}

\begin{lemma}\label{lemma: mix derivative growth rate}
Assume \eqref{assum:R}, \eqref{assum:P1}, \eqref{assum:P2}, and \eqref{assum:RD+} hold. For all $i \geq 2$,  $x_i,z_1\in \R$, 
\begin{equation}
  |\partial_{z_1} \partial_{x_i} T_i^{\star}(x_i\mid z_1)|  \lesssim \frac{L_V \overline{L}}{\ell_V^2}\, (1 + |x_i|) \,.
\end{equation}
\end{lemma}
\begin{proof}[Proof of \cref{lemma: mix derivative growth rate}]
    For ease of notation, we fix $i$ and drop the subscripts in $z_1,x_i$. Denote by $T_z(x) \deq  T_i^{\star}(x\mid z)$ and $\rmq_z$ the density of $\rmq_i^\star(\cdot\mid z)$. Taking the log of the change of variables formula gives
\begin{align*}
    \log \partial_x T_z(x) = - \log \sqrt{2\pi} - \frac{x^2}{2} - \log \rmq_z(T_z(x))
\end{align*}
for all $x\in \R$. Then, taking the derivative w.r.t.\ $z$ yields that
\begin{align*}
    \frac{\partial_z \partial_x T_z(x)}{\partial_x T_z(x)} = - \frac{\partial_x \rmq_z(T_z(x))\,\partial_z T_z(x)}{\rmq_z(T_z(x))}  -\frac{\partial_z \rmq_z(T_z(x))}{\rmq_z(T_z(x))}\qquad \text{for all}\,\, x\in \R\,.
\end{align*}
Rearranging to isolate $\partial_z \partial_x T_z(x)$ yields
\begin{equation} \label{mixed derivatives T}
    \partial_z \partial_x T_z(x) = \partial_x T_z(x)\,\bigl(-  \partial_x\log \rmq_z(T_z(x))\,\partial_z T_z(x)  -\partial_z \log \rmq_z(T_z(x))\bigr)
\end{equation}
for all $x\in \R$. Here, $\partial_x \log \rmq_z(T_z(x))$ should be understood as the derivative of $\log \rmq_z(\cdot)$ evaluated at the point $T_z(x)$.
By $\ell_V$-log-concavity of $\rmq_z$ and Caffarelli's contraction theorem, $T_z$ is $\sqrt{1/\ell_V}$-Lipschitz, i.e., $0\le \partial_x T_z\leq \frac{1}{\sqrt{\ell_V}}$. On the other hand, by \cref{lemma: lipschitz bound}, $|\partial_z T_z(x)| \leq \frac{\overline{L}}{\ell_V}$.

Therefore, 
\begin{align*}
    |\partial_z \partial_x T_z(x)| &= \bigl| \partial_x T_z(x)\,\bigl(-  \partial_x\log \rmq_z(T_z(x))\,\partial_z T_z(x)  -\partial_z \log \rmq_z(T_z(x))\bigr)\bigr|\\
    & \leq \frac{1}{\sqrt{\ell_V}}\,\Bigl(|\partial_x\log \rmq_z(T_z(x))|\,\frac{\overline{L}}{\ell_V} + |\partial_z \log \rmq_z(T_z(x))|\Bigr)\,.
\end{align*}

Denote by $\tilde m_z$ and $m_z$ the mode and mean of $\rmq_z$, respectively. Since $\rmq_z$ is $L_V$-log-smooth by \cref{thm: star graph regularity}, and $T_z$ is smooth (in $x$), we can apply  \cref{lem:T_near_mean} and \cref{lem:mean_near_mode} and obtain
\begin{align*}
    |\partial_x\log \rmq_z(T_z(x))| &\leq \underbrace{|\partial_x\log \rmq_z(\tilde m_z)|}_{=0} + L_V\,|T_z(x) - \tilde m_z|\\
    & \leq L_V \left(|T_z(x) - T_z(0)| + |T_z(0) - m_z| + |m_z - \tilde m_z|\right)\\
    & \leq L_V\,\Bigl(\frac{1}{\sqrt{\ell_V}}\,|x| + \sqrt{\frac{2}{\pi}}\,\frac{1}{\sqrt{\ell_V}} + \frac{1}{\sqrt{\ell_V}}\Bigr)\\
    & \lesssim \frac{L_V}{\sqrt{\ell_V}}\,(1+ |x|)\,.
\end{align*}

We proceed with bounding $\partial_z \log \rmq_z(T_z(x))$. Note that by the fundamental theorem of calculus and as $|\partial_z\partial_x \log \rmq_z(x)| \leq \overline{L}$, we have (again, using \cref{lem:T_near_mean} and \cref{lem:mean_near_mode})
\begin{align*}
    |\partial_z\log \rmq_z(T_z(x))| &\leq |\partial_z\log \rmq_z(\tilde m_z)| + \overline{L}\,(|T_z(x) - T_z(0)| + |T_z(0) - m_z| + |m_z-\tilde m_z|)\\
    & \lesssim |\partial_z\log \rmq_z(\tilde m_z)| + \frac{\overline{L}}{\sqrt{\ell_V}}\,(1+|x|)\,.
\end{align*}
By \cref{thm-self-consistency}, the derivative $\partial_z \log \rmq_z(z_i)  = \partial_z \log \rmq_i^\star(z_i \mid z)$ satisfies
\begin{equation}\label{eq:zlogp}
    \begin{aligned}
    -\partial_z \log \rmq_z(z_i)  & = \partial_z \Bigl(\int V(z,z_i, z_{-\{1,i\}})\,\rmq_{-i}^\star(\dd z_{-\{1,i\}}\mid z)\Bigr) \\
    &\qquad +  \partial_z\log  \int \exp \Bigl(-\int V(z,z_i', z_{-\{1,i\}})\,\rmq_{-i}^\star(\dd z_{-\{1,i\}}\mid z) \Bigr)\, \dd z_i'\,.
\end{aligned}
\end{equation}
Note that 
\begin{align*}
    &\partial_z\log  \int \exp \Bigl(-\int V(z,z_i', z_{-\{1,i\}})\,\rmq_{-i}^\star(\dd z_{-\{1,i\}}\mid z) \Bigr)\, \dd z_i' \\
     &\qquad = - \int \partial_z \Bigl(\int V(z,z_i', z_{-\{1,i\}})\,\rmq_{-i}^\star(\dd z_{-\{1,i\}}\mid z) \Bigr)\, \rmq^\star_i(\dd z_i'\mid z)\,.
\end{align*}
Denoting
\[
F(z,z_i)  \deq  \partial_z \Bigl(\int V(z,z_i, z_{-\{1,i\}})\,\rmq_{-i}^\star(\dd z_{-\{1,i\}}\mid z) \Bigr)\,,\]
we can rewrite~\eqref{eq:zlogp} as
\begin{align*}
    -\partial_z \log \rmq_z(z_i) & = F(z,z_i)  - \int F(z,z_i')\, \rmq^\star_i(\dd z_i'\mid z)\\
    & =  \int \bigl(F(z,z_i)  - F(z,z_i')\bigr)\, \rmq^\star_i(\dd z_i'\mid z)\,.
\end{align*}
Moreover, by \cref{lemma: log density mix derivative bound}, we have 
\begin{align*}
\left| \partial_{z_i} F(z, z_i) \right|
&= \left| \partial_{z_i} \partial_z \log \rmq_i^\star \left(z_i \mid z \right) \right|
\leq \overline{L}\,. 
\end{align*}
Therefore, the mean value theorem and \citet[Lemma~2]{Dalalyan2022} yield that
\begin{align*}
    |\partial_z \log \rmq_z(\tilde m_z)|
    & \leq  \overline{L}  \int |\tilde m_z - z_i'|\,\rmq^\star_i(\dd z_i'\mid z)
    \leq \frac{\overline{L}}{\sqrt{\ell_V}}\,.
\end{align*}
Combining the above displays, we conclude: 
\begin{equation*}
    |\partial_z \partial_x T_z(x)|  \lesssim \frac{1}{\sqrt{\ell_V}}\,\Bigl(\frac{L_V \overline{L}}{\ell_V\sqrt{\ell_V}}\, (1 + |x|) + \frac{\overline{L}}{\sqrt{\ell_V}}\,(1+|x|) + \frac{\overline{L}}{\sqrt{\ell_V}} \Bigr)
    \lesssim \frac{L_V \overline{L}}{\ell_V^2} \, (1 + |x|)\,.
\end{equation*}
\end{proof}

We conclude this subsection with the proof of \cref{thm:star_caff}.

\begin{proof}[Proof of \cref{thm:star_caff}]\label{proof:thm:star_caff}
 The first part of \cref{thm:star_caff} follows along the lines of the proof of \citet[Theorem~5.4]{JiaChePoo25MFVI}. Then, the bound on $\partial_{z_1} T_i^\star(x_i \mid \cdot)$ is shown in \cref{lemma: lipschitz bound}, and the bound on $\partial_{z_1} \partial_{x_i} T_i^\star(x_i \mid \cdot)$ is shown in \cref{lemma: mix derivative growth rate}.
\end{proof}

\subsection{Dictionary of maps}\label{subsec:dictionary}
To continue the development from \cref{sec:gradient_descent}, we now design a parameterized family of star-separable transport maps $\cT_{\Theta}$ to approximate the SSVI minimizer. The construction is based on a finite dictionary of star-separable maps $\cM$, and we consider a polyhedral subset of $\cT_{\rm star}$ defined by the conic hull ${\rm cone}(\cM)$. Following \citet{JiaChePoo25MFVI}, the set ${\rm cone}(\cM)$ is constructed as
\begin{equation*}
    \mathrm{cone}(\cM)  \deq  \Bigl\{\sum_{T\in \cM} \lambda_T T \,\,\Big| \,\, \lambda \in \mathbb{R}_+^{|\cM|}\Bigr\}\,, 
\end{equation*}
where $\R_+^{|\cM|}$ is the non-negative orthant in $\R^{|\cM|}$. Clearly, when $\cM \subseteq \cT_{\rm star}$, $\mathrm{cone}(\cM)$ is a convex cone in $\cT_{\rm star}$.

We now construct the dictionary $\cM$ from one-dimensional piecewise linear maps. Define $\psi:\R \to \R$ by
\begin{equation*}
\psi(x) \deq
\begin{cases}
0\,, & x \leq 0\,, \\
x\,, & x \in [0,1]\,, \\
1\,, & x \geq 1\,.
\end{cases}
\end{equation*}

The function $\psi$ is the basic building block of our dictionary. Let $R > 0$, partition the interval $[-R,R]$ into $N = 2R/\delta$ sub-intervals\footnote{We assume that $R$ is a multiple of $\delta$ and omit the explicit dependence of $\cM$ on $\delta$ and $R$ for brevity.} of length $\delta > 0$, and let the set of endpoints be $\{b_1,\dots,b_{N+1}\}$ with $b_1 = -R$ and $b_{N+1} = R$. For convenience, we denote by $\cB =\{b_1,\ldots, b_{N}\}$ and set $b_0\deq b_1$, $b_{N+2} \deq b_{N+1}$. The function classes are given by
\begin{equation} \label{def: piecewise linear classes}
\begin{aligned}
    \cM_0 & \deq  \Bigl\{x \mapsto \psi\Bigl(\frac{x_1- b}{\delta}\Bigr)\, e_1  \,\,\Bigl| \,\, b \in \cB \Bigr\}\,,\\
\cM_1 & \deq  \Bigl\{x \mapsto \psi\Bigl(\frac{x_i - b}{\delta}\Bigr) \psi\Bigl(\frac{ x_1 - b' }{\delta}\Bigr) \mathbbm{1}_{\left\{x_1 \in [b', b'+\delta) \right\}} \, e_i \,\, \Bigl| \,\,i \geq 2,\, b,b'\in \cB \Bigr\}\,, \\
\cM_2 & \deq  \Bigl\{x \mapsto \psi\Bigl(\frac{x_i - b}{\delta}\Bigr) \psi\Bigl(1 - \frac{ x_1 - b' }{\delta}\Bigr) \mathbbm{1}_{\left\{x_1 \in [b', b'+\delta) \right\}} \, e_i \,\, \Bigl| \,\,i \geq 2,\, b,b'\in \cB \Bigr\}\,, \\
      \cM_3 & \deq  \Bigl\{x \mapsto \psi\Bigl(\frac{x_i- b}{\delta}\Bigr) \mathbbm{1}_{\left\{x_1 \geq R \right\}}\, e_i  \,\,\Bigl|\,\,i \geq 2,\, b\in \cB \Bigr\}\,,\\
         \cM_4 & \deq  \Bigl\{x \mapsto \psi\Bigl(\frac{x_i- b}{\delta}\Bigr) \mathbbm{1}_{\left\{x_1 <-R \right\}}\, e_i  \,\,\Bigl| \,\,i \geq 2,\, b\in \cB \Bigr\}\,,\\
    \cM_5 & \deq  \Bigl\{x \mapsto \pm\psi\Bigl(\frac{x_1- b}{\delta}\Bigr)\, e_i  \,\,\Bigl| \,\,i \geq 2,\, b\in \cB \Bigr\}\,,\\
\cM & \deq  \cM_0 \cup \cM_1\cup\cM_2\cup\cdots\cup \cM_5\,.
\end{aligned}
\end{equation}

We note that $\cM_0$ consists of non-decreasing piecewise linear functions of $x_1$. The maps in $\cM_1$ are piecewise linear maps depending jointly on $x_i$ and $x_1$, which are non-decreasing over the set $(x_1,x_i) \in [b', b'+\delta) \times \R$.  $\cM_2$ is built on $\cM_1$ by flipping the sign of $x_1$. Each element of $\cM_2$ is a a monotonically increasing function in $x_i$ but monotonically decreasing in $x_1$ for $x_1 \in [b', b'+\delta)$. Finally, $\cM_3,\cM_4,\cM_5$ are designed to approximate $T_i^\star$ when $x_1, x_i \notin [-R, R]$.

We further enrich $\cM$ with constant coordinate shifts $\pm e_i$, $i \in [d]$ to handle tail behavior. Equivalently, we define the \emph{augmented cone}:
\begin{equation*}
    \underline{\rm cone} (\cM)  \deq  \Bigl\{\sum_{T \in \cM} \lambda_T T + v \Bigm\vert \lambda \in \R_+^{|\cM|},\, v \in \R^d  \Bigr\}\,.
\end{equation*}
Finally, by \cref{thm:star_caff}, for $z_1 \in \R$ and $i \in \{2,\dots,d\}$,
    \begin{align*}
        \sqrt{1/L_V'} \leq \left| \partial_{x_1} T_1^\star(\cdot) \right| \leq \sqrt{1/\ell_V'}\,, \qquad 
        \sqrt{1/L_V} \leq \left| \partial_{x_i} T_i^\star(\cdot \mid z_1) \right| \leq \sqrt{1/\ell_V}\,.
    \end{align*}
Define $\balpha \deq \big((L_V')^{-1/2}, L_V^{-1/2},\dots,L_V^{-1/2}\big)$.
We arrive at the \emph{$\balpha$-augmented-and-spiked cone}:
\begin{equation}\label{eq:augmented-spiked-cone}
    \underline{\rm cone} (\cM;\, \balpha \id)  \deq   \balpha \,{\id} + \underline{\rm cone} (\cM)\,.
\end{equation}
Thus, $\underline{\rm cone}(\cM;\, \balpha \id) \subset \cT_{\rm star}$.

\subsubsection{The approximation procedure}
We now show that the dictionary in the previous subsection is such that there exists an approximator $\widehat T \in \underline{\rm cone} \left(\cM;\, \balpha\, {\mathrm{id}}\right)$ which approximates $T^\star$ to arbitrary precision (see \cref{coro:approximation bound}).

Recall that the optimal star-separable map has the structure
$$T^\star(x) = \left[T^\star_1(x_1), T^{\star}_2(x_2\mid T^\star_1(x_1)), \dotsc, T^{\star}_d(x_d\mid T^\star_1(x_1)) \right]\qquad\text{for all}\qquad x\in \R^d\,.$$ 
For clarity, we write $T_i^\star(x_i;x_1) \deq T_i^\star(x_i\mid T^\star_1(x_1))$ to emphasize the dependence of $T_i^\star$ on $x_1$.

To approximate $T^\star$, we first approximate $T^\star_1$ by a univariate piecewise linear function $\widehat T_1$. For $i\geq 2$, we approximate $T_i^{\star}(x_i; x_1)$ by decoupling its dependence on $x_i$ and $x_1$ and constructing approximators for both parts separately.

We restrict $(x_i,x_1)$ to $[-R,R]\times[-R,R]$ for all $i\ge 2$.

To approximate the map $T_1^{\star}$, we define
\begin{equation} \label{eq:T1-approx}
\widehat T_1(x_1) \deq T_1^{\star}(-R) + \sum_{m=1}^{N} \lambda_m\, \psi\Bigl(\frac{x_1 - b_m}{\delta}\Bigr)\,,
\end{equation}
where the non-negative coefficients $\lambda_m \deq T_1^{\star}(b_{m+1}) - T_1^{\star}(b_m) \ge 0$ encode the local increments of $T_1^{\star}$ along the partition grid.

We next construct an approximation for $T_i^{\star}$ with $i \ge 2$. For each endpoint $b_j \in \cB\cup\{b_{N+1}\}$, define the localized approximation as
\begin{equation} \label{eq:Tib-approx}
\widehat T_i(x_i; b_j) \deq T_i^{\star}(-R; b_j) + \sum_{m=1}^{N} \lambda_{m,j}\, \psi\Bigl(\frac{x_i - b_m}{\delta}\Bigr)\,,
\end{equation}
where $\lambda_{m,j} \deq T_i^{\star}(b_{m+1}; b_j) - T_i^{\star}(b_m; b_j) \ge 0$.

For $x_1 \in I_j = [b_j, b_{j+1})$, $j\in [N]$, we interpolate between adjacent approximation maps:
\begin{equation}\label{eq:Ti approx}
\begin{aligned}
    \widehat T_i(x_i;  x_1)
    &\deq \frac{b_{j+1} - x_1}{\delta} \, \widehat T_i(x_i;  b_j) + \frac{x_1 - b_j}{\delta}\,    \widehat T_i(x_i;  b_{j+1})\\
    &= \psi\Bigl(1-\frac{x_1 - b_j}{\delta}\Bigr)\, \widehat T_i(x_i;  b_j) +\psi\Bigl(\frac{x_1 - b_j}{\delta}\Bigr)\,  \widehat T_i(x_i;  b_{j+1})\\
    &= \widehat T_i(x_i;  b_j) +  \psi\Bigl(\frac{x_1 - b_j}{\delta}\Bigr) \,\bigl( \widehat T_i(x_i;  b_{j+1}) - \widehat T_i(x_i;  b_j)\bigr)\,.
\end{aligned}
\end{equation}
When $x_i\in I_j$, $j\in \{0,N+1\}$, we define 
\[
\widehat T_i(x_i;  x_1) \deq \widehat T_i(x_i; b_j)\,.
\] 

Combining all the above, we conclude the construction of the approximator $\widehat{T}_i$:
\begin{align*}
        \widehat T_i(x_i;  x_1)
        &= \widehat T_i(x_i; -R) \mathbbm{1}_{\left\{x_1 <-R \right\}}  + \widehat T_i(x_i; R) \mathbbm{1}_{\left\{x_1 \geq R \right\}}   \\
        &\qquad +\sum_{j=1}^N \biggl(\psi\Bigl(1-\frac{x_1 - b_j}{\delta}\Bigr)\, \widehat T_i(x_i;  b_j) +\psi\Bigl(\frac{x_1 - b_j}{\delta}\Bigr)\, \widehat T_i(x_i;  b_{j+1})\biggr)\mathbbm{1}_{\left\{x_1 \in [b_j, b_{j+1}) \right\}}\\
     &=  T_i^\star(-R,-R)  + \sum_{j=1}^N \psi\Bigl(\frac{x_1 - b_j}{\delta}\Bigr)\, \bigl(T_i^{\star}(-R; b_{j+1}) - T_i^{\star}(-R; b_j)\bigr) \\
     &\qquad + \sum_{m=1}^N  \biggl\{\lambda_{m, 0}\, \psi\Bigl(\frac{x_i - b_m}{\delta}\Bigr)\mathbbm{1}_{\left\{x_1 <-R \right\}}  + \lambda_{m, N+1}\, \psi\Bigl(\frac{x_i - b_m}{\delta}\Bigr)\mathbbm{1}_{\left\{x_1 \geq R \right\}}\\
     &\qquad\qquad +\lambda_{m,j}\, \psi\Bigl(\frac{x_i - b_m}{\delta} \Bigr)\, \psi\Bigl(1-\frac{x_1 - b_j}{\delta}\Bigr) \mathbbm{1}_{\left\{x_1 \in [b_j, b_{j+1}) \right\}} \\
     &\qquad\qquad + \lambda_{m,j+1}\, \psi\Bigl(\frac{x_i - b_m}{\delta}\Bigr)\, \psi\Bigl(\frac{x_1 - b_j}{\delta}\Bigr) \mathbbm{1}_{\left\{x_1 \in [b_j, b_{j+1}) \right\}} \biggr\}\,.
\end{align*}
We remark that the proposed piecewise linear interpolation yields a continuous approximation in $x_1$ while preserving monotonicity in $x_i$. Moreover, it is straightforward to verify that $\widehat{T} \in \underline{{\rm cone}}(\cM)$.

\begin{lemma}\label{lemma:hat derivative}
 Assume \eqref{assum:R}, \eqref{assum:P1}, \eqref{assum:P2}, and \eqref{assum:RD+} hold.  For $i \geq 2$, the approximator $\widehat{T}_i$ defined under~\eqref{eq:Tib-approx}-\eqref{eq:Ti approx} satisfies the following: 
 \begin{enumerate}[label = (\roman*)]
  \item When  $x_1\in I_0\cup I_{N+1}$, $\partial_{x_1}\widehat T_i(x_i;x_1)=0$.
 \item When  $x_i\in I_0\cup I_{N+1}$, $\partial_{x_i}\widehat T_i(x_i;x_1)=0$.
     \item For all $x_i, x_1 \in \R$, 
\begin{equation*}
    |\partial_{x_i}\widehat T_i(x_i; x_1)|\le 1/\sqrt{\ell_V} \quad \text{and} \quad   |\partial_{x_1} \widehat{T}_i(x_i ; x_1) | \leq  \frac{\overline{L}}{\ell_V \sqrt{\ell_V'}}\,.
\end{equation*}
Furthermore,
\begin{equation*}
|\partial_{x_i} \partial_{x_1} \widehat{T}_i(x_i ; x_1)| \leq  \Bigl(\frac{L_V}{\ell_V^{3/2}}+ \frac{L_V \overline{L}}{\ell_V^2\sqrt{\ell_V'}} \Bigr)\, (1+|x_i| +\delta)\,. 
\end{equation*}
 \end{enumerate}
\end{lemma}
\begin{proof}
(i) and (ii) hold trivially by construction. We now show that~(iii) holds. 
Let $\eta_{k, j} \deq T_i^\star(b_k; b_{j+1}) -T_i^\star(b_k; b_j) $. Note the following identity:
\begin{equation*}
   \lambda_{k,j} + \eta_{k+1,j} = T_{i}^\star(b_{k+1};b_{j+1}) - T_i^\star(b_k;b_j)=  \lambda_{k,j+1} + \eta_{k,j}\,.
\end{equation*}
Let $\Delta_{j, k} \deq \lambda_{k,j+1}  - \lambda_{k,j} =  \eta_{k+1,j} -  \eta_{k,j}$. For all $x_1\in I_j$, $x_i\in I_k$, we first compute the partial gradients of $\widehat{T_i}$ w.r.t.\ $x_1,x_i$ explicitly.  Write  $u_k(x_i) = \psi\bigl(\frac{x_i - b_k}{\delta}\bigr)$ and $w_j(x_1) =\psi\bigl(\frac{x_1 - b_j}{\delta}\bigr)$, then 
\begin{align*}
    \partial_{x_1} \widehat{T}_i(x_i ; x_1) &= \frac{\widehat{T}_i(x_i;b_{j+1})- \widehat{T}_i(x_i;b_j)}{\delta}\\
    & = (1-u_k(x_i))\, \frac{T_i^\star(b_{k};b_{j+1}) - T_i^\star(b_{k};b_{j})}{\delta} \\
    &\qquad{} + u_k(x_i)\,  \frac{T_i^\star(b_{k+1};b_{j+1}) - T_i^\star(b_{k+1};b_{j})}{\delta}
\end{align*}
and 
\begin{align*}
    \partial_{x_i} \widehat{T}_i(x_i ; x_1) &= (1-w_j(x_1))\, \partial_{x_i}\widehat T_i(x_i;  b_j) + w_j(x_1)\,  \partial_{x_i} \widehat T_i(x_i;  b_{j+1})\\
    & =(1-w_j(x_1)) \, \frac{ T_i^{\star}(b_{k+1}; b_j) - T_i^{\star}(b_{k}; b_j)}{\delta} \\
    &\qquad{} + w_j(x_1) \,  \frac{ T_i^{\star}(b_{k+1}; b_{j+1}) - T_i^{\star}(b_{k}; b_{j+1})}{\delta}\,.
\end{align*}

Using the definitions of $\lambda_{k, j}$ and $\eta_{k, j}$, we obtain the forms
\begin{equation} \label{eqn:hat derivative}
    \partial_{x_1} \widehat{T}_i(x_i ; x_1) = \frac{\eta_{k,j}}{\delta} + u_k(x_i)\, \frac{\Delta_{j, k}}{\delta} \quad \text{and} \quad \partial_{x_i} \widehat{T}_i(x_i ; x_1) = \frac{\lambda_{k,j}}{\delta} + w_j(x_1)\,\frac{\Delta_{j, k}}{\delta}\,. 
\end{equation}
By \cref{thm:star_caff} and the mean value theorem, we derive that
\begin{equation}\label{eq:lambda-kj}
    0\leq \frac{\lambda_{k,j}}{\delta} \leq \frac{1}{\sqrt{\ell_V}}
\end{equation}
and for some $\eta_j \in [b_j,b_{j+1}]$, 
\begin{equation}\label{eq:eta-kj}
    \Bigl|\frac{\eta_{k,j}}{\delta}\Bigr| = \frac{1}{\delta}\,\bigl|T_i^\star(b_{k};b_{j+1})  - T_i^\star(b_k;b_j) \bigr| \leq  \bigl|\partial_{2} T_i^\star(b_{k}\mid T_1^\star(\eta_j))\bigr|\,|\partial_{x_1} T_1^\star(\eta_j)|\leq  \frac{\overline{L}}{\ell_V \sqrt{\ell_V'}}\,.
\end{equation}
Similarly, for some $\xi_k,\chi_k \in I_k$, $\beta_k$ between $\xi_k, \chi_k$, and $\gamma_j \in [b_j,b_{j+1}]$, 
\begin{align*}
    \Bigl|\frac{\Delta_{j,k}}{\delta}\Bigr| &= \left|\frac{\lambda_{k,j+1}  - \lambda_{k,j} }{\delta}\right| = \left|\frac{ T_i^\star(b_{k+1};b_{j+1})  - T_i^\star(b_k;b_{j+1}) -  (T_i^\star(b_{k+1};b_j)  - T_i^\star(b_k;b_j))}{\delta}\right|\\
    &=  \left|\partial_{x_i} T_i^\star(\xi_k;b_{j+1}) - \partial_{x_i} T_i^\star(\chi_k;b_{j}) \right|\\
    &\leq  \left|\partial_{x_i} T_i^\star(\xi_k;b_{j+1}) - \partial_{x_i} T_i^\star(\chi_k;b_{j+1}) \right| + \left|\partial_{x_i} T_i^\star(\chi_k;b_{j+1}) - \partial_{x_i} T_i^\star(\chi_k;b_{j}) \right|\\
     &\leq  |\partial_{x_i}^2 T_i^\star(\beta_k; b_{j+1})|\, \delta + \left|\partial_{x_i} \partial_2 T_i^\star(\chi_k \mid T_1^\star(\gamma_j)) \right||\partial_{x_1} T_1^\star(\gamma_j)|\,\delta\\
     &\leq \frac{L_V}{\ell_V^{3/2}}\,(1+|\beta_k|)\, \delta + \frac{L_V \overline{L}}{\ell_V^2\sqrt{\ell_V'}}\,(1+|\chi_k|)\,\delta\\
     &\leq \frac{L_V}{\ell_V^{3/2}}\,(1+|x_i| + 
     \delta)\, \delta + \frac{L_V \overline{L}}{\ell_V^2\sqrt{\ell_V'}}\,(1+|x_i| +\delta)\,\delta\,. 
\end{align*}
Thus, 
\begin{equation}\label{eq:delta-jk}
    \Bigl|\frac{\Delta_{j,k}}{\delta^2}\Bigr| \leq \Bigl(\frac{L_V}{\ell_V^{3/2}}+ \frac{L_V \overline{L}}{\ell_V^2\sqrt{\ell_V'}} \Bigr)\, (1+|x_i| +\delta)\,. 
\end{equation}
By~\eqref{eqn:hat derivative}--\eqref{eq:delta-jk}, we conclude that for all $x_i, x_1 \in \R$, 
\begin{equation*}
    0\leq \partial_{x_i} \widehat{T}_i(x_i ; x_1) \leq \frac{1}{\delta} \max\{\lambda_{k, j}, \lambda_{k, j+1}\} \leq \frac{1}{\sqrt{\ell_V}}\,,
\end{equation*}
\begin{equation*}
    |\partial_{x_1} \widehat{T}_i(x_i ; x_1)| \leq \frac{1}{\delta}\max\{|\eta_{k, j}|, |\eta_{k+1, j}|\}\leq  \frac{\overline{L}}{\ell_V \sqrt{\ell_V'}}\,,
\end{equation*}  
and 
\begin{equation*}
| \partial_{x_i} \partial_{x_1} \widehat{T}_i(x_i ; x_1)| = \Bigl| \frac{\Delta_{j, k}}{\delta^2} \Bigr| \leq \Bigl(\frac{L_V}{\ell_V^{3/2}}+ \frac{L_V \overline{L}}{\ell_V^2\sqrt{\ell_V'}} \Bigr)\, (1+|x_i| +\delta)\,.
\end{equation*}
This concludes the proof.
\end{proof}

The next result states that $T_i^{\star}$ can be well approximated by $\widehat T_i$ when the leaf variable is fixed.
\begin{lemma}\label{lemma:T_1 error bound}
Let \eqref{assum:R}, \eqref{assum:P1}, and \eqref{assum:RD+} hold. The estimate $\widehat T_1$ of $T^\star_1$ defined in~\eqref{eq:T1-approx} satisfies
\begin{equation*}
    \|\widehat T_1 - T^\star_1 \|^2_{L^2(\rho)} \lesssim \, \frac{1}{\ell_V'}\exp(-R^2/2) +  \frac{(L_V')^{2}}{(\ell_V^{'})^3}\, \delta^4\,, 
\end{equation*}
and 
\begin{equation*}
    \|(\widehat T_1 - T^\star_1)'\|^2_{L^2(\rho)} \lesssim \,\frac{1}{\ell_V' R} \exp(-R^2/2 ) +  \frac{(L_V')^{2}}{(\ell'_V)^3}\,\delta^2\,. 
\end{equation*}
If \eqref{assum:P2} also holds, then for any $b\in \cB \cup \{b_{N+1} \}$, the estimate $\widehat T_i^1(\cdot; b)$ defined in \eqref{eq:T1-approx} satisfies
\begin{equation*}
    \|\widehat T_i(\cdot; b) - T_i^{\star}(\cdot; b) \|^2_{L^2(\rho)} \lesssim \, \frac{1}{\ell_V}\exp(-R^2/2 )+  \frac{L_V^2}{\ell_V^3}\, \delta^4\,, 
\end{equation*}
and 
\begin{equation*}
    \|\partial_{x_i} \widehat T_i(\cdot; b) - \partial_{x_i} T_i^{\star}(\cdot; b)\|^2_{L^2(\rho)} \lesssim \,\frac{1}{\ell_V R} \exp(-R^2/2 ) +  \frac{L_V^2}{\ell_V^3}\delta^2\,. 
\end{equation*}
\end{lemma}
\begin{proof}
It suffices to show the first pair of bounds; the second pair follows by the same argument. Following the proof of Theorem 5.6 in \cite{JiaChePoo25MFVI}, the pointwise bounds hold for $|x_1|\geq R$:
\begin{equation*}
|\widehat T_1(x_1) - T^\star_1(x_1)|\le \frac{1}{\sqrt{\ell_V'}}\,(|x_1|-R)\,,\qquad |(\widehat T_1 - T^\star_1)'(x_1)|\le \frac{1}{\sqrt{\ell_V'}}\,.
\end{equation*}
Combining these and integrating by parts yields
\begin{equation*}
\begin{aligned}
\int_{\R \setminus (-R,R)} |\widehat T_1 - T^\star_1|^2\,\dd\rho_1
& \lesssim \frac{1}{\ell_V'} \int_{\R \setminus (-R,R)} (|x_1|-R)^2 \rho_1(\dd x_1) 
\lesssim \frac{1}{\ell_V'}\,e^{-R^2/2}\,.
\end{aligned}
\end{equation*}
To bound the derivative term, we use the Mills ratio bound~\citep[see, e.g.,][Exercise 2.2]{Wainwright2019} to get 
\begin{equation*}
\begin{aligned}
\int_{\R \setminus (-R,R)} |(\widehat T_1 - T^\star_1)'|^2\,\dd \rho_1
\lesssim \frac{1}{\ell_V'} \int_{\R \setminus (-R,R)} \rho_1(\dd x_1) 
\lesssim \frac{1}{\ell_V' R}\,e^{-R^2/2}\,.
\end{aligned}
\end{equation*}
We now control the differences on the compact interval $[-R,R]$. By the definition of $\widehat T_1$ and the mean value theorem, for any $x_1\in I_j$ with $j\in [N]$, there exist $\zeta_j,\eta_j,\chi_j\in [b_j,b_{j+1}]$ such that 
\begin{equation*}
\begin{aligned}
|\widehat T_1(x_1) - T^\star_1(x_1)|^2
&= \Bigl|\frac{T^\star_1(b_{j+1}) - T^\star_1(b_j)}{\delta}\,(x_1 - b_j) + T^\star_1(b_j)- T^\star_1(x_1)\Bigr|^2 \\
&= \left|(T_1^\star)'(\zeta_{j})\,(x_1 - b_j) - (T_1^\star)'(\eta_{j})\,(x_1-b_j)\right|^2 \\
&= \big|(T_1^\star)''(\chi_j)\, (x_1 - b_j)\,(\zeta_j - \eta_j)\big|^2 \\
&\le \delta^4\, \big|(T_1^\star)''(\chi_j)\big|^2\,.
\end{aligned}
\end{equation*}
Using the second derivative bound on $T^\star_1$ in \cref{thm:star_caff}, 
\begin{equation*}
\begin{aligned}
\sum_{j\in [N]}\int_{I_j}|\widehat T_1 - T^\star_1|^2\,\dd \rho_1
& \le \delta^4 \sum_{j\in [N]}\int_{I_j} (1+ |\chi_j|)^2\,\frac{(L_V')^2}{(\ell_V')^3}\, \rho_1(\dd x_1)\\ 
& \le \delta^4 \sum_{j\in [N]}\int_{I_j} (1+ |x_1| + \delta)^2\, \frac{(L_V')^2}{(\ell_V')^3}\,\rho_1(\dd x_1)\\
& \lesssim \frac{(L_V')^2}{(\ell_V')^3}\, \delta^4 \int (1+ |x_1|^2 + \delta^2)\, \rho_1(\dd x_1) 
\lesssim \frac{(L_V')^2}{(\ell_V')^3}\, \delta^4\,.
\end{aligned}
\end{equation*}
Similarly, for the derivative difference, for some $\zeta_j,\beta_j \in [b_j,b_{j+1}]$ (depending on $x_1$),
\begin{equation*}
\begin{aligned}
|(\widehat{T}_1 - T_1^\star)'(x_1)|^2
&= \Bigl|\frac{T^\star_1(b_{j+1}) - T^\star_1(b_j)}{\delta} - (T_1^\star)'(x_1)\Bigr|^2 \\
&= \big|(T_1^\star)'(\zeta_{j}) - (T_1^\star)'(x_1)\big|^2\\  
&= \big|(T_1^\star)''(\beta_j)\, (\zeta_j -x_1)\big|^2
\le \delta^2 \, \big|(T_1^\star)''(\beta_j)\big|^2\,.
\end{aligned}
\end{equation*}
Applying the same second derivative bound,
\begin{align*}
\sum_{j\in [N]}\int_{I_j}|(\widehat{T}_1 - T_1^\star)'|^2\,\dd \rho_1
& \le \delta^2 \sum_{j\in [N]}\int_{I_j} (1+ |\beta_j|)^2\,\frac{(L_V')^2}{(\ell_V')^3}\, \rho_1(\dd x_1)\\ 
& \le \delta^2 \sum_{j\in [N]}\int_{I_j} (1+ |x_1| + \delta)^2\, \frac{(L_V')^2}{(\ell_V')^3}\,\rho_1(\dd x_1) \\
& \lesssim \frac{(L_V')^2}{(\ell_V')^3}\, \delta^2\,.
\end{align*}
\end{proof}

We now prove our main approximation result.
\begin{lemma}\label{lemma: approximation error explicit bound}
Let \eqref{assum:R}, \eqref{assum:P1},  \eqref{assum:P2}, and \eqref{assum:RD+} hold. Then the approximator $\widehat T$ defined in~\eqref{eq:T1-approx}--\eqref{eq:Ti approx} satisfies
\begin{align*}
        \| \widehat T - T^\star \|^2_{L^2(\rho)}
        &\lesssim \Bigl(\frac{1}{\ell_V'} + \frac{d}{\ell_V} + \frac{L_V'}{\ell_V\ell_V'} \Bigr) \exp(-R^2/2)
+ \frac{d\overline{L}^2}{\ell_V^2\ell_V'}\,\delta^2 + \Bigl(\frac{(L_V')^{2}}{(\ell_V^{'})^3} + \frac{dL_V^2}{\ell_V^3}\Bigr)\, \delta^4\,.
\end{align*}
Assume further that \eqref{assum:GR} holds.
Then,
\begin{multline*}
    \|D(\widehat T -T^\star)\|^2_{L^2(\rho)}\lesssim \Bigl(\frac{1}{\ell_V'}  + \frac{d}{\ell_V} + \frac{d\overline L^2}{\ell_V^2 \ell_V'}\Bigr)\,\frac{\exp(-R^2/2)}{R} \\+ \Bigl[\Bigl(\frac{1}{d} + \frac{\overline{L}^2}{\ell_V^2}\Bigr)\frac{(L_V')^{2}}{(\ell'_V)^3}+ \frac{L_V^2 }{\ell_V^3}  + \frac{L_V^2 \overline{L}^2}{\ell_V^4 \ell_V'} + \frac{(1+ M_{2\gamma}(\rho_1))\,\LGR^2}{(\ell_V')^2} \Bigr]\,d\delta^2 \,.
\end{multline*}
\end{lemma}
\begin{remark}
\begin{enumerate}[label = (\alph*)]
\item The error bound for $\|\widehat T - T^\star\|_{L^2(\rho)}^2$ arises from three sources:  
(i) the error in approximating $T_1^\star$;  
(ii) the error in approximating the first argument of $T_i^\star(x_i; x_1)$; and  
(iii) the error in approximating the second argument of $T_i^\star(x_i; x_1)$.

\item The error bound for $\|D(\widehat T - T^\star)\|^2_{L^2(\rho)}$ consists of: (i) the approximation error for $DT^\star_1$; (ii) 
the error in approximating the derivative $\partial_{x_i} T_i^{\star}(x_i;x_1)$ with respect to the first argument; 
and (iii) the error in approximating the derivative with respect to the root variable, i.e., $|\partial_{x_1}T_i^{\star}(x_i;x_1) - \partial_{x_1}\widehat T_i(x_i;x_1)|^2$.
\item Under the setting of \cref{thm: Gaussian}, the star-separable map $T^\star = (T^\star_1, T^\star_2, \dotsc, T_d^{\star})$ from $\cN(0,I)$ to $\pi^\star = \cN(m^\star, \Sigma^\star)$ has the explicit form
\begin{equation*}\label{eq: star-separable map between normal distributions}
\begin{aligned}
        T^\star_1(x_1) &= m_1 + \sqrt{\sigma_{11}}\, x_1\,, \\
        T_i^{\star}(x_i;x_1) &= m_i + \sigma_{i1}\,(\sigma_{11})^{-1/2}\,(T_1^\star(x_1)-m_1) + \sqrt{(\sigma^{ii})^{-1}}\, x_i\,, \qquad i \ge 2\,.
\end{aligned}
\end{equation*}
Therefore, \eqref{assum:GR} is satisfied, and the approximation is exact up to the truncation error.
\end{enumerate}
\end{remark}
Given Lemma~\ref{lemma: approximation error explicit bound}, we can establish the following result, which shows the \emph{existence} of a map $\widehat T$ in our construction which satisfies the desired approximation guarantee.
\begin{corollary}\label{coro:approximation bound}
Assume \eqref{assum:R}, \eqref{assum:P1}, \eqref{assum:P2}, \eqref{assum:RD+}, and \eqref{assum:GR} hold. 
 Then, for any $\epsilon > 0$, there exists a closed convex set $\Theta$ of dimension $O((d^2/\epsilon^2) \log(d/\epsilon^2))$ and an affine parametrization $\theta \mapsto T_\theta$, where the implied constant depends polynomially on the constants $\ell_V, L_V, \ell_V', L_V', \overline{L}, \LGR, M_{2\gamma}(\rho_1)$ and their inverses, where $M_{2\gamma}(\rho_1)$ is the $2\gamma$-th absolute moment of $\rho_1$, such that there exists $\hat{\theta} \in \Theta$ satisfying
\begin{equation*}
    \|T_{\hat{\theta}} - T^{\star}\|_{L_2(\rho)} \leq \epsilon\qquad \text{and}\qquad \|D(T_{\hat\theta} - T^\star)\|_{L^2(\rho)} \le \epsilon\,.
\end{equation*}
\end{corollary}
\begin{remark}
\begin{enumerate}[label = (\alph*)]
    \item The proof of \cref{lemma: approximation error explicit bound} and \cref{coro:approximation bound} can be adapted without changes to produce a map $\widehat T \in \underline{\mathrm{cone}}(\cM;\,\balpha\,\mathrm{id})$ that approximates the optimal star-separable map $T^\star$ and its gradient arbitrarily well in $L^2(\rho)$. Indeed, we can approximate $T^\star-\balpha\,\mathrm{id}$ by some $\widehat T_0 \in \underline{\mathrm{cone}}(\cM)$ and set $\widehat T \coloneqq \widehat T_0 + \balpha\,\mathrm{id}$.
\item  Notice that the error rates for approximating $T^\star$ and $DT^\star$ are both of order $\vae$, which differs from the error rates in \cite{JiaChePoo25MFVI} for estimating univariate optimal transport maps.
In \cite{JiaChePoo25MFVI}, the functional value approximation bound scales as $\vae^2$, while the derivative approximation scales as $\vae$. 
The dominant error here comes from approximating each $T_i^\star$ in its second argument, i.e., $|T_i^{\star}(x_i;x_1) - T_i^{\star}(x_i;b_j)|$, which scales as $\delta \asymp \vae$.
\end{enumerate}
\end{remark}

\begin{proof}
Given $R$ and $\delta$, the size of $\cM_0$ is $O(R / \delta)$, the sizes of $\cM_1$, $\cM_2$ are each at most $O\!\left((d-1)\,(R/\delta)^2\right)$, while the sizes of $\cM_3$, $\cM_4$, $\cM_5$ are each $O\!\left((d-1)\,(R/\delta)\right)$. Therefore, the size of $\cM$ is $|\cM_0| + |\cM_1 | + |\cM_2| + |\cM_3| + |\cM_4| + |\cM_5| = O(d\, \left(R/\delta\right)^2 )$. In what follows, we suppress polynomial dependence on $L_V$, $\ell_V$, etc.\ in the asymptotic notation $\lesssim$, $\asymp$, keeping only the dependence on the dimension $d$ and the accuracy $\epsilon$.

Choose $R \asymp \sqrt{\log (d/\epsilon^2)}$ and $\delta \asymp \epsilon/\sqrt d$. Then, from the first bound of Lemma~\ref{lemma: approximation error explicit bound}, it is straightforward to see that $\|\widehat T - T^\star\|_{L^2(\rho)} \le \epsilon$. Moreover, from the second bound on Lemma~\ref{lemma: approximation error explicit bound}, these choices yield $\|D(\widehat T - T^\star)\|_{L^2(\rho)} \le \epsilon$ as well, which we will need for later reference.
\end{proof}
Finally, we introduce the following control on the second-order growth of $T_i^\star$ w.r.t.\ $x_1$, which serves as the final ingredient to establish \cref{lemma: approximation error explicit bound}.
\begin{lemma}\label{lemma:x1-hessian}
Assume \eqref{assum:R}, \eqref{assum:P1}, \eqref{assum:P2}, \eqref{assum:RD+}, and \eqref{assum:GR} hold.  
Then, for all $i \ge 2$ and $x_i, x_1 \in \R$, 
\begin{equation*}
\partial_{x_1}^2 T_i^\star(x_i; x_1)
\;\lesssim\;
\frac{\LGR}{\ell_V'}\, (1 + |x_i|)^{\gamma}
+ \frac{L_V'\,\overline{L}}{(\ell_V')^{3/2}\,\ell_V}\, (1 + |x_1|)\,. 
\end{equation*}
\end{lemma}
\begin{proof}
By the chain rule,
\begin{align*}
\partial_{x_1}^2 T_i^\star(x_i; x_1)
&= \partial_2^2 T_i^{\star}(x_i \mid T_1^\star(x_1))\, \bigl(\partial_{x_1} T_1^\star(x_1)\bigr)^2
   + \partial_2 T_i^{\star}(x_i \mid T_1^\star(x_1))\, \partial_{x_1}^2 T_1^\star(x_1)\,.
\end{align*}
By \cref{thm:star_caff}, the derivatives satisfy
\begin{equation*}
\bigl|\partial_{x_1} T_1^\star(x_1)\bigr|^2 \le \frac{1}{\ell_V'}\,, 
  \quad \bigl| \partial_2 T_i^{\star}(x_i \mid T_1^\star(x_1))\bigr|\leq \frac{\overline{L}}{\ell_V}\,,\quad
\bigl|\partial_{x_1}^2 T_1^\star(x_1)\bigr|
  \lesssim \frac{L_V'}{(\ell_V')^{3/2}}\, (1 + |x_1|)\,.
\end{equation*}
In addition, \eqref{assum:GR} implies that 
\[
\bigl|\partial_2^2 T_i^{\star}(x_i \mid T_1^\star(x_1))\bigr|
  \lesssim \LGR\, (1 + |x_i|)^{\gamma}\,.
\]
Combining these bounds with \cref{thm:star_caff} and applying the triangle inequality yields
\begin{equation*}
\bigl|\partial_{x_1}^2 T_i^\star(x_i; x_1)\bigr|
\;\lesssim\;
\frac{\LGR}{\ell_V'}\, (1 + |x_i|)^{\gamma}
+ \frac{L_V'\,\overline{L}}{(\ell_V')^{3/2}\,\ell_V}\, (1 + |x_1|)\,. 
\end{equation*}
\end{proof}
Now, we are ready to prove Lemma~\ref{lemma: approximation error explicit bound}.

\begin{proof}[Proof of \cref{lemma: approximation error explicit bound}]
\textbf{Approximation error of $\widehat T - T^\star$.}
First, we bound the approximation error $\|\widehat T - T^\star\|_{L^2(\rho)}$. For every $i \in [d]$ and $j,k \in \{0,\ldots, N+1\}$, define
\begin{equation*}
    S_{j,k}^{(i)} \deq \{x\in \R^d: x_1 \in I_j,\, x_i \in I_k \}\,,
\end{equation*}
where $I_j=[b_j,b_{j+1})$, $I_k = [b_k,b_{k+1})$ are the subintervals.

We adopt the shorthand: $\sum_j \deq \sum_{j=0}^{N+1}$, $\sum_k \deq \sum_{k=0}^{N+1}$, $\sum_l \deq \sum_{l=0}^{N+1}$. 
Write $\partial_1,\partial_2$ for differentiation with respect to the first and second arguments, respectively; e.g., $\partial_1 T_i^\star(x_i\mid z_1)=\partial_{x_i}T_i^\star(x_i\mid z_1)$ and $\partial_2 T_i^\star(x_i\mid z_1)=\partial_{z_1}T_i^\star(x_i\mid z_1)$.  Recall $T_i^\star(x_i;x_1) \deq T_i^\star(x_i\mid T_1^\star(x_1))$ and $\rho=\cN(0,I)=\bigotimes_{i=1}^d \rho_i$ with $\rho_1=\cdots=\rho_d=\cN(0,1)$.

We expand
\begin{equation*}\label{eq: L^2 error between OT maps}
\begin{aligned}
\|\widehat T - T^\star \|^2_{L^2(\rho)}  
&= \int \bigl|\widehat T_1(x_1) - T^\star_1(x_1)\bigr|^2 \,\rho_1(\dd x_1)
+ \sum_{i = 2}^d \int \bigl|\widehat T_i(x_i; x_1) - T_i^{\star}(x_i; x_1)\bigr|^2  \,\rho(\dd x)\,.
\end{aligned}
\end{equation*}

By the definition of $\widehat T_i$, we know that for $x_1\in I_j$, 
\begin{equation*}
  \widehat T_i(x_i;x_1) =  (1- w_j(x_1))\, \widehat T_i(x_i;  b_j) +w_j(x_1)\,  \widehat T_i(x_i;  b_{j+1})\,,
\end{equation*}
where we recall that $w_j(x_1)=\psi\bigl(\frac{x_1-b_j}{\delta}\bigr) \in [0,1]$ and the convention that $b_0=b_1= -R$ and $b_{N+1} = b_{N+2} = R$. In particular, we observe that when $j \in \{0,N+1\}$, then $\widehat T_i(x_i;x_1)  =  \widehat T_i(x_i;  b_{j})$.

Applying the triangle inequality followed by Jensen's inequality yields
\begin{equation*}
\begin{aligned}
&\sum_{j,k} \int_{S^{(i)}_{j,k}} \bigl|\widehat T_i(x_i; x_1) - T_i^{\star}(x_i; x_1)\bigr|^2  \,\rho(\dd x) \\
&\qquad = \sum_{j,k} \int_{S^{(i)}_{j,k}} \bigl|(1- w_j(x_1))\,  \widehat T_i(x_i;  b_j) + w_j(x_1)\, \widehat T_i(x_i;  b_{j+1}) - T_i^{\star}(x_i; x_1)\bigr|^2 \rho(\dd x)\\
&\qquad\lesssim \sum_{j,k} \int_{ S^{(i)}_{j,k}} w_j(x_1)\,\big|\widehat T_i(x_i; b_j) - T_i^{\star}(x_i ;b_j)\big|^2 \,\rho(\dd x)\\
&\qquad\qquad + \sum_{j,k}\int_{ S^{(i)}_{j,k}} (1-w_j(x_1))\,\big|\widehat T_i(x_i; b_{j+1}) - T_i^{\star}(x_i ;b_{j+1})\big|^2 \,\rho(\dd x)  \\
&\qquad\qquad +  \sum_{j,k} \int_{ S^{(i)}_{j,k}} w_j(x_1)\,\bigl|T_i^{\star}(x_i ;b_j)  - T_i^{\star}(x_i ; x_1) \bigr|^2  \,\rho(\dd x) \\
&\qquad\qquad +  \sum_{j,k} \int_{ S^{(i)}_{j,k}} (1-w_j(x_1))\,\bigl|T_i^{\star}(x_i ;b_{j+1})  - T_i^{\star}(x_i ; x_1) \bigr|^2  \,\rho(\dd x)\\
&\qquad \lesssim \sum_{j,k} \int_{ S^{(i)}_{j,k}} \big|\widehat T_i(x_i; b_j) - T_i^{\star}(x_i ;b_j)\big|^2 \,\rho(\dd x)
+ \sum_{j,k} \int_{ S^{(i)}_{j,k}}\left|T_i^{\star}(x_i ;b_j)  - T_i^{\star}(x_i ; x_1) \right|^2  \,\rho(\dd x)\,.
\end{aligned}
\end{equation*}
Combining the bounds above gives 
\begin{equation*}\label{eq: approx 1}
\begin{aligned}
 \int \bigl|\widehat T_i(x_i; x_1) - T_i^{\star}(x_i; x_1)\bigr|^2  \,\rho(\dd x) 
 &\lesssim \underbrace{\sum_{j,k} \int_{ S^{(i)}_{j,k}} \big|\widehat T_i(x_i; b_j) - T_i^{\star}(x_i ;b_j)\big|^2 \,\rho(\dd x)}_{C_i}  \\
&\qquad +  \underbrace{\sum_{j,k} \int_{ S^{(i)}_{j,k}}\left|T_i^{\star}(x_i ;b_j)  -    T_i^{\star}(x_i ; x_1) \right|^2  \,\rho(\dd x)}_{D_i}\,.
\end{aligned}
\end{equation*}
We bound $C_i,D_i$ separately. For $C_i$, by Fubini's theorem and \cref{lemma:T_1 error bound},
\begin{equation}\label{eq: approx C_i 2}
\begin{aligned}
C_i
&= \sum_j \int_{\{x_1 \in I_j\}} \bigl| \widehat T_i (x_i; b_j) - T_i^{\star}(x_i; b_j)   \bigr|^2\, \rho(\dd x)\\
&=    \sum_{j} \rho_1(I_j) \int \bigl|\widehat T_i(x_i; b_j) - T_i^{\star}(x_i; b_j)  \bigr|^2\,\rho_i(\dd x_i)\\
&\lesssim  \Bigl(\frac{1}{\ell_V}\exp(-R^2/2) + \frac{L_V^2}{\ell_V^3}\, \delta^4\Bigr) \sum_{j} \rho_1(I_j) \\
&= \frac{1}{\ell_V}\exp(-R^2/2) + \frac{L_V^2}{\ell_V^3}\, \delta^4\,.
\end{aligned}
\end{equation}
For $D_i$, using Fubini's theorem and then \cref{lemma:perturbation Ti}, we derive that 
\begin{equation*}
\begin{aligned}
\sum_{i=2}^d D_i &=  \sum_{i=2}^d \sum_{j}  \int_{\{x_1 \in I_j\}} \bigl|T_i^{\star}(x_i; b_j)   - T_i^{\star}(x_i ;x_1)\bigr |^2  \rho(\dd x)\\
&= \sum_{j}  \int_{\{x_1 \in I_j\}}\Bigl( \sum_{i=2}^d \int \bigl|T_i^{\star}(x_i \mid T_1^\star(b_j))   - T_i^{\star}(x_i \mid T_1^\star(x_1))\bigr |^2 \rho_i(\dd x_i)\Bigr)\,\rho_1(\dd x_1)\\
& \le \frac{\bigl\|\sum_{i\geq 2}|\partial_{1i}V|^2 \bigr\|_{L^\infty}}{\ell_V^2}  \sum_{j} \int_{I_j} \bigl|T_1^\star(b_j) - T_1^\star(x_1)\bigr|^2\,\rho_1(\dd x_1)\,.
\end{aligned}
\end{equation*}
Furthermore, since $T_1^\star$ is $(1/\sqrt{\ell_V'})$-Lipschitz by \cref{thm:star_caff}~(i), we conclude that
\begin{equation}\label{eq: approx D_i}
\sum_{i=2}^d D_i   \le \frac{\bigl\|\sum_{i \geq 2}|\partial_{1i}V|^2 \bigr\|_{L^\infty}}{\ell_V^2\ell_V'} \sum_{j} \int_{I_j}| b_j - x_1 |^2 \, \rho_1(\dd x_1)\,. 
\end{equation}
Since the intervals $I_j$ have length $\delta$ when $j\in [N]$,
\begin{equation*}\label{Gaussian 2nd bound}
\begin{aligned}
 \sum_{j} \int_{I_j} | b_j - x_1 |^2\, \rho_1(\dd x_1)  &\le  \delta^2 + \int_{-\infty}^{-R} |x_1 + R |^2\, \rho_1(\dd x_1) + \int_{R}^\infty |x_1 -R|^2\, \rho_1(\dd x_1) \\
&\lesssim \delta^2 + \exp( - R^2/2)\,.
\end{aligned}
\end{equation*}
Thus, \eqref{eq: approx D_i} becomes
\begin{equation}\label{eq: approx D_i 2}
   \sum_{i=2}^d D_i \lesssim \frac{\sum_{i\geq 2}\left\||\partial_{1i}V|^2 \right\|_{L^\infty}}{\ell_V^2\ell_V'}\, \bigl(\delta^2 + \exp ( - R^2/2) \bigr) \lesssim \Bigl(\frac{L_V'}{\ell_V\ell_V'} \wedge \frac{d\overline L^2}{\ell_V^2 \ell_V'}\Bigr) \,\bigl(\delta^2 + \exp ( - R^2/2 ) \bigr)\,,
\end{equation}
where the last step uses \eqref{assum:RD} since $\frac{1}{2}\,L_V' - \frac{\bigl\|\sum_{i \geq 2}|\partial_{1i}V|^2 \bigr\|_{L^\infty}}{\ell_V} \geq 0$, and $\overline L \ge \max_{i\ge 2}{\|\partial_{1i} V\|_{L^\infty}}$.

Putting together \eqref{eq: approx C_i 2} and \eqref{eq: approx D_i 2}, and the bound on $\|\widehat T_1 - T^\star_1\|_{L^2(\rho)}^2$ from \cref{lemma:T_1 error bound}, we obtain
\begin{equation*}
\begin{aligned}
\|\widehat T -  T^\star \|^2_{L^2(\rho)}   & \lesssim \int \bigl|\widehat T_1(x_1) - T^\star_1(x_1)\bigr|^2 \,\rho_1(\dd x_1) + \sum_{i =2}^d  C_i +  \sum_{i =2}^d D_i \\
& \lesssim \frac{1}{\ell_V'}\exp(-R^2/2 )+  \frac{(L_V')^{2}}{(\ell_V^{'})^3}\, \delta^4 + \Bigl(\frac{L_V'}{\ell_V\ell_V'} \wedge \frac{d\overline L^2}{\ell_V^2 \ell_V'}\Bigr) \,\bigl(\delta^2 + \exp ( - R^2/2) \bigr) \\
&\qquad{} + 
(d -1)\,\Bigl( \frac{1}{\ell_V}\exp(-R^2/2 )+ \frac{L_V^2}{\ell_V^3}\, \delta^4\Bigr)\,. 
\end{aligned}
\end{equation*}
The right-hand side in the above display can be rearranged into
\begin{align*}
        \| \widehat T - T^\star \|^2_{L^2(\rho)}
        &\lesssim \Bigl(\frac{1}{\ell_V'} + \frac{d}{\ell_V} + \frac{L_V'}{\ell_V\ell_V'} \Bigr) \exp(-R^2/2)
+ \frac{d\overline{L}^2}{\ell_V^2\ell_V'}\,\delta^2 + \Bigl(\frac{(L_V')^{2}}{(\ell_V^{'})^3} + \frac{dL_V^2}{\ell_V^3}\Bigr)\, \delta^4\,,
\end{align*}
as desired.

\textbf{Approximation error of the gradients $D(\widehat T -T^\star)$.} 
We decompose
\begin{equation*}
\begin{aligned}
\|D(\widehat T -T^\star)\|^2_{L^2(\rho)}
&= \int  \bigl|\partial_{x_1}  T^\star_1(x_1) - \partial_{x_1}\widehat T_1(x_1)\bigr|^2\, \rho_1(\dd x_1) \\
&\qquad + \sum_{i = 2}^d \underbrace{\int  \bigl|\partial_{x_i} T_i^{\star}(x_i; x_1) - \partial_{x_i}  \widehat T_i(x_i; x_1)\bigr|^2 \,\rho(\dd x)}_{\eqqcolon A_i} \\
&\qquad +\sum_{i = 2}^d \underbrace{\int  \bigl|\partial_{x_1} T_i^{\star}(x_i; x_1) - \partial_{x_1}  \widehat T_i(x_i; x_1)\bigr|^2 \,\rho(\dd x)}_{\eqqcolon B_i}. 
\end{aligned}
\end{equation*}
\textbf{Bounding $A_i$.}
By \eqref{eqn:hat derivative}, we derive for $x_1\in I_j$,
\begin{equation*}
    \partial_{x_i} \widehat{T}_i(x_i ; x_1)  = (1-w_j(x_1))\,\partial_{x_i}\widehat{T}_i(x_i;b_j)+ w_j(x_1)\, \partial_{x_i}\widehat{T}_i(x_i;b_{j+1})\,.
\end{equation*}
For each $i\geq 2$, applying the triangle inequality implies
\begin{align*}
    &\bigl|\partial_{x_i} T_i^{\star}(x_i; x_1) - \partial_{x_i}  \widehat T_i(x_i; x_1)\bigr|^2 \\
    &\qquad \lesssim\bigl|\partial_{x_i}  T_i^{\star}(x_i ;x_1 ) - (1-w_j(x_1))\,\partial_{x_i}  T_i^{\star}(x_i; b_j) - w_j(x_1)\, \partial_{x_i}  T_i^{\star}(x_i; b_{j+1})\bigr|^2\\
    &\qquad\qquad{} +\bigl|(1-w_j(x_1))\,\bigl(\partial_{x_i}  T_i^{\star}(x_i; b_j) - \partial_{x_i}\widehat T_i(x_i; b_j)\bigr)\bigr|^2 \\
    &\qquad\qquad{} + \bigl|w_j(x_1)\, \bigl(\partial_{x_i}  T_i^{\star}(x_i; b_{j+1}) - \partial_{x_i} \widehat T_i(x_i;b_{j+1}) \bigr) \bigr|^2\\
    &\qquad \lesssim \underbrace{\bigl|\partial_{x_i}  T_i^{\star}(x_i ;x_1 ) - \partial_{x_i}  T_i^{\star}(x_i; b_j) \bigr|^2 + \bigl|\partial_{x_i}  T_i^{\star}(x_i ;x_1 ) - \partial_{x_i}  T_i^{\star}(x_i; b_{j+1})\bigr|^2}_{a_{ij}^\star} \\
    &\qquad\qquad + \underbrace{\bigl|\partial_{x_i}  T_i^{\star}(x_i; b_j) - \partial_{x_i}\widehat T_i(x_i; b_j) \bigr|^2 + \bigl|\partial_{x_i}  T_i^{\star}(x_i; b_{j+1}) - \partial_{x_i}\widehat T_i(x_i; b_{j+1}) \bigr|^2}_{\hat a_{ij}}\,.
\end{align*}
Therefore,
\begin{equation}
\begin{aligned}
A_i &\lesssim 
 \sum_{j}
\underbrace{\int_{\{x_1 \in I_j\}} a_{ij}^\star\,\rho(\dd x)}_{A_{ij}^\star}
+ \underbrace{\sum_{j} \int_{\{x_1\in I_j\}} \hat a_{ij} \,\rho(\dd x)}_{\hat A_{i}}. 
\end{aligned}
\end{equation}
Thus, we can control $A_i$ by the error $A_{ij}^\star$ and the approximation error $\hat A_{i}$.

By the chain rule and \Cref{thm:star_caff}~(ii),
\begin{equation}\label{mixed-derivative-bound-Ti}
\bigl|\partial_{x_1}\partial_{x_i} T_i^{\star}(x_i; x_1)\bigr|
= \bigl|\partial_{x_i}\partial_2 T_i^{\star}(x_i\mid T_1^\star(x_1))\bigr|\,\bigl|\partial_{x_1} T^\star_1(x_1)\bigr|
\lesssim \frac{L_V \overline{L}}{\ell_V^2\sqrt{\ell_V'}}\,(1+|x_i|)\,,
\end{equation}where $\partial_{2} T_i^{\star}(x_i \mid z_1)  \deq  \partial_{z_1} T_i^{\star}(x_i\mid z_1)$. Therefore, using the mean value theorem, we know that for $x_1\in I_j$, $j \in [N]$, 
\begin{equation*}
\begin{aligned}
   \left|\partial_{x_i}  T_i^{\star}(x_i ;x_1 ) - \partial_{x_i}  T_i^{\star}(x_i; b_j) \right|   & \lesssim  \frac{L_V \overline{L}}{\ell_V^2\sqrt{\ell_V'}}\, (1 + |x_i|)\,|x_1 - b_j| \leq \frac{L_V \overline{L}}{\ell_V^2\sqrt{\ell_V'}}\, (1 + |x_i|)\,\delta\,.
\end{aligned}
\end{equation*}
The above bounds holds when we replace $b_j$ with $b_{j+1}$. 
Therefore, Fubini's theorem gives
\begin{equation*}\label{eq: Ai1 1}
\begin{aligned}
A_{ij}^\star & \lesssim \Bigl(\frac{L_V \overline{L}}{\ell_V^2\sqrt{\ell_V'}}\Bigr)^2  \int_{\{x_1 \in I_j\}}   (1 + |x_i| )^2\, \delta^2\,\rho(\dd x)\\
& = \frac{L_V^2 \overline{L}^2}{\ell_V^4 \ell_V'}\, \int  (1 + |x_i| )^2\, \rho_i(\dd x_i)\,  \delta^2 \rho_1(I_j)
\lesssim \frac{L_V^2 \overline{L}^2}{\ell_V^4 \ell_V'}\, \delta^2 \rho_1(I_j)\,.
\end{aligned}
\end{equation*}
On the other hand, \cref{thm:star_caff}~(i) and the Mills ratio bound imply that for $j\in \{0,N+1\}$,
\begin{equation*}
\begin{aligned}
         &\int_{\{x_1 \in I_j\}} \bigl|\partial_{x_i}  T_i^{\star}(x_i ;x_1 ) - \partial_{x_i}  T_i^{\star}(x_i; b_j)\bigr|^2\,\rho(\dd x)
         \leq \int_{\{x_1 \in I_j\}} \bigl(\frac{1}{\sqrt{\ell_V}}\bigr)^2\,\rho(\dd x)\\
         &\qquad \lesssim  \frac{\rho_1 \left([R, \infty)\right)}{\ell_V} \lesssim \frac{\exp \left(- R^2/2 \right)}{\ell_V R}\,.
\end{aligned}
\end{equation*}
Combining the previous two bounds yields that
\begin{equation}\label{eq:Ai1}
    \begin{aligned}
\sum_j A_{ij}^\star &\lesssim    \frac{L_V^2 \overline{L}^2}{\ell_V^4 \ell_V'}\, \delta^2 \sum_{j=1}^N \rho_1(I_j) + \frac{\exp \left( - R^2/2 \right)}{\ell_V R}
\lesssim \frac{L_V^2 \overline{L}^2}{\ell_V^4 \ell_V'}\, \delta^2 + \frac{\exp \left( - R^2/2 \right)}{\ell_V R}\,. 
\end{aligned}
\end{equation}

For $\hat A_{i}$, observe that
\begin{equation*}
\begin{aligned}
\hat A_{i}
&\lesssim \int  \max_{j}{
 \bigl|\partial_{x_i}  T_i^{\star}(x_i ; b_j) - \partial_{x_i} \widehat T_i(x_i ; b_j) \bigr|^2}\,  \rho_i(\dd x_i)\\
&=\underbrace{ \int_{-R}^R  \max_{j}{ 
 \bigl|\partial_{x_i}  T_i^{\star}(x_i ; b_j) - \partial_{x_i} \widehat T_i(x_i ; b_j) \bigr|^2}  \rho_i(\dd x_i)}_{\eqqcolon \hat A_{i1}} \\
 &\qquad + \underbrace{\int_{\R \backslash [-R, R]}  \max_{j}{\bigl|\partial_{x_i}  T_i^{\star}(x_i ; b_j) - \partial_{x_i}  \widehat T_i(x_i ; b_j)\bigr|^2}\, \rho_i(\dd x_i)}_{\eqqcolon \hat A_{i2}}\,.
\end{aligned}
\end{equation*}
From \Cref{lemma:hat derivative}, $|\partial_{x_i}\widehat T_i(\cdot\mid z_1)|\le 1/\sqrt{\ell_V}$. Using the Mills ratio and by \cref{thm:star_caff}~(i), we have
\begin{equation}\label{eq: Ai22}
\hat A_{i2}  \lesssim  \frac{\rho \left([R, \infty)\right)}{\ell_V} \lesssim \frac{\exp \left(- R^2/2 \right)}{\ell_V R}\,. 
\end{equation}
As \cref{thm:star_caff}~(i) shows $|\partial_{x_i}^2 T_i^\star(x_i; x_1)| \lesssim \frac{L_V}{\ell_V^{3/2}}\,(1+|x_i|)$ for $x_i\in I_k$, $k\in [N]$, we have
\begin{align*}
\big|\partial_{x_i}T_i^{\star}(x_i\mid b_j) - \partial_{x_i}\widehat T_i(x_i\mid b_j)\big|
\le \sup_{\tilde x_i\in I_k}\big|\partial_{x_i}^2 T_i^\star(\tilde x_i;b_j)\big|\;\delta
\lesssim \frac{L_V}{\ell_V^{3/2}}\,(1 + |x_i|+\delta)\,\delta\,.
\end{align*}
Hence,
\begin{equation}\label{eq: Ai21}
\begin{aligned}
\hat A_{i1} &=  \int_{-R}^R \max_{j}{  \bigl|\partial_{x_i}  T_i^{\star}(x_i ; b_j) - \partial_{x_i} \widehat T_i^1(x_i ; b_j) \bigr|^2}\,  \rho_i(\dd x_i)\\
&\leq \frac{L_V^2 }{\ell_V^3}\, \delta^2 \int_{-R}^R (1 + |x_i| + \delta)^2\,\rho_i(\dd x_i)\lesssim \frac{L_V^2 }{\ell_V^3}\, \delta^2\,. 
\end{aligned}
\end{equation}
Combining~\eqref{eq:Ai1},~\eqref{eq: Ai22}, and~\eqref{eq: Ai21}, we see that
\begin{equation}\label{eq:Ai}
A_i \lesssim \hat A_{i1} + \hat A_{i2} \lesssim \Bigl( \frac{L_V^2 \overline{L}^2}{\ell_V^4 \ell_V'} + \frac{L_V^2 }{\ell_V^3} \Bigr)\,\delta^2 + \frac{\exp \left( - R^2/2 \right)}{\ell_VR}\,. 
\end{equation}

\textbf{Bounding $B_i$.} 
 We now bound $B_i$:
\begin{align*}
B_i &= \int \bigl|\partial_{x_1}T_i^{\star}(x_i; x_1) - \partial_{x_1}\widehat T_i(x_i; x_1)\bigr|^2 \,\rho(\dd x)\\ 
&= \underbrace{\sum_{j\in\{0,N+1\}} \sum_k \int_{S^{(i)}_{j,k}} \bigl|\partial_{x_1}T_i^{\star}(x_i;x_1) - \partial_{x_1}\widehat T_i(x_i;x_1)\bigr|^2 \,\rho(\dd x)}_{\eqqcolon B_{i1}}\\
&\qquad + \underbrace{\sum_{j \in [N]} \sum_k \int_{S_{j,k}^{(i)}} \bigl|\partial_{x_1}T_i^{\star}(x_i;x_1) - \partial_{x_1}\widehat T_i(x_i;x_1)\bigr|^2 \,\rho(\dd x)}_{\eqqcolon B_{i2}}\,.
\end{align*}
For $j\in\{0,N+1\}$, recall that 
$|\partial_{x_1} T_i^{\star}(x_i; x_1)| \leq \frac{\overline{L}}{\ell_V \sqrt{\ell_V'}}$.  
In addition, \Cref{lemma:hat derivative} shows $|\partial_{x_1} \widehat{T}_i(x_i; x_1)| \leq \frac{\overline{L}}{\ell_V \sqrt{\ell_V'}}$. 
By the Mills ratio bound,
\begin{equation}\label{eq:Bi1}
    \begin{aligned}
           B_{i1}\lesssim \frac{\overline{L}^2}{\ell_V^2\ell_V'} \int_{\R\setminus [-R,R]} \rho_1(\dd x_1) \lesssim \frac{\overline{L}^2}{\ell_V^2\ell_V' R} \exp(-R^2/2)\,.
    \end{aligned}
\end{equation}
Fix $j \in [N]$ and $k \in \{0, \ldots, N+1\}$.  
By the mean value theorem, there exists $\zeta_j \in I_j$ such that 
$\partial_{x_1} T_i^{\star}(x_i;\zeta_j) = \delta^{-1}\,\bigl(T_i^\star(x_i; b_{j+1}) - T_i^\star(x_i; b_j)\bigr)$.  
Applying the mean value theorem again together with~\eqref{eqn:hat derivative} with $\eta_{k, j} = T_i^\star(b_k; b_{j+1}) -T_i^\star(b_k; b_j) $, $\Delta_{j, k} =  \eta_{k+1,j} -  \eta_{k,j}$, we see that for some $\zeta_j\in I_j$, 
\begin{equation*}
\begin{aligned}
& \bigl|\partial_{x_1}T_i^{\star}(x_i;x_1) - \partial_{x_1} \widehat T_i(x_i;x_1)\bigr| \\
&\qquad \leq \Bigl|\partial_{x_1} T_i^{\star}(x_i;x_1) - \partial_{x_1} T_i^{\star}(b_k;\zeta_j) - (x_i - b_k)\,\frac{\Delta_{j,k}}{\delta^2}\Bigr|  \\
&\qquad = \Bigl|\partial_{x_1} T_i^{\star}(x_i;x_1) - \partial_{x_1} T_i^{\star}(b_k;x_1)
+ \partial_{x_1} T_i^{\star}(b_k;x_1) - \partial_{x_1} T_i^{\star}(b_k;\zeta_j)
- (x_i - b_k)\,\frac{\Delta_{j,k}}{\delta^2}\Bigr|  \\
&\qquad = \Bigl|\bigl(\partial_{x_1}\partial_{x_i}T_i^{\star}(\beta_k;x_1) - \tfrac{\Delta_{j,k}}{\delta^2}\bigr)\,(x_i - b_k)
+ \partial_{x_1} T_i^{\star}(b_k;x_1) - \partial_{x_1} T_i^{\star}(b_k;\zeta_j)\Bigr|
\end{aligned}
\end{equation*}
for some $\beta_k$ between $x_i$ and $b_k$.
By \cref{thm:star_caff}~(ii), 
\begin{multline*}
    |\partial_{x_1}\partial_{x_i}T_i^{\star}(\zeta';x_1)|
\le |\partial_{2}\partial_{x_i}T_i^{\star}(\zeta'\mid T_1^\star(x_1))|\,
|\partial_{x_1}T_1^\star(x_1)|\\
\lesssim
\frac{L_V \overline{L}}{\ell_V^2\sqrt{\ell_V'}}\,(1 + |\zeta'|)
\le \frac{L_V \overline{L}}{\ell_V^2\sqrt{\ell_V'}}\,(1 + |x_1| + \delta)\,.
\end{multline*}
In addition, \cref{lemma:x1-hessian} implies
$$|\partial_{x_1} T_i^{\star}(b_k;x_1)
- \partial_{x_1} T_i^{\star}(b_k;\zeta_j)|
\lesssim
\frac{\LGR}{\ell_V'}\,(1 + |x_1| + \delta)^{\gamma}\,\delta
+ \frac{L_V'\,\overline{L}}{(\ell_V')^{3/2}\,\ell_V}\,(1 + |x_1| + \delta)\,\delta\,.$$  
Combining the above bounds and applying the triangle inequality together with 
\Cref{lemma:hat derivative}, we have
\begin{align*}
    |\partial_{x_1}T_i^{\star}(x_i;x_1)
- \partial_{x_1}\widehat T_i(x_i;x_1)|
&\lesssim 
\Bigl[
\frac{L_V \overline{L}}{\ell_V^2\sqrt{\ell_V'}}\,(1 + |x_i| + \delta)\\
&\qquad  + \frac{L_V'\,\overline{L}}{(\ell_V')^{3/2}\,\ell_V}\,(1 + |x_1| + \delta)
+ \frac{\LGR}{\ell_V'}\,(1 + |x_1| + \delta)^{\gamma}
\Bigr]\,\delta\,.
\end{align*}
Therefore,
\begin{equation}\label{eq:Bi2}
\begin{aligned}
B_{i2}
&= \sum_{j,k=1}^N \int_{S_{j,k}^{(i)}}
   |\partial_{x_1}T_i^{\star}(x_i;x_1)
       - \partial_{x_1}\widehat T_i(x_i;x_1)|^2
   \,\rho(\dd x)  \\
&\lesssim 
\frac{L_V^2 \overline{L}^2}{\ell_V^4 \ell_V'}
  \sum_{j,k=1}^N \int_{S_{j,k}^{(i)}} (1 + |x_i| + \delta)^2\, \delta^2\,\rho(\dd x)\\
&\qquad + \frac{(L_V')^2 \overline{L}^2}{(\ell_V')^3 \ell_V^2}
  \sum_{j,k=1}^N \int_{S_{j,k}^{(i)}} (1 + |x_1| + \delta)^2\, \delta^2\,\rho(\dd x)  \\
&\qquad + \frac{\LGR^2}{(\ell_V')^2}
  \sum_{j,k=1}^N \int_{S_{j,k}^{(i)}} (1 + |x_1| + \delta)^{2\gamma}\,\delta^2\,\rho(\dd x)  \\
&\lesssim 
\delta^2\,
\Bigl[
\frac{L_V^2 \overline{L}^2}{\ell_V^4 \ell_V'}
+ \frac{(L_V')^2 \overline{L}^2}{(\ell_V')^3 \ell_V^2}
+ \frac{\LGR^2\,(1+M_{2\gamma}(\rho_1))}{(\ell_V')^2}
\Bigr]\,.
\end{aligned}
\end{equation}

Hence, combining \eqref{eq:Bi1} and \eqref{eq:Bi2} yields
\begin{equation}\label{eq:Bi}
\begin{aligned}
B_i
\lesssim \frac{\overline{L}^2}{\ell_V^2\ell_V' R} \exp(-R^2/2) + \delta^2\, \Bigl[\frac{L_V^2 \overline{L}^2}{\ell_V^4 \ell_V'} + \frac{(L_V')^2 \overline{L}^2}{(\ell_V')^3 \ell_V^2} + \frac{(1+M_{2\gamma}(\rho_1))\,\LGR^2}{(\ell_V')^2} \Bigr]\,. 
\end{aligned}
\end{equation}
Finally, using \cref{lemma:T_1 error bound} together with \eqref{eq:Ai} and \eqref{eq:Bi},
\begin{align*}
 \|D(\widehat T -T^\star)\|^2_{L^2(\rho)} 
&\lesssim  \Bigl(\frac{1}{\ell_V'} + \frac{d-1}{\ell_V} + \frac{d\overline L^2}{\ell_V^2 \ell_V'}\Bigr)\,\frac{\exp(-R^2/2)}{R} 
+ \frac{(L_V')^{2}}{(\ell'_V)^3}\,\delta^2  \\
&\qquad + \Bigl[\frac{L_V^2 }{\ell_V^3} + \frac{L_V^2 \overline{L}^2}{\ell_V^4 \ell_V'} + \frac{(L_V')^2 \overline{L}^2}{(\ell_V')^3 \ell_V^2} + \frac{(1+M_{2\gamma}(\rho_1))\,\LGR^2}{(\ell_V')^2} \Bigr]\,(d -1)\,\delta^2 \\
&\lesssim \Bigl(\frac{1}{\ell_V'}  + \frac{d}{\ell_V} + \frac{d\overline L^2}{\ell_V^2 \ell_V'}\Bigr)\,\frac{\exp(-R^2/2)}{R} \\
&\qquad + \Bigl[\Bigl(\frac{1}{d} + \frac{\overline{L}^2}{\ell_V^2}\Bigr)\frac{(L_V')^{2}}{(\ell'_V)^3}+\frac{L_V^2 }{\ell_V^3}  + \frac{L_V^2 \overline{L}^2}{\ell_V^4 \ell_V'} +  + \frac{(1+M_{2\gamma}(\rho_1))\,\LGR^2}{(\ell_V')^2} \Bigr]\,d\delta^2\,.
\end{align*}
\end{proof}

\subsection{The computational guarantees}\label{sec:computational guarantee}
In previous section, we explicitly construct a map $\widehat T \in \underline{\rm cone}(\cM; \balpha\, \mathrm{id})$ that approximate the optimal star-separable map $T^\star$ and its gradient arbitrarily well in $L^2(\rho)$. However, this construction requires knowledge of $T^\star$. In this section, we show that a good approximation to $T^\star$ can by computed by solving a convex optimization problem.

Given the dictionary $\cM = \cM_0\cup \cM_1 \cup \cM_2\cup \cdots \cup \cM_5$ defined in \eqref{def: piecewise linear classes}, we search for the minimizer $T^\star_\cM$ of the KL divergence with respect to $\pi$ over $\underline{\rm cone}(\cM;\, \balpha\, \mathrm{id})$, i.e.,
\begin{equation}\label{eq:SSVIPM}\tag{SSVI-opt}
    T^\star_\cM = \argmin_{T \in \underline{\rm cone}(\cM;\, \balpha\, \mathrm{id})} \kl{T_\# \rho}{\pi}.
\end{equation}
The above optimization problem admits a unique minimizer. 
The computational guarantees in \cref{thm: optimality gap} are shown without invoking the explicit form of the maps in $\underline{\rm cone}(\cM;\, \balpha\, \mathrm{id})$, as we choose to work with \eqref{eq:SSVIPM}.

 The proof of \cref{thm: optimality gap} involves two steps. First, by \cref{coro:approximation bound}, we choose an ``oracle'' approximator $\widehat{T} \in \underline{\rm cone} \left(\cM;\, \balpha\, \mathrm{id} \right)$ that approximates $T^\star$. Second, we show that $T^\star_{\cM}$ is close to the approximating map $\widehat{T}$ in $L^2(\rho)$ using the geometry of the problem~\eqref{eq:SSVIPM}. The result then follows by combining the bounds from the two steps.

\begin{theorem}\label{thm: optimality gap}
Let \eqref{assum:R}, \eqref{assum:P1}, \eqref{assum:P2}, \eqref{assum:RD+}, and \eqref{assum:GR} be satisfied. Then, for each $\epsilon > 0$, there exists a dictionary $\cM$ of size $O((d^2/\epsilon^2) \log(d/\epsilon^2))$ such that
\begin{equation*}
 \|T^\star_\cM - T^\star\|_{L^2(\rho)} \lesssim \epsilon,
\end{equation*}
where $T^\star_\cM$ is the minimizer to~\eqref{eq:SSVIPM} and the constants depend polynomially on $\ell_V$, $L_V$, $\ell_V'$, $L_V'$, $\overline{L}$, $\LGR$, $M_{2\gamma}(\rho_1)$ and their inverses.
\end{theorem}

 In the sequel, we will work with the following measures: $\pi^\star_\cM  \deq  (T^\star_\cM)_\# \rho$, $\hat \pi \deq  (\widehat{T})_\# \rho$.  Define the finite-dimensional set $$\cP_{\cM}  \deq  \left\{T_\#\rho: T \in \underline{\rm cone} \left(\cM;\, \balpha\, \mathrm{id} \right)\right\}\subset \cT_{\rm star}\,.$$ 
The following result shows that the objective value of $\pi^\star_\cM$ is close to that of $\hat{\pi}$, which is used to control the distance between $T^\star_\cM$ and $\widehat{T}$ in \cref{thm: optimality gap}.

\begin{lemma}\label{lemma: objective gap}
Under the setting of \cref{thm: optimality gap}, there exists $\kappa > 0$ depending polynomially on $\ell_V, \ell_V', L_V, L_V'$ such that
\begin{align*}
0
&\leq \kl{\hat \pi}{\pi} - \kl{\pi^\star_\cM}{\pi} \\
&\leq \frac{1}{2}\, \bigl((L_V \lor L_V')\, \|\widehat T - T^\star\|^2_{L^2(\rho)}  + \kappa^2\, \| D(\widehat T -T^\star)\|^2_{L^2(\rho)}  \bigr)\,. 
\end{align*}
\end{lemma}
\begin{proof}
The first inequality is straightforward since $\hat \pi \in \cP_\cM$. To show the second inequality, we use the fact that $\cP_{\cM} \subset \cC_{\rm star}$ to obtain
\[
 \kl{\hat \pi}{\pi} - \kl{\pi^\star_\cM}{\pi} \leq  \kl{\hat \pi}{\pi} - \kl{\pi^\star}{\pi}\,.
\]
To conclude, it suffices to bound the right-hand side above. 

Recall that $T^\star$ (resp.\ $\widehat T$) denotes the star-separable map that pushes forward $\rho$ to $\pi^\star$ (resp.\ $\hat{\pi}$). By definition, $T^\star,\widehat T\in \cT_{\rm star}$.
Define $T_t \deq  (1-t)T^\star + t\widehat{T}$ and $\mu_t \deq  {T_t}_\# \rho$, then $\mu_0 = \pi^\star$ and $\mu_1 = \hat \pi$. We note that $\mu_t \in \cC_{\rm star}$ as $\cT_{\rm star}$ is convex. Recall~\eqref{eq:potential_time_deriv} and~\eqref{entropy time derivative}, with $T_0 = T^\star$ and $T_1 = \widehat T$.

By \cref{lemma: log-concavity-pi-schur}, $\nabla^2V(T_t) \preceq [L_V \lor (L_V'/2)]\, I$, and thus the second derivative of the potential energy is bounded by
\begin{equation}\label{eq: Vt}
     \frac{\dd^2\cV(\mu_t)}{\dd t^2} \leq \left(L_V \lor (L_V'/2)\right) \|\widehat T - T^\star\|^2_{L^2(\rho)}\,.  
\end{equation}

By \cref{thm:star_caff} and \cref{lemma:hat derivative}, we apply \cref{lemma:L_t-spectral} with $L_1  = DT^\star$, $L_2 = D\widehat T$, $c_1 =  \sqrt{1/L_V} \wedge \sqrt{1/L_V'} = \sqrt{1/(L_V\vee L_V')}$, and $c_3 =\frac{\overline{L}}{\ell_V\sqrt{\ell_V'}}$ to derive that
\[
\|(D T_t)^{-1}\|_2 \leq \frac{c_1 + \sqrt{d-1}\,c_3}{c_1^2}\leq  \frac{1}{c_1} +  \frac{L_V'}{2\sqrt{\ell_V'}\,c_1^2} \eqqcolon \kappa\,,
\]
where the last inequality follows from $\frac{L_V'}{2\sqrt{\ell_V'}} \geq \frac{(d-1)\,\overline{L}}{\ell_V\sqrt{\ell_V'}}= (d-1)\,c_3$ by \eqref{assum:RD+} and $d\geq 2$. Therefore, by~\eqref{entropy time derivative},  
\begin{equation}\label{eq: Ht}
    \frac{\dd^2\cH(\mu_t)}{\dd t^2}  \leq \kappa^2\, \| D(\widehat T -T^\star)\|^2_{L^2(\rho)}\,.
\end{equation}

Thus, combining~\eqref{eq: Vt} and~\eqref{eq: Ht} yields
\begin{equation}\label{eq: second bound}
     \frac{\dd^2\kl{\mu_t}{\pi}}{\dd t^2} \leq \left(L_V \lor (L_V'/2)\right)\|\widehat T - T^\star\|^2_{L^2(\rho)}  + \kappa^2\, \| D(\widehat T -T^\star)\|^2_{L^2(\rho)}\,.
\end{equation}
Moreover, as
\begin{align*}
\kl{\hat \pi}{\pi} - \kl{\pi^\star}{\pi}
&= \inner{[\nabla_W \kl{\cdot}{\pi}](\pi^\star),\, T-\id}_{L^2(\pi^\star)} \\
&\qquad{} + \int_0^1\int_0^t \frac{\dd^2}{\dd s^2}\,\kl{\mu_s}{\pi}\dd s\,\dd t\,,
\end{align*}
the optimality of $\pi^\star$ and the derivative bound~\eqref{eq: second bound} imply that
\[
\kl{\hat \pi}{\pi} - \kl{\pi^\star}{\pi} \leq  \frac{1}{2}\, \bigl((L_V \lor L_V')\, \|\widehat T - T^\star\|^2_{L^2(\rho)}  + \kappa^2\, \| D(\widehat T -T^\star)\|^2_{L^2(\rho)}  \bigr) \,.
\]
\end{proof}
\begin{proof}[Proof of \cref{thm: optimality gap}]
Let $\widehat{T}$ be the approximation of $T^\star$ in $\underline{\rm cone}(\cM;\, \balpha\, \mathrm{id})$, constructed in \cref{coro:approximation bound} with $|\cM| = O(\dim(\Theta)) =  O\left((d^2/\epsilon^2) \log(d/\epsilon^2) \right)$. Then $T^\star_\cM, \widehat T \in  \cT_{\rm star}$. By the triangle inequality,
\begin{equation*}
    \|T^\star_\cM - T^\star\|_{L^2(\rho)} 
    \leq \|T^\star_\cM - \widehat T\|_{L^2(\rho)} + \|\widehat T - T^\star\|_{L^2(\rho)}\,.
\end{equation*}
The second term is controlled by \cref{coro:approximation bound}, recalling that
\begin{equation}\label{eq: approximation L2}
\|\widehat T - T^\star\|_{L^2(\rho)} \leq \vae\,.
\end{equation}
It remains to bound the first term. By \cref{lemma: log-concavity-pi-schur}, $\pi$ is $(\ell_V' \wedge \ell_V)$-log-concave. Then, by \cref{thm:convexity}, the map $T \mapsto \mathrm{KL}(T_\# \rho \mmid \pi)$ is $(\ell_V' \wedge \ell_V)$-convex over the convex set $\underline{\rm cone}(\cM;\, \balpha\, \mathrm{id})\subset \cT_{\rm star}$. Together with the optimality of $\pi^\star_{\cM}$, this yields
\begin{equation}\label{eq:approx M hat}
    \frac{\ell_V' \wedge \ell_V}{2}\, \|T^\star_{\cM} - \widehat T\|_{L^2(\rho)}^2 
    \leq \KL(\widehat \pi \mmid \pi) - \KL(\pi^\star_{\cM} \mmid \pi)\,.
\end{equation}
In addition, by \cref{lemma: objective gap}, there exists a constant $\kappa>0$ depending polynomially on  $\ell_V, \ell_V', L_V, L_V'$ such that 
\begin{equation*}
     \kl{\hat{\pi}}{\pi} - \kl{\pi^\star}{\pi}\leq \frac{1}{2}\,\bigl((L_V'\vee L_V)\,\|\widehat T - T^\star\|^2_{L^2(\rho)}  + \kappa^2\, \| D(\widehat T -T^\star)\|^2_{L^2(\rho)}\bigr)\,.
\end{equation*}
Together with~\eqref{eq:approx M hat}, this gives
\begin{equation}
    \|T^\star_\cM - \widehat T\|_{L^2(\rho)}^2 \leq \frac{L_V'\vee L_V}{\ell_V'\wedge \ell_V}\, \|\widehat T - T^\star\|^2_{L^2(\rho)}  + \frac{\kappa^2}{\ell_V'\wedge \ell_V}\, \| D(\widehat T -T^\star)\|^2_{L^2(\rho)}\,.
\end{equation}
The first term is already controlled in~\eqref{eq: approximation L2}. On the other hand, by \cref{coro:approximation bound},  
\begin{equation*}
\|D(\widehat T - T^\star) \|_{L^2(\rho)}     \leq \vae\,.
\end{equation*}
The above bound and~\eqref{eq: approximation L2} imply that $\|T^\star_\cM - T^\star \|_{L^2(\rho)}  \lesssim \epsilon$.
\end{proof}

We conclude this section with a technical lemma used in establishing \cref{lemma: objective gap}.
\begin{lemma}\label{lemma:L_t-spectral}
     For $i=0,1$, let $L_i \deq D_i + u_ie_1^\top$, where $D_i$ is a diagonal matrix and $u_i = (0,u_{i2},\ldots, u_{id})^\top \in \R^d$. For $t \in [0,1]$, define $L_t \deq (1-t)L_0 + t L_1$. Suppose further that $c_1 I \preceq D_0, D_1 \preceq c_2I$ and $\max\{\|u_0\|_\infty, \|u_1\|_{\infty}\}\leq c_3$ for some constants $c_1,c_2, c_3 >0$. Then, 
\[
 \|L_t\|_2\leq c_2 + \sqrt{d-1}\,c_3\qquad {\rm and}\qquad \|L_t^{-1}\|_2 \leq \frac{c_1 + \sqrt{d-1}\,c_3}{c_1^2}\,.
\]
\end{lemma}
\begin{proof}
Denote $D_t \deq (1-t)D_0 + tD_1$ and $u_t \deq (1-t)u_0 + tu_1$. Then $L_t = D_t + u_t e_1^\top$ is invertible. Therefore,
\[
\|L_t\|_2 \leq \|D_t\|_2 + \|u_te_1^\top\|_2 \leq c_2 + \sqrt{d-1}\,c_3\,.
\]
On the other hand, by the Sherman--Morrison formula \citep{Sherman1950},
\[
L_t^{-1} = D_t^{-1} - \frac{D_t^{-1} u_t e_1^\top D_t^{-1}}{1 + e_1^\top D_t^{-1} u_t}
= D_t^{-1} - D_t^{-1} u_t e_1^\top D_t^{-1}\,,
\]
since $e_1^\top D_t^{-1} u_t = 0$ (the first coordinate of $u_t$ is $0$). Hence,
\[
\|L_t^{-1}\|_2 \le \|D_t^{-1}\|_2 + \|D_t^{-1} u_t\|\,\|e_1^\top D_t^{-1}\|
\le \frac{1}{c_1} + \frac{\sqrt{d-1}\,c_3}{c_1^2}\,,
\]
because $D_t \succeq c_1 I$, $\|u_t\| \le \sqrt{d-1}\,c_3$, and $\|e_1^\top D_t^{-1}\|=\|D_t^{-1}e_1\|\le 1/c_1$. 
\end{proof}

\subsection{Algorithm details}\label{subsec:algorithm}
This subsection provides the details for computing the minimizer $T^\star_\cM$ of the problem~\eqref{eq:SSVIPM} using the projected gradient descent algorithm. Recall that
\begin{equation}
   T^\star_\cM = \argmin_{T \in \underline{\rm cone}(\cM;\, \balpha\, \mathrm{id})} \kl{T_\# \rho}{\pi}.
\end{equation}
where $\underline{\rm cone}(\cM; \balpha\, \mathrm{id})$ denotes the augmented spiked cone defined in~\eqref{eq:augmented-spiked-cone} with $\alpha_1 = (L'_V)^{-1/2}$ and $\alpha_2 = \cdots = \alpha_d = L_V^{-1/2}$.

To resolve this, we show that~\eqref{eq:SSVIPM} can be instead reformulated as a \emph{finite-dimensional}, \emph{strongly convex} optimization problem. To this end, we introduce the following set: let $\cM_5'$ be the set of functions in $\cM_5$ with positive signs, that is
\begin{align}\label{eq:tilde M}
\begin{split}
    \cM_5' &\deq \Bigl\{x \mapsto \psi\Bigl(\frac{x_1- b}{\delta}\Bigr)\, e_i  \,\,\Bigl| \,\,i \geq 2,\, b\in \cB \Bigr\}\,.
\end{split}
\end{align}
Define $\cM' \deq   \cM_0\cup\cM_1\cup \cdots \cup \cM_4 \cup \cM_5'$.  We require that $\int T(x)\, \rho(\dd x) = 0$ for all $T \in \cM'$. We can achieve this by centering each function and absorbing the constant shift in the vector $v$. Moreover, by construction, all functions in $\cM'$ are linearly independent in the pointwise sense: no element of $\cM'$ is a linear combination of the others almost everywhere with respect to $\rho$.

Let $\Theta  \deq  \bigl(\R_+^{|\cup_{i = 0}^4 \cM_i|} \times \R^{|\cM_5'|} \bigr)\times \R^d$.  
Then, \eqref{eq:SSVIPM} is equivalent to the following finite-dimensional, $(\ell_V\wedge \ell_V')$-strongly convex optimization problem:
\begin{equation} \label{eq:modified-projected-GD}
(\hat \lambda, \hat v)=    \argmin_{(\lambda,v) \in \Theta} \KL\biggl(\Bigl(\balpha \, \id
  + \sum_{T \in \cup_{i = 0}^4 \cM_i} \lambda_T T
  + \sum_{\tilde T \in \cM_5'} \lambda_{\tilde T} \tilde T
  + v \Bigr)_{\#} \rho\biggm\Vert\pi\biggr)\,,
\end{equation}
where $\Theta$ is equipped with the norm $ \|(\lambda, v)\|_{\Theta}  \deq  \| \sum_{T \in \cM'} \lambda_T T + v\|_{L^2(\rho)}$. To see the equivalence, we first note that by direct calculation
\begin{equation*}
    \left\|(\lambda, v) \right\|_{\Theta}^2 = \Bigl\|\sum_{T \in \cM'} \lambda_T T + v \Bigr\|_{L^2(\rho)}^2 = \langle \lambda, Q \lambda \rangle + \langle v, v \rangle = \|\lambda\|_Q^2 + \|v\|^2\,, 
\end{equation*}
where $Q$ is a positive semi-definite Gram matrix with entries $Q_{T, \tilde{T}}  \deq  \langle T, \tilde{T} \rangle_{L_2(\rho)}$ for $T, \tilde T \in \cM'$. In fact, $Q$ is positive definite since functions in $\cM'$ are linearly independent almost everywhere under $\rho$. As a result, $\|(\lambda, v)\|_{\Theta}  = 0$ implies $(\lambda, v) = (0,0)$, i.e., $\|\cdot\|_\Theta$ is indeed a norm. 

Define $T_{\lambda ,v}  \deq  \balpha \id+ \sum_{T \in \cM'} \lambda_T T + v$. The set $\underline{\rm cone}\left(\cM;\,  \balpha \id \right)_{\#} \rho$ endowed with the linearized optimal transport distance admits a natural isometric embedding into the space $\Theta$ equipped with norm $\|\cdot \|_\Theta$, in the sense that
\begin{align*}
        {\rm d}^2_{\rm LOT, \rho} \left(\mu_{\eta, u}, \mu_{\lambda, v}\right) &=\Bigl\| \sum_{T \in \cM'} (\eta_T - \lambda_T) T + u - v\Bigr\|_{L^2(\rho)}^2 \\
         &= \left\| \eta - \lambda \right\|_Q^2 + \|u -v\|_2^2 =  \left\|(\eta, u) - (\lambda, v) \right\|_{\Theta}^2\,.
\end{align*}
This equivalence of norms, together with the discussion in \Cref{sec:lift_maps}, establishes the claimed equivalence between the lifted and finite-dimensional formulations.

We can now employ tools from finite-dimensional convex optimization to solve \eqref{eq:SSVIPM} with convergence guarantees. In particular, we compute the minimizer using projected gradient descent (PGD) as described in \Cref{algo:PGD};  see, for example, Section~10.2 of \citet{Beck2017} for a detailed discussion of its convergence properties.
\begin{algorithm}[htp]
\caption{Projected gradient descent over $\underline{\rm cone}\left(\cM;\,  \balpha \id \right)$}
\label{algo:PGD}
\begin{algorithmic}
\State \textbf{Input:} Piecewise linear dictionary $\cM'$, 
and projection cone 
$\cK  \deq  \R_+^{|\cup_{i=0}^4 \cM_i|} \times \R^{|\cM_5'|}$. Constants $\alpha  \deq  \ell_V \land \ell_V'$, $L  \deq  L_V \lor (L_V'/2)$, $\Upsilon  \deq  9 \delta^{-2}\, \bigl((L_V')^{1/2} + (d - 1)\, L_V^{1/2}\bigr)^2\, \|Q^{-1}\|_2$. 
\State \textbf{Output:} $(\hat \lambda, \hat v )$. 
\State \textbf{Initialization:} Initialize $ (\lambda^{(0)}, v^{(0)}) \in \Theta$
\For{$t = 0, 1,2,3, \dotsc$}
\State $\lambda^{(t + 1)} = \text{proj}_{\cK, Q} \bigl(\lambda^{(t)}- \frac{1}{L + \Upsilon}\, Q^{-1}  \nabla_\lambda \KL(\mu_{\lambda^{(t)}, v^{(t)}}\mmid \pi) \bigr)$
\State $v^{(t + 1)} = v^{(t)} - \frac{1}{L + \Upsilon}\, \nabla_v \KL(\mu_{\lambda^{(t)}, v^{(t)}}\mmid\pi)$
\EndFor
\end{algorithmic}
\end{algorithm}
\begin{theorem}[Convergence results for \cref{algo:PGD}]
\label{thm:PGD}
Let Assumptions \eqref{assum:R}, \eqref{assum:P1}, \eqref{assum:P2}, and \eqref{assum:RD+} be satisfied. Let $\Upsilon  \deq  9 \delta^{-2}\, \bigl((L_V')^{1/2} + (d - 1)\, L_V^{1/2}\bigr)^2\, \|Q^{-1}\|_2$, and let $\{(\lambda^{(t)}, v^{(t)})\}_{t \in \mathbb{N}}$ denote the iterates generated by \cref{algo:PGD}. Then for $\kappa  \deq  \frac{L_V \lor (L_V'/2) + \Upsilon}{\ell_V \land \ell_V'}$, 
\begin{enumerate}[label = (\roman*)]
\item ${\rm d}^2_{\rm LOT, \rho} (\mu_{\lambda(t),v(t)}, \mu_{\hat \lambda, \hat v}) \leq \left(1 - \kappa^{-1}\right)^t {\rm d}^2_{\rm LOT, \rho} (\mu_{\lambda(0), v(0)} , \mu_{\hat \lambda, \hat v})$. 
\item $\KL(\mu_{\lambda(t),v(t)}\mmid \pi) - \KL(\mu_{\hat \lambda, \hat v}\mmid \pi)  \leq (1 - \kappa^{-1})^t\, \frac{(L + \Upsilon)}{2}\, {\rm d}^2_{\rm LOT, \rho} (\mu_{\lambda(0), v(0)} , \mu_{\hat \lambda, \hat v} )$. 
\end{enumerate}
\end{theorem}
We first note that applying \cref{thm:PGD} with $\theta^{(t)}  \deq  (\lambda^{(t)},v^{(t)})$ shows \cref{thm:GD informal}.
The key step of the proof is to establish the \textit{smoothness} property for the KL divergence over the finite-dimensional space $\underline{\rm cone} (\cM; \balpha \id)$, which is non-trivial given that the entropy functional is non-smooth over the full Wasserstein space \citep[see, e.g.,][Section~6.2.2]{ChewiBook}. 
 
\textbf{Computing the gradient.} To apply the optimization algorithms, we compute the gradient of $\kl{\mu_{\lambda, v}}{\pi} $ w.r.t.\ $(\lambda, v)$ and the projection operator w.r.t.\ $\|\cdot\|_Q$. Under the change of variables formula,
\begin{equation*}
\begin{aligned}
        \kl{\mu_{\lambda, v}}{\pi} =&  -\int \log \det \Bigl({\rm diag}(\balpha)  + \sum_{T \in \cM'} \lambda_T D T(x) \Bigr)\, \rho(\rd x) \\
        &\qquad{} + \int V\Bigl({\rm diag}(\balpha)\, x + \sum_{T \in \cM'} \lambda_T T(x) + v\Bigr)\, \rho(\rd x) + \textsc{const}\,. 
\end{aligned}
\end{equation*}
Since ${\rm diag}(\balpha) + \sum_{T\in \cM'}\lambda_T D T(x)$ is always invertible, the partial gradients are computed as:
\begin{align}\label{eq:kl_grad_v}
    \nabla_v  \kl{\mu_{\lambda, v}}{\pi}  = \int \nabla V\Bigl({\rm diag}(\balpha)\, x + \sum_{T \in \cM'} \lambda_T T(x) + v\Bigr)\,\rho(\rd x)\,, 
\end{align}
and
\begin{align}
 \partial_{\lambda_T} \kl{\mu_{\lambda, v}}{\pi} =& -\int \biggl[\tr \Bigl( \Bigl({\rm diag}(\balpha) + \sum_{T \in \cM'} \lambda_T D T(x) \Bigr)^{-1}\, D T(x) \Bigr) \nonumber\\
 &\qquad{} +  \Bigl\langle T(x),\,  \nabla V\Bigl({\rm diag}(\balpha)\, x + \sum_{T \in \cM'} \lambda_T T(x) + v\Bigr) \Bigr\rangle  \biggr]\, \rho(\rd x)\,. \label{eq:kl_grad_lambda}
\end{align}
To compute the first term in $\nabla_{\lambda_T} \kl{\mu_{\lambda, v}}{\pi}$, note that $D T$ is the following lower triangular matrix with non-negative diagonals: 
\begin{equation*}
[D T(x)]_{ij} =
\begin{cases}
\partial_{x_i} T_i(x_i; x_1)\,, & \text{if } j = i\,, \\
\partial_{x_1} T_i (x_i; x_1)\,, & \text{if } j = 1\,, \\
0\,, & \text{otherwise}\,.
\end{cases}
\end{equation*}
The inverse of ${\rm diag}(\balpha) + \sum_{T \in \cM'} \lambda_T D T(x)$ can be efficiently calculated using forward substitution with $d$ linear equations, each with at most two non-zero terms per row. This gives a total computational cost of $O(d\,|\cM'| )$ to compute all $DT(x)$ terms and $O(d)$ to invert it. 

The terms involving $\nabla V$ can be approximated by taking the averages of $\nabla V$ evaluated at the samples drawn from the Gaussian distribution $\rho$. 

Working toward proving \cref{thm:PGD}, we collect some key facts about the matrix $\sum_{T \in \cM'} \lambda_T D T(x)$, including a bound on its spectral norm. 

\begin{lemma}\label{lemma:cone differential bound}
Fix $x \in \R^d$. Define $M(\lambda) \deq \sum_{T \in \cM'} \lambda_T\, D T(x)$ for $\lambda \in \R_+^{|\cup_{i=0}^4 \cM_i|} \times \R^{|\cM_5'|}$. Then, $M(\lambda)$ is lower triangular with non-negative diagonal entries and if $\|\lambda\| = 1$, then $\|M(\lambda)\|_2^2 \;\le\; 9\,\delta^{-2}$. 
\end{lemma}
\begin{proof}
Recall that $\cM' = \bigcup_{j = 0}^4 \cM_j \cup \cM_5'$. For each $T \in \cM'$, the Jacobian matrix $D T(x)$ admits the following structure:
\begin{itemize}
\item If $T \in \cM_0$, the $(1,1)$ entry of $D T(x)$ is 
$\delta^{-1}\mathbbm{1}_{\{x_1 \in [b, b + \delta)\}}$, and all the others are zero.
\item If $T \in \cM_1$, the $(i,i)$ entry of $D T(x)$ is 
$\delta^{-1}\psi\bigl(\frac{x_1 - b'}{\delta}\bigr)
\mathbbm{1}_{\{x_i \in [b, b + \delta),\, x_1 \in [b', b' + \delta)\}}$, 
and the $(i,1)$ entry is 
$\delta^{-1}\psi\bigl(\frac{x_i - b}{\delta}\bigr)
\mathbbm{1}_{\{x_i \in [b, b + \delta),\, x_1 \in [b', b' + \delta)\}}$. 
All other entries are zero.
\item If $T \in \cM_2$, the $(i,i)$ entry of $D T(x)$ is 
$\delta^{-1}\psi\bigl(1 - \frac{x_1 - b'}{\delta}\bigr)
\mathbbm{1}_{\{x_i \in [b, b + \delta),\, x_1 \in [b', b' + \delta)\}}$, 
and the $(i,1)$ entry is 
$-\delta^{-1}\psi\bigl(\frac{x_i - b}{\delta}\bigr)
\mathbbm{1}_{\{x_i \in [b, b + \delta),\, x_1 \in [b', b' + \delta)\}}$. 
All other entries are zero.
\item If $T \in \cM_3$, the $(i,i)$ entry of $D T(x)$ is 
$\delta^{-1}\mathbbm{1}_{\{x_i \in [b, b + \delta),\, x_1 \ge R\}}$, 
and all the others are zero.
\item If $T \in \cM_4$, the $(i,i)$ entry of $D T(x)$ is 
$\delta^{-1}\mathbbm{1}_{\{x_i \in [b, b + \delta),\, x_1 \le -R\}}$, 
and all the others are zero.
\item If $T \in \cM_5'$, the $(i,1)$ entry of $D T(x)$ is 
$\delta^{-1}\mathbbm{1}_{\{x_1 \in [b, b + \delta)\}}$, 
and all the others are zero.
\end{itemize}

By the structure of the maps in $\cM'$, we conclude that for any $x\in \R^d$,
\begin{equation}\label{eq:DT-ii}
    [D T(x)]_{ii} \ge 0  \,\,\text{for any } T \in \cM',\quad \text{and } \quad \sum_{T \in \cM'} [D T(x)]_{ii} \le \delta^{-1}\,.
\end{equation}
 For any $\lambda  \in \R_+^{|\cup_{j = 0}^4 \cM_j|} \times \R^{|\cM_5'|}$, it follows that $ M\deq \sum_{T \in \cM'} \lambda_T D T(x)$ is a lower triangular matrix with $M_{ii} \geq 0$.

 To control $\|M\|_2$, we follow the argument as in \cref{lemma:L_t-spectral} to write $M = \mathbf{D} + v e_1^\top$ where $\mathbf{D}$ is a diagonal matrix and $v = (0, v_2, \dotsc, v_d)^\top$. Then,
\begin{equation}
\|M\|_2 \leq \|\mathbf{D}\|_2 + \|v\| \,. 
\end{equation}
By \eqref{eq:DT-ii}, we derive that for $\|\lambda\| =1$, 
\begin{equation*}
    0\leq \mathbf{D}_{ii} \leq \sum_{T \in \cM'} [D T(x)]_{ii} \le \delta^{-1}\,.
\end{equation*}
It remains to compute $\|v\|$. For each $i\geq 2$, if $x_i \in [-R, R)$ and $x_1 \in [-R, R)$, there exist unique $T_1^i \in \cM_1$, $T_2^i \in \cM_2$, and $T_5^i \in \cM_5'$ depending only on the sub-intervals of $(x_1,x_i)$ such that 
\begin{equation*}
      \max\bigl\{|[D T_1^i(x)]_{i1}|,\, |[D T_2^i(x)]_{i1}|,\, |[D T_5^i(x)]_{i1}| \bigr\} \le \delta^{-1}\,,
\end{equation*}
and $[D T(x)]_{i1} = 0$ for any other $T \in \cM'$. If $x_i \not\in [-R,R)$ and $x_1 \in [-R,R)$, there exists $T_5^i \in \cM_5'$ such that $|[DT_5^i(x)]_{i1}| \leq \delta^{-1}$ and $[DT(x)]_{i1} = 0$ for any other $T\in \cM'$. Otherwise, if $x_1 \notin [-R,R)$, then $[D T(x)]_{i1} = 0$ for all $T \in \cM'$. Therefore,
\begin{equation*}
\begin{aligned}
         \|v\|^2 =& \sum_{i = 2}^d \Bigl(\sum_{T \in \cM'}  \lambda_T\, [D T(x)]_{i1} \Bigr)^2 \\
         \leq & \sum_{i=2}^d\bigl(\bigl[|\lambda_{T_1^i}D T_1^i(x)|+ |\lambda_{T_2^i}D T_2^i(x)|  + |\lambda_{T_5^i}D T_5^i(x)|\bigr]_{i1}\bigr)^2 \leq 3\,\|\lambda\|^2\, \delta^{-2}= 3\delta^{-2}\,.
\end{aligned}
\end{equation*}
Hence, $\|M\|_2  \leq 3\delta^{-1}$.
\end{proof}

We now establish the geodesic smoothness result for the KL divergence over $\underline{\rm cone} (\cM; \balpha \id)$. Recall that $\Upsilon \deq \bigl(L_V'^{1/2} + (d - 1)\, L_V^{1/2}\bigr)^2\, \frac{9}{\delta^2}\, \|Q^{-1}\|_2$ is defined in \cref{thm:PGD}.
\begin{lemma} \label{lemma:smoothness over cone}
If $\pi$ is $L$-log-smooth, the function $(\lambda, v) \mapsto \kl{\mu_{\lambda, v}}{\pi}$ is $(L + \Upsilon)$-smooth over the space $\left(\Theta, \|\cdot\|_\Theta \right)$.
\end{lemma}
\begin{proof}
Denote by $M  \deq  \sum_{T \in \cM'} \lambda_T D T(x)$. Recall the expressions~\eqref{eq:kl_grad_v} and~\eqref{eq:kl_grad_lambda} for the gradients of the KL divergence.

The second derivative of $\kl{\mu_{\lambda, v}}{\pi}$ w.r.t.\ $v$ is
\begin{equation}\label{eq:kl-hessian-2-v}
    \nabla^2_v \kl{\mu_{\lambda, v}}{\pi}  = \int \nabla^2 V\Bigl({\rm diag}(\balpha)\, x + \sum_{T \in \cM'} \lambda_T T(x) + v\Bigr)\, \rho(\rd x)\,, 
\end{equation}
and the mixed partial derivatives w.r.t.\ $(\lambda, v)$ are
\begin{equation}\label{eq:kl-hessian-lambda-v}
    \partial_{\lambda_T} \nabla_v  \kl{\mu_{\lambda, v}}{\pi} = \int   \nabla^2 V\Bigl({\rm diag}(\balpha)\, x + \sum_{T \in \cM'} \lambda_T T(x) + v\Bigr)\, T(x)\, \rho(\rd x).
\end{equation}
The Hessian of $\kl{\mu_{\lambda, v}}{\pi}$ w.r.t.\ $\lambda$ has entries:
\begin{align*}
\partial_{\lambda_T} \partial_{\lambda_{T^{'}}}\kl{\mu_{\lambda, v}}{\pi} &=
\int \biggl[ \underbrace{\tr \Bigl( \bigl({\rm diag}(\balpha) + M \bigr)^{-1} \, D T' (x)\, \bigl({\rm diag}(\balpha) + M \bigr)^{-1}\,  D T(x) \Bigr)}_{\eqqcolon H_{T, T'}} \\
&\qquad{} + \Bigl\langle T(x),  \nabla^2 V\Bigl({\rm diag}(\balpha) \, x + \sum_{T \in \cM'} \lambda_T T(x) + v\Bigr)\, T'(x) \Bigr\rangle \biggr]\, \rho(\rd x)\,. 
\end{align*}
To bound the operator norm of the matrix $H = (H_{T, T'})_{T, T' \in \cM'}$, consider an arbitrary unit vector $u \in \R^{|\cM'|}$ with $\|u\| = 1$. We compute
\begin{align*}
u^\top H u &= \sum_{T, T{'} \in \cM'} u_T \tr \bigl( ({\rm diag}(\balpha) + M)^{-1}\, D T^{'} (x) \,({\rm diag}(\balpha) + M)^{-1}\, D T(x) \bigr) u_{T^{'}} \\
&= \tr \Bigl( ({\rm diag}(\balpha) + M)^{-1} \sum_{T^{'} \in \cM'} u_{T^{'}} D T^{'} (x) \,({\rm diag}(\balpha) + M)^{-1} \sum_{T \in \cM'} u_T D T(x) \Bigr)\,. 
\end{align*}
For any lower triangular matrices $A,B \in \R^{d\times d}$ such that $B$ has positive diagonal entries, the following hold  $\tr(A^2) \leq \tr(A)^2$ and $|\tr(A B)| \leq \|A\|_2 \tr(B)$. Together with \cref{lemma:cone differential bound}, we derive that 
\begin{align*}
u^\top H u  &\leq \tr \Bigl(({\rm diag}(\balpha) + M)^{-1} \sum_{T\in\cM'} u_T D T(x) \Bigr)^2 \\
&\leq \tr \bigl( ({\rm diag}(\balpha) + M)^{-1} \bigr)^2\, \frac{9}{\delta^2}\\
& =  \tr \bigl( ({\rm diag}(\balpha) + {\rm diag}(M))^{-1} \bigr)^2\, \frac{9}{\delta^2} \\
&\leq \Bigl(\sum_{j = 1}^d \alpha_j^{-1}\Bigr)^2\, \frac{9}{\delta^2} = \bigl(L_V'^{1/2} + (d - 1)\, L_V^{1/2}\bigr)^2\, \frac{9}{\delta^2}\,,
\end{align*}
where the last two lines follow by the property of lower triangular matrices and the non-negativity of the diagonal entries of $M$. Therefore, we have shown that 
\begin{equation}\label{eq:kl-hessian-2-lambda}
    H \preceq L_H I_{|\cM'|}\,,\qquad\text{where}\,\, L_H  \deq  \bigl(L_V'^{1/2} + (d - 1)\, L_V^{1/2}\bigr)^2\, \frac{9}{\delta^2}\,.
\end{equation}

Denote by $\mb T_{\cM'} = [T_1,\ldots, T_{|\cM'|}]: \R^d \to \R^{d\times |\cM'|}$ the matrix-valued function whose columns are transport maps in $\cM'$. 
By~\eqref{eq:kl-hessian-2-v}--\eqref{eq:kl-hessian-2-lambda}, we derive that
\begin{align*}
\nabla^2_{\lambda, v} \kl{\mu_{\lambda, v}}{\pi} 
&= \int \begin{bmatrix} \mb T_{\cM'}(x) \\ I \end{bmatrix} 
\nabla^2 V\Bigl({\rm diag}(\balpha) \, x + \sum_{T \in \cM'} \lambda_T T(x) + v\Bigr) 
\left[\mb T_{\cM'}(x), I \right] \rho(\rd x) \\ 
&\qquad{} \phantom{=}+  
\begin{bmatrix}
H & 0_{|\cM'| \times d} \\
0_{ d \times |\cM'|} & 0_{d\times d} 
\end{bmatrix}  \\
&\preceq 
L \int \left[\mb T_{\cM'}(x), I \right]^\top 
\left[\mb T_{\cM'}(x), I \right] \rho(\rd x) 
+ L_H 
\begin{bmatrix}
I_{|\cM'|} & 0_{|\cM'| \times d} \\
0_{ d \times |\cM'|} & 0_{d\times d} 
\end{bmatrix}\,.
\end{align*}
Under the geometry $\left(\Theta, \|\cdot\|_\Theta \right)$, the partial Hessian $\nabla_{\lambda}^2 \kl{\mu_{\lambda, v}}{\pi}$ needs to be rescaled by $Q^{-1/2}$. Concretely, for any linear operator $A:\R^{|\cM'| + d} \to \R^{|\cM'| + d}$, denote by $\|A\|_{\Theta} \deq \sup\left\{\|A\theta\|_\Theta^*\,:\, \|\theta\|_\Theta =1 \right\}$ the operator norm of $A$ where $\|\cdot\|_{\Theta}^*$ is the dual norm of $\|\cdot\|_\Theta$, then $\|A\|_{\Theta} = \|G^{-1/2} A G^{-1/2}\|_2$ with $G = \begin{bmatrix}
    Q& 0_{|\cM'| \times d }\\
    0_{d\times |\cM'|} & I
\end{bmatrix}$. Thus,
\begin{equation*}
\begin{aligned}
     &\bigl\| \nabla^2_{\lambda, v} \kl{\mu_{\lambda, v}}{\pi}  \bigr\|_{\Theta} \\
     &\qquad \leq    L\, \Bigl\| \int \left[ \mb T_{\cM'}(x)\,Q^{-1/2},\, I\right]^\top \left[ \mb T_{\cM'}(x)\,Q
     ^{-1/2},\, I \right] \rho(\rd x) \Bigr\|_2
     + L_H\, \|Q^{-1}\|_2\\
&\qquad =  L\, \Bigl\| \begin{bmatrix} Q^{-1/2}\int \mb T_{\cM'}(x)^\top  \mb T_{\cM'}(x)\, \rho(\rd x)\, Q^{-1/2} & 0_{|\cM'| \times d} \\[4pt]
0_{d \times |\cM'|} & I
\end{bmatrix} \Bigr\|_2  + L_H\, \|Q^{-1}\|_2 \\[6pt]
&\qquad \le  L + L_H\, \|Q^{-1}\|_2\,.
\end{aligned}
\end{equation*}
The second equality follows from the block structure of the matrix and the fact that 
$\int \mb T_{\cM'}(x)\, \rho(\dd x) = 0_{d \times |\cM'|}$ after re-centering. 
\end{proof}
\begin{proof}[Proof of \cref{thm:PGD}]\label{proof: thm:PGD}
First, we apply \cref{lemma: log-concavity-pi-schur}: as $\pi$ satisfies \eqref{assum:P1} and \eqref{assum:RD+} implies \eqref{assum:RD}, $\pi$ is $\ell_V' \land \ell_V$-log-concave. Since \eqref{assum:R} and \eqref{assum:P2} are also satisfied, $\pi$ is $L_V \lor (L_V'/2)$-log-smooth.
 
By~\cref {lemma:smoothness over cone}, $(\lambda, v)  \mapsto  \kl{\mu_{\lambda, v}}{\pi}$ is $(\ell_V' \wedge \ell_V)$-strongly convex and $(L_V \vee (L_V'/2)+\Upsilon)$-smooth over $(\Theta, \|\cdot\|_\Theta)$.
The guarantee now follows from \cite[Theorem~10.29]{Beck2017}.
\end{proof}
\end{appendix}

\end{document}